\documentclass{hcig}

\usepackage{tocloft}
\usepackage{hyperref}
\usepackage{natbib}
\usepackage{url}
\usepackage[T1]{fontenc}
\usepackage[utf8]{inputenc}

\usepackage{amsmath}
\usepackage{amssymb}
\usepackage{amsfonts}
\usepackage{mathtools}
\usepackage{amsthm}
\usepackage{array}
\usepackage{bbold}
\usepackage{bm}
\usepackage{float}
\usepackage{esvect}
\usepackage{latexsym}
\usepackage{wrapfig}

\usepackage{graphicx}
\usepackage{tikz}
\usepackage{enumitem}
\usepackage{subcaption}

\usepackage{algorithm}
\usepackage{algpseudocode} 
\usepackage{setspace}

\usepackage{booktabs}
\usepackage{multirow}

\theoremstyle{definition}
\newtheorem{definition}{Definition}
\newtheorem{assumption}{Assumption}
\newtheorem{lemma}{Lemma}
\newtheorem{theorem}{Theorem}

\algrenewcommand\algorithmicrequire{\textbf{Input:}}
\algrenewcommand\algorithmicensure{\textbf{Output:}}

\def\defn{\,\triangleq\,}
\def\argmin{\mathop{\mathsf{arg\,min}}} 
\def\argmax{\mathop{\mathsf{arg\,max}}}
\def\lim{\mathop{\mathsf{lim}}} 
\def\min{\mathop{\mathsf{min}}}
\def\max{\mathop{\mathsf{max}}}
\def\sup{\mathop{\mathsf{sup}}}

\def\prox{\mathsf{prox}}
\def\log{\mathsf{log\,}}
\def\exp{\mathsf{exp}}

\def\FI{\mathsf{FI}}
\def\KL{\mathsf{KL}}

\def\ebm{{\bm{e}}}

\def\kbm{{\bm{k}}}

\def\xbm{{\bm{x}}}
\def\ybm{{\bm{y}}}
\def\zbm{{\bm{z}}}
\def\Xbm{{\bm{X}}}

\def\zerobm{\bm{0}}
\def\onebm{\bm{1}}

\def\Abm{{\bm{A}}}
\def\Bbm{{\bm{B}}}

\def\Hbm{{\bm{H}}}
\def\Ibm{{\bm{I}}}

\def\Kbm{{\bm{K}}}

\def\Mbm{{\bm{M}}}

\def\Tbm{{\bm{T}}}

\def\Wbm{{\bm{W}}}
\def\Xbm{{\bm{X}}}

\def\R{\mathbb{R}}
\def\E{\mathbb{E}}

\def\Xbf{{\mathbf{X}}}

\def\Ncal{{\mathcal{N}}}

\def\Lcal{{\mathcal{L}}}

\DeclareMathAlphabet{\mathsfit}{T1}{\sfdefault}{\mddefault}{\sldefault}

\def\Ssfit{{\mathsfit{S}}}
\def\Dsfit{{\mathsfit{D}}}

\def\Tsfit{\mathsfit{T}}

\def\epsilon{\varepsilon}

\def\dXt{{d\Xt}}
\def\dBt{{d\Bbm_t}}
\def\Xt{{\xbm_t}}

\def\Bt{{\Bbm_t}}

\def\sigmat{\sigma_t}

\def\Wn{\Wbm_n}

\def\Xn{{\xbm_n}}

\def\Xnpo{{\xbm_{n+1}}}

\def\Znpo{{\zbm_{n+1}}}

\def\sigman{\sigma_n}

\def\psigma{p_\sigma}

\def\proposed{PiX-MC}
\def\aproposed{APiX-MC}
\def\mbproposed{multi-block PiX-MC}
\def\ambproposed{multi-block APiX-MC}

\algrenewcommand\algorithmicrequire{\textbf{Input:}}
\algrenewcommand\algorithmicensure{\textbf{Output:}}

\title{Picard Proximal Monte Carlo for Parallel Bayesian Imaging with Score-Based Generative Priors}

\author{Deliang Wei$^1$, Evan Bell$^1$, Wenhan Guo$^1$, Yifan Chen$^2$, and Yu Sun$^{1,}$\textsuperscript{\Letter}}
\address{$^1$Johns Hopkins University \quad $^2$University of California, Los Angeles\\\smallskip
{\footnotesize \textsuperscript{\Letter}Corresponding author: ysun214@jh.edu}}

\headertitle{PiX-MC}
\headerauthors{Wei et al.}

\begin{document}

\maketitle
\thispagestyle{firstpagestyle}

\begin{abstract}
Bayesian imaging inverse problems often require sampling from high-dimensional posterior distributions. While recent score-based and diffusion models provide expressive Bayesian priors, their sampling procedures remain inherently sequential and computationally expensive for large-scale imaging applications. We propose \textit{\textbf{\proposed}}, a time-parallel posterior sampling framework based on proximal Langevin dynamics and Picard iteration. The proximal-likelihood formulation exploits the fact that many imaging likelihoods admit efficient, problem-specific proximal operators, while Picard refinement exposes parallelism across discretization nodes and naturally supports multi-GPU implementation. To further improve practical scalability and sampling performance, we develop multi-block and annealed variants of the proposed framework. We establish convergence guarantees under transparent assumptions, accommodating non-log-concave posteriors, imperfect learned score models, multi-block implementations, and annealing schedules. Experiments on a diverse collection of imaging inverse problems demonstrate that {\proposed} substantially reduces wall-clock time while preserving reconstruction quality. On a $512\times512\times80$ sparse-view computed tomography (CT) problem, annealed multi-block {\proposed} achieves up to a $50\times$ runtime speedup over the standard Langevin sampler using eight GPUs.
\end{abstract}

\section{Introduction}

The recovery of an unknown image $\xbm \in \R^d$ from a set of sparse and noisy measurements $\ybm \in \R^c$ is a fundamental problem in computational imaging. 
It is often mathematically formulated as an \emph{inverse problem} of the following system
\begin{equation}
\ybm = \Abm(\xbm) + \ebm,
\end{equation}
where the forward operator $\Abm: \R^d \rightarrow \R^c$ models the imaging system response and 
$\ebm \in \R^c$ denotes measurement noise. 
Due to the sparsity and noisiness of $\ybm$, multiple plausible $\xbm$ can explain the same set of measurements, leading to significant uncertainty in the recovered image.
Bayesian sampling methods offer a principled framework for addressing this ambiguity. 
Rather than seeking a single point estimate, they aim to characterize the full posterior distribution $\pi(\xbm|\ybm)$ given an image prior $p(\xbm)$.
This capability is particularly beneficial in high-stakes biomedical and scientific imaging applications, where different imaging results can lead to different scientific conclusions, clinical decisions, or downstream actions~\citep{begoli2019the,kam2019uncertainty}.

Traditionally, the image prior is handcrafted, including total variation and sparsity-promoting penalties~\citep{rudin1992nonlinear, sauer1992bayesian, gu2014weighted}; yet such priors often fall short in capturing the complexity of real-world image distributions. 
Recently, \emph{score-based diffusion models (SDMs)} have emerged as a powerful class of generative models capable of modeling such complexity, inspiring a growing body of work on SDM-based posterior sampling for imaging inverse problems. 
One popular line of work incorporates measurements as guidance to steer the generative process toward the posterior distribution~\citep{ho2021classifierfree, chung2022diffusion}; this approach is convenient to implement on top of pre-trained models and has been demonstrated across a wide range of imaging applications~\citep{chung2022score, song2024solving, he2025diffusion}. 
However, such guidance-based methods generally lack theoretical guarantees on whether the generated samples indeed follow the true posterior, significantly limiting their interpretability and reliability in high-stakes settings.

In light of this, another class of methods integrates SDMs with classical \emph{Markov chain Monte Carlo (MCMC)} frameworks, providing more rigorous posterior sampling guarantees~\citep{brooks2011handbook}. 
Popular instantiations include combining Langevin dynamics with the score network of an SDM~\citep{jalal2021robust,sun2024provable}, and incorporating a full SDM as an empirical proximal sampler in the split Gibbs process~\citep{coeurdoux2024plug,wu2024principled}. 
Despite their algorithmic differences, these methods share a common sequential structure: samples must be evolved through long iterative trajectories, each requiring repeated evaluations of both the score model and the forward operator. 
As a result, the computational cost scales poorly with the problem dimension and the complexity of the forward model, making large-scale imaging applications particularly challenging. 
This limitation is increasingly pronounced as modern machine learning systems are routinely deployed on \emph{multiple graphics processing unit (multi-GPU)} platforms that offer substantial 
\begin{wrapfigure}{r}{0.6\textwidth}
  \vspace{-10pt}
  \begin{center}
    \includegraphics[width=0.6\textwidth]{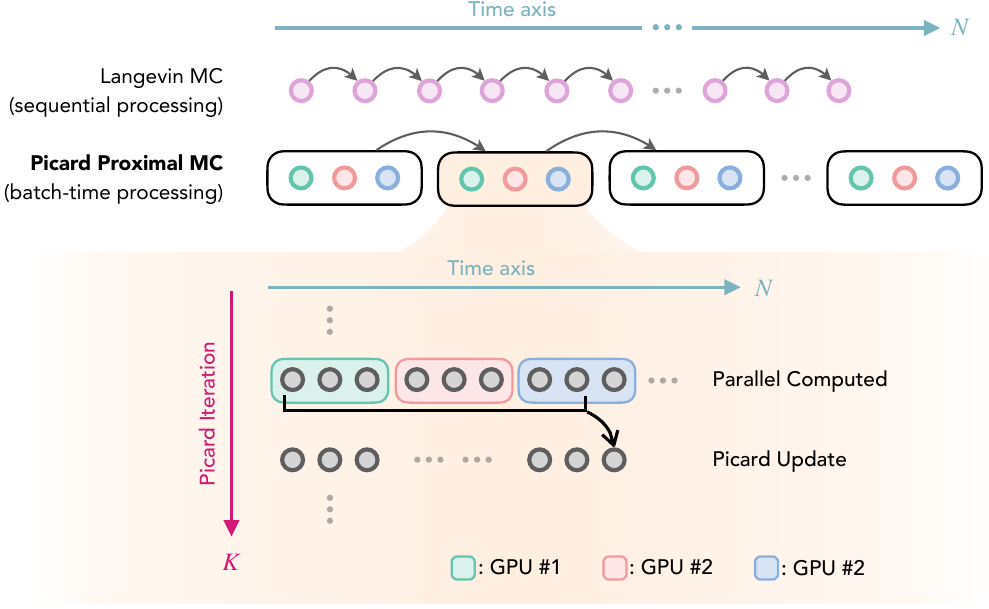}
  \end{center}
  \vspace{-10pt}
  \caption{Schematic illustration of the proposed Picard Proximal Monte Carlo ({\proposed}) method. {\proposed} employs time-batch processing and exploits Picard iteration to enable parallel computation across time nodes.}
  \label{fig:1}
\end{wrapfigure}
parallel computational resources; yet existing samplers are not designed to exploit such parallelism.
While efforts have explored parallelism for solvers applicable to inverse problems~\citep{lian2015asynchronous,zhou2018distributed,sun2021asyncred}, they are largely restricted to deterministic optimization frameworks, limiting the direct applicability of these techniques to posterior sampling.

In this work, we bridge this gap by proposing \emph{\textbf{Pi}card pro\textbf{X}imal \textbf{M}onte \textbf{C}arlo (\textbf{\proposed})}, a time-parallel posterior sampling framework for scalable Bayesian imaging. 
{\proposed} is built upon a forward-backward splitting of Langevin dynamics that decouples the prior and likelihood, enabling modular integration of a \emph{score-based generative prior} and a likelihood update via \emph{proximal mapping}.
To overcome the sequential bottleneck, {\proposed} further leverages \emph{Picard iteration} to parallelize the sampling trajectory across time steps. 
A schematic illustration is provided in Figure~\ref{fig:1}.
The key contributions of this work are as follows:
\begin{itemize}
\item We formulate the sampling trajectory of {\proposed} as a fixed-point problem and solve it via Picard iteration. 
Rather than parallelizing across image pixels, {\proposed} introduces parallelism across time steps, enabling concurrent processing over multiple GPUs. 
Unlike existing Picard-based methods~\citep{shih2023parallel, anari2024fast}, {\proposed} incorporates a forward--backward formulation that imposes measurement consistency through a proximal update. 
This design is particularly attractive for computational imaging, as it naturally exploits the structure of many imaging forward operators.
We further develop \emph{annealed} and \emph{multi-block} variants of {\proposed}. 
The former employs temperature annealing to accelerate the underlying sequential sampling trajectory, while the latter partitions the sampling trajectory into multiple blocks, allowing {\proposed} to still operate efficiently under limited GPU budgets.

\item We present a comprehensive theoretical analysis of the proposed suite of {\proposed} algorithms. 
Our analysis bounds the time-averaged relative Fisher information by an $O(1/T)$ transient term together with explicit approximation-error terms.
The theory accommodates imperfect learned score networks and weakly convex imaging likelihoods, covering practical settings encountered in a wide range of imaging problems. 
Furthermore, our results provide a detailed error decomposition that explicitly characterizes the individual contributions of forward--backward splitting, Picard iterations, and time discretization, offering interpretable insight into the sources of sampling error.

\item We validate {\proposed} on a diverse set of computational imaging inverse problems, including both linear and nonlinear forward models. 
In addition, we consider large-scale problems ranging from $1024\times1024$ image deblurring to $512\times512\times80$ sparse-view computed tomography (CT), covering both 2D and 3D imaging settings. 
These experiments are conducted on multi-GPU platforms to reflect realistic deployment conditions.
The results demonstrate that {\proposed} generalizes across these settings while substantially reducing wall-clock time without sacrificing reconstruction quality. On the large-scale 3D CT task, {\aproposed} achieves up to a $50\times$ runtime speedup over the standard Langevin sampler using eight GPUs.

\end{itemize}

\section{Background and Related Works}\label{sec:background}
We begin by reviewing regularized image reconstruction from a Bayesian perspective.
We then introduce SDMs and discuss posterior sampling methods for imaging inverse problems.
Finally, we review related parallelization approaches for image reconstruction.

\subsection{Bayesian Image Reconstruction}
Classical approaches to imaging inverse problems are formulated as regularized optimization problems, seeking a solution that simultaneously fits the measurements and conforms to prior knowledge about the image. 
Within the Bayesian framework, this corresponds to \emph{maximum a posteriori (MAP)} estimation
\begin{equation}
\label{Eq:MAP}
\hat{\xbm} = \argmax_{\xbm\in\R^d} \ell(\ybm|\xbm)\, p(\xbm) 
= \argmin_{\xbm\in\R^d} \big\{ L(\xbm) + V(\xbm) \big\},
\end{equation}
where $\ell(\ybm|\xbm)$ is the likelihood of the measurements, $L(\xbm) = -\log \ell(\ybm|\xbm)$ the data-fidelity term, and $V(\xbm) = -\log p(\xbm)$ is the regularizer. 
A common choice for the data-fidelity is the least-squares loss $L(\xbm) = \frac{1}{2\beta^2}\|\ybm - \Abm(\xbm)\|_2^2$, which arises from the Gaussian likelihood $\ybm|\xbm \sim \mathcal{N}(\Abm(\xbm), \beta^2\Ibm)$ with noise variance $\beta^2>0$. 
Popular regularizers include transform-domain sparsity penalties such as $V(\xbm)=\tau\|\Tbm\xbm\|_1$~\citep{candes2006robust, donoho2006compressed}, where $\Tbm$ is a sparsifying transform and $\tau > 0$ controls the regularization strength.

Proximal methods~\citep{boyd2011} are widely adopted for solving~\eqref{Eq:MAP} in computational imaging.
Their popularity stems from the ability to handle nonsmooth regularizers, such as transform-domain sparsity~\citep{beck.teboulle2009fast}, while exploiting the composite structure of~\eqref{Eq:MAP}. 
Moreover, proximal algorithms naturally accommodate the structure of imaging forward operators, often enabling efficient implementations when $\Abm$ admits fast transforms or closed-form inverses~\citep{afonso2010fast}. 
The fundamental building block is the \emph{proximal operator}
\begin{equation}
\label{Eq:ProximalOperator}
\prox_{\eta F}(\zbm) \defn \argmin_{\xbm \in \R^d} \left\{
\frac{1}{2}\|\xbm-\zbm\|_2^2 + \eta F(\xbm)\right\},
\end{equation}
where the quadratic term encourages the output to remain close to the input $\zbm$, and $\eta>0$ controls the influence of $F$. 
The proximal operator admits an interpretation as a \emph{backward} or \emph{implicit} gradient step, as its optimality condition yields
\begin{equation}\nonumber
\xbm \in \zbm - \eta \partial F(\xbm) \quad \Leftrightarrow \quad 
\xbm = (I + \eta \partial F)^{-1}(\zbm),
\end{equation}
where $\partial F(\xbm)$ denotes the subdifferential of $F$ at $\xbm$, and $(I + \eta \partial F)^{-1}$ is known as the \emph{resolvent} of $\partial F$. 
A common proximal algorithm is \emph{forward-backward splitting (FBS)}, which combines an explicit gradient step with an implicit proximal step. 
For instance, given a smooth $V$ and well-structured $\Abm$, the update of FBS reads as
\begin{align}
\label{Eq:PGM}
\Xnpo &= \prox_{\gamma L}\big(\Xn - \gamma\nabla V(\Xn)\big) \nonumber\\
&= (I + \gamma \partial L)^{-1}\circ(I - \gamma\nabla V)(\Xn).
\end{align}
where $\circ$ denotes function composition, and $\gamma>0$ denotes the stepsize.
Note that the prior and likelihood information are incorporated separately through the gradient and proximal steps, respectively.
Classic convex optimization theory has established convergence of FBS to a minimizer of~\eqref{Eq:MAP} under various conditions~\citep{nesterov2004, beck.teboulle2009convex}. 

\subsection{Score-Based Diffusion Models as Learned Image Priors}\label{sec: 2.2 score}

SDMs offer a powerful framework for learning image priors from data.
Rather than modeling the image distribution $p$ directly, SDMs learn the score of a Gaussian-smoothed version of $p$. 
Specifically, let $\xbm \sim p$ be a clean image and let $\zbm = \xbm + \sigma\bm{\zeta}$ with $\bm{\zeta} \sim \Ncal(\bm{0}, \Ibm)$ be a noisy observation at noise level $\sigma > 0$. 
Denote by $p_\sigma = p * \mathcal{N}(\bm{0}, \sigma^2\Ibm)$ the distribution of $\zbm$. 
By Tweedie's formula~\citep{efron2011tweedie}, the score of $p_\sigma$ admits the closed-form expression
\begin{equation}
\label{Eq:Tweedie}
\nabla \log p_\sigma(\zbm) = \left(\E[\xbm|\zbm] - \zbm\right) / \sigma^2,
\end{equation}
where $\E[\xbm|\zbm]$ is the minimum mean-squared error (MMSE) estimator of the clean image given the noisy observation. Equation~\eqref{Eq:Tweedie} establishes a direct connection between score estimation and Gaussian denoising.
This observation motivates denoising score matching (DSM)~\citep{vincent2011a}, which learns a neural network approximation of the smoothed score, \emph{i.e.} $\Ssfit_\theta(\,\cdot\,;\sigma)\approx-\nabla\log \psigma$
where $\theta$ denotes the weights of the network.
A common DSM training objective is
\begin{equation}
\Lcal(\theta) = \E\left[\left\|\Ssfit_\theta(\zbm\,;\sigma) - \bm{\zeta}/\sigma\right\|^2\right],\quad\text{with}\quad \zbm = \xbm + \sigma\bm{\zeta},
\end{equation}
which can be optimized efficiently using a dataset of clean images. 
As $\sigma\rightarrow 0$, the smoothed density $p_\sigma$ approaches $p$, yielding $\Ssfit_\theta(\,\cdot\,;\sigma)\approx-\nabla\log p = \nabla V$.
As a result, the learned score network provides a data-driven approximation of the gradient of the prior potential.

With the learned score network $\Ssfit_\theta(\cdot;\sigma)$, one can construct sampling algorithms that generate samples from the learned prior. 
A popular approach is based on \emph{Langevin dynamics}~\citep{parisi1981correlation, song2019generative}, which evolves the sample according to the following stochastic differential equation (SDE)
\begin{equation}
\label{Eq:ALD}
\dXt=-\Ssfit_\theta(\Xt,\sigmat)\,dt+\sqrt{2}\,\dBt,
\end{equation}
where $\sigmat$ denotes a time-dependent noise level and $\Bt\in\R^d$ is standard Brownian motion. 
While one could choose $\sigmat$ to be a small constant, empirical studies have shown that gradually decreasing $\sigmat$ toward zero during sampling leads to substantially improved performance~\citep{neal2001, song2019generative}. 
This strategy, commonly known as \emph{annealing}, allows the sampler to first explore the global structure of the distribution at high noise levels before progressively refining samples at lower noise levels.
An alternative family of SDMs formulates sampling through a reverse-time diffusion process~\citep{ho2020denoising, song2020score, karras2022elucidating}. 
Let $p_t$ denote the density of a forward noising process governed by $\dXt = u(t)\,\Xt\,dt + v(t)\,\dBt$, where $u(t)$ and $v(t)$ denote the drift and diffusion coefficients, respectively~\citep{karras2022elucidating}.
The corresponding reverse-time dynamics is given by
\begin{equation}
\label{Eq:Reverse}
\dXt=\left[u(t)\Xt+v(t)^2\Ssfit_\theta(\Xt,\sigmat)\right]dt + v(t)\,\dBt,
\end{equation}
where $\sigmat$ corresponds to the noise level associated with $p_t$, which is predetermined by the forward process.
Starting from Gaussian noise, simulating~\eqref{Eq:Reverse} progressively transforms the sample toward the target image distribution.
Both Langevin-based and reverse-diffusion formulations have demonstrated remarkable empirical success in generating high-quality images across a wide range of domains, including natural images~\citep{song2020improved, dhariwal2021diffusion}, medical imaging~\citep{kazerouni2023diffusion,dorjsembe2024conditional}, and scientific imaging applications~\citep{feng2023score, sun2024provable}.
This strong generative capability makes score-based diffusion models particularly attractive as image priors for solving inverse problems, as discussed in the next section.

\subsection{Posterior Sampling Methods for Imaging}

Existing posterior sampling methods can largely be viewed as constructing stochastic processes whose stationary or terminal distribution approximates the target posterior $\pi(\xbm|\ybm)$. 
These methods can be roughly categorized into three families according to the type of their stochastic process: Langevin-based methods, split Gibbs samplers, and reverse-time diffusion approaches.

\paragraph{Langevin-Based Methods} The underlying idea of this family is to construct a Langevin diffusion process whose invariant distribution coincides with the posterior.
Following~\eqref{Eq:ALD}, the drift is decomposed into a likelihood gradient $\nabla L$ and a learned score $\Ssfit_\theta$, and a forward Euler discretization yields the \textit{\textbf{Langevin Monte Carlo (L-MC)}} iteration
\begin{equation}
\label{Eq:LMC}
\Xnpo = \Xn - \gamma\big(\nabla L(\Xn) + \Ssfit_\theta(\Xn; 
\sigman)\big) + \sqrt{2\gamma}\,\Wn,
\end{equation}
where $\gamma>0$ denotes the step size and where $\Wbm_{n}=(\Bbm_{(n+1)\gamma}-\Bbm_{n\gamma})/\sqrt{\gamma}
\sim\mathcal N(0,\Ibm)$ is a $d$-dimensional standard Gaussian vector.
Numerous variants of L-MC have been developed for imaging inverse problems. 
Recent works have incorporated customized annealing strategies for the likelihood gradient $\nabla L$ and the score network $\Ssfit_\theta$ to accelerate empirical convergence~\citep{jalal2021robust, 
sun2024provable}, explored score-based priors in latent spaces~\citep{holden2022bayesian, coeurdoux2024normalizing, spagnoletti2025latino}, and extended Langevin sampling to non-Gaussian likelihoods with highly nonconvex landscapes~\citep{melidonis2023efficient}. 
These methods have been successfully applied across a broad range of imaging tasks~\citep{jalal2021robust,coeurdoux2024normalizing,sun2024provable}. 
On the theoretical side, substantial effort has been devoted to establishing quantitative convergence guarantees for both continuous-time and discretized Langevin dynamics; see, e.g.,~\citep{gross1975logarithmic,durmus2019analysis}. 
While classical analyses often rely on strong log-concavity assumptions, recent advances have significantly broadened the theoretical understanding beyond this regime~\citep{balasubramanian2022towards}. 
Many of these results are formulated through functional inequalities satisfied by the target distribution, including logarithmic Sobolev~\citep{vempala2019rapid} and Poincaré inequalities~\citep{chewi2025analysis}.

One closely related line of work exploits the composite structure of the posterior potential through proximal updates. When the prior potential $V$ is nonsmooth, \cite{durmus2018efficient} proposed to replace the prior score with the gradient of the \emph{Moreau-Yosida envelope}, yielding the following update
\begin{equation}
\label{Eq:MYULA}
\Xnpo = \Xn - \gamma\left(\nabla L(\Xn) + \frac{1}{\lambda}\big(
\Xn - \prox_{\lambda V}(\Xn)\big)\right) + \sqrt{2\gamma}\,\Wn,
\end{equation}
where $\frac{1}{\lambda}\big(\xbm_n - \prox_{\lambda V}(\xbm_n)\big)$ is the gradient of the Moreau-Yosida envelope of $V$. 
Subsequent works have replaced the proximal operator with pretrained deep denoisers, giving rise to plug-and-play Moreau-Yosida Langevin samplers~\citep{laumont2022bayesian}.
Another line of work considers an asymmetric discretization of the Langevin dynamics, applying a backward step to one term and a forward step to the other~\citep{durmus2019analysis, wibisono2019proximal, salim2019stochastic}. 
This yields a \emph{proximal Langevin} update of the form
\begin{align}
\label{Eq:PLA}
\Xnpo &= \Xn - \gamma\big(\nabla L(\Xnpo) + \nabla V(\Xn)\big) + \sqrt{2\gamma}\,\Wn, \nonumber\\
&= \prox_{\gamma L}\Big(\Xn - \gamma\nabla V(\Xn)+\sqrt{2\gamma}\,\Wn\Big).
\end{align}
which applies a proximal step on the negative log-likelihood $L$ and a gradient step on the negative log-prior $V$. 
This forward-backward structure draws a natural connection to FBS in~\eqref{Eq:PGM}.
\proposed~complements this line of research by integrating Picard iteration into the proximal Langevin framework, enabling parallel trajectory evaluation for large-scale imaging inverse problems.

\paragraph{Split Gibbs Samplers}
An alternative approach to posterior sampling is based on variable splitting and Gibbs sampling~\citep{vono2019split, pereyra2023split, bouman2023generative, faye2024regularization}. 
These methods introduce an auxiliary variable $\zbm$ and couple it with $\xbm$ through a quadratic penalty $\frac{1}{2\rho}\|\xbm-\zbm\|^2$, where $\rho>0$ controls the coupling strength.
Sampling is then performed by alternating between the conditional distributions
\begin{subequations}
\begin{align}
\Xnpo&\sim\exp\left(-L(\xbm)-\frac{1}{2\rho}\|\xbm-\zbm_n\|^2\right)\\
\Znpo&\sim\exp\left(-V(\zbm)-\frac{1}{2\rho}\|\zbm-\xbm_{n+1}\|^2\right).\label{Eq:PriorSampler}
\end{align}
\end{subequations}
The resulting decoupling allows the likelihood and prior terms to be handled independently, making it possible to exploit the structure of the forward model while incorporating sophisticated learned priors in a modular fashion. 
Recent works have employed diffusion models to approximate the prior conditional distribution in~\eqref{Eq:PriorSampler}, establishing convergence to the augmented posterior under various conditions~\citep{wu2024principled, coeurdoux2024plug, xu2024provably}, with extensions to blind inverse problem settings~\citep{hu2026prism}.
Despite these advances, the conditional updates remain computationally demanding, as each Gibbs iteration typically requires solving an additional sampling subproblem.

\paragraph{Guided Reverse Diffusion} Rather than constructing a Markov chain that directly targets the posterior distribution, guided reverse-diffusion methods modify the reverse diffusion trajectory using measurement information. 
Starting from a learned generative diffusion process, the reverse dynamics are guided toward measurement-consistent samples through likelihood-based corrections~\citep{chung2022diffusion} or approximations of the posterior score~\citep{ho2021classifierfree}. 
This viewpoint underlies a broad class of diffusion posterior sampling and guided diffusion methods, which have demonstrated strong empirical performance across a wide range of imaging tasks~\citep{chung2022score,chung2023parallel,kawar2022denoising,rout2023solving,liu2023dolce,song2024solving,yang2025ct}. 
Compared with Langevin-based approaches, however, these methods typically rely on approximations of the time-dependent posterior score and therefore do not always admit a straightforward characterization in terms of a target invariant distribution.

\subsection{Related Parallel Imaging Methods \& Picard Iteration}
Most existing parallel methods consider a general optimization formulation applicable to the MAP problem in~\eqref{Eq:MAP}. 
They accelerate computation by exploiting parallelism within individual iterations, including stochastic gradient methods~\citep{recht2011hogwild, lian2017can}, block-coordinate updates~\citep{fercoq2015accelerated, richtarik2016parallel}, and asynchronous procedures~\citep{peng2016arock, sun2017asynchronous, hannah2019a2bcd, sun2021asyncred}. 
Some of these techniques have also been extended to sampling~\citep{salim2019stochastic, salim2020primal, ding2021random}. 
However, they do not reduce the intrinsic sequential structure of iterative algorithms, which remains the fundamental computational bottleneck for large-scale posterior sampling.

A more recent line of work seeks to parallelize the sampling trajectory itself. 
In image generation, \emph{Picard iteration} has been used to evaluate multiple denoising steps simultaneously~\citep{shih2023parallel, so2025pcm}. 
Consider a dynamical system $d\xbm_t=f_t(\xbm_t)\,dt$, $t\in[0,T]$, with initial condition $\xbm_0$.
Picard iteration generates a sequence of trajectories according to
\begin{equation}
\xbm_t^{(0)}\equiv \xbm_0, \quad \xbm_t^{(k+1)} = \xbm_0+\displaystyle\int_0^tf_s(\xbm_s^{(k)})\,ds, \quad \text{for } k=0,\ldots,K-1.
\end{equation}
Note that each point $\xbm_t^{(k+1)}$ depends only on the previous trajectory $(\xbm_s^{(k)})_{s\in[0,t]}$, rather than on $\xbm_s^{(k+1)}$ for $s<t$. 
This formulation breaks the sequential dependency across time and allows $\xbm_t^{(k+1)}$ to be computed in parallel across time nodes after discretization.
On the theoretical side, \cite{anari2024fast} established parallelization guarantees for Langevin dynamics under functional inequalities such as logarithmic Sobolev, with subsequent works extending these ideas to diffusion models~\citep{chen2024accelerating}, diagonal trajectory updates~\citep{zhou2025parallel}, and higher-order Langevin schemes~\citep{mahajan2025fast}. 
Despite these advances, existing trajectory-parallel methods are designed primarily for either unconditional generative models or posterior distributions with explicit unnormalized densities. They do not directly accommodate the structure of Bayesian imaging posteriors, which often combine proximal likelihood updates, learned image priors, and annealing schedules.
{\proposed} significantly extends this line of work to posterior sampling for imaging inverse problems, and provides a comprehensive theoretical analysis that accommodates all these components.

\section{Method: Picard Proximal Monte Carlo ({\proposed})}
\label{sec:method}

This section presents the proposed {\proposed} framework. 
We first construct a proximal-likelihood drift motivated by the FBS-style update, which induces a sequential proximal Langevin dynamics. 
We then reformulate the resulting dynamics using Picard iteration, enabling parallel computation across sequential time steps. 
We further establish finite-horizon bounds on the time-averaged relative Fisher information of {\proposed} under transparent assumptions.
Finally, we introduce multi-block and annealed variants of {\proposed} to further improve scalability and imaging performance. 
The analysis extends to both variants under
their stated variant-specific assumptions.

\subsection{Proximal-Likelihood Drift via Forward-Backward Splitting}
We consider the FBS update in \eqref{Eq:PGM} to construct our drift term.
To distinguish the FBS splitting parameter from the Langevin stepsize $\gamma$ introduced later, we denote the former by $\eta>0$ throughout this section. Replacing $\nabla V(\xbm)$ with a pretrained score network $\Ssfit_\theta(\xbm;\sigma)$ yields

\begin{equation}
\Xnpo=\prox_{\eta L}\bigl(\Xn-\eta \Ssfit_\theta(\Xn;\sigma)\bigr),
\end{equation}
or equivalently,
\begin{equation}
\frac{\Xnpo-\Xn}{\eta}=-\frac{1}{\eta}\Big(\Xn-\prox_{\eta L}\bigl(\Xn-\eta \Ssfit_\theta(\Xn;\sigma)\bigr)\Big).
\end{equation}
The incremental representation motivates the following ordinary differential equation (ODE)
\begin{equation}\label{eq: ode and T}
d\xbm_t=-\Tsfit_{\eta,\sigma}(\xbm_t)dt, \quad \text{where}\quad\Tsfit_{\eta,\sigma}\defn\frac{1}{\eta}\Big(I-\prox_{\eta L}\circ\big(I-\eta \Ssfit_\theta(\cdot;\sigma)\big)\Big).
\end{equation}
When the noise level $\sigma$ is fixed throughout the iteration, we omit its dependence and simply write $\Ssfit_\theta$ and $\Tsfit_\eta$. 
The drift $\Tsfit_\eta$ combines the learned prior score model $\Ssfit_\theta$ with a proximal likelihood update, enabling efficient incorporation of likelihood terms through their problem-specific proximal operators. 
We next characterize the discrepancy between $\Tsfit_\eta$ and the ideal posterior drift $-\nabla \log\pi(\xbm|\ybm) \defn \nabla L(\xbm)+\nabla V(\xbm)$.
To proceed, we impose the following transparent assumptions.

\begin{assumption}\label{ass: L and V 1}
The negative log-likelihood $L$ is $\alpha_L$-weakly convex and $\beta_L$-smooth.
The negative log-prior $V$ is $\alpha_V$-weakly convex and $\beta_V$-smooth. In addition, assume
$0<Z:=\int_{\R^d}\exp[-L(\xbm)-V(\xbm)]\,d\xbm<\infty$ and define
$\pi(d\xbm):=Z^{-1}\exp[-L(\xbm)-V(\xbm)]\,d\xbm$.
\end{assumption}
The conditions on the likelihood are standard and hold for many linear and nonlinear imaging models; see \cite{renaud2026provably} for examples. 
Moreover, if $L$ is $\alpha_L$-weakly convex and $0<\eta<1/\alpha_L$, then the proximal operator $\prox_{\eta L}=(I+\eta\nabla L)^{-1}$ is well defined, single valued, and $1/(1-\eta\alpha_L)$-Lipschitz; see Proposition 2 of \cite{gribonval2020characterization}.
We note that the smoothness of the likelihood is commonly required in the analysis of proximal Langevin algorithms~\citep{wibisono2019proximal, salim2019stochastic}.

\begin{assumption}\label{ass: accurate score 2}
Given an arbitrary $\sigma>0$, the score network $\Ssfit_\theta(\xbm;\sigma)$ satisfies:
\begin{enumerate}[label=(\roman*)]
\item There exists $\delta>0$ such that, for all $\xbm \in \R^d$, we have 
$$\|\Ssfit_\theta(\xbm;\sigma)-\nabla V(\xbm)\| \leq \delta;$$
\item There exists $L_\Ssfit>0$ such that for any $\xbm_1,\xbm_2\in \R^d$, 
\begin{equation}
\|\Ssfit_\theta(\xbm_1;\sigma)- \Ssfit_\theta(\xbm_2;\sigma)\| \le L_\Ssfit \|\xbm_1-\xbm_2\|.
\end{equation}
\end{enumerate} 
\end{assumption}
    
\noindent
Assumption~\ref{ass: accurate score 2}(i) allows for imperfectly learned score networks with $\delta$-accuracy to $\nabla V$, while
Assumption~\ref{ass: accurate score 2}(ii) is widely used for convergence analysis with score networks~\citep{yang2022convergence,lee2022convergence,chen2023probability,sun2024provable, renaud2026provably, renaud2026stability}.  
Note that, under Assumptions \ref{ass: L and V 1}--\ref{ass: accurate score 2}, the proximal drift $\Tsfit_\eta$ is naturally $\Lambda_\eta$-Lipschitz
\begin{equation}
\Lambda_\eta\defn\displaystyle L_\Ssfit + \beta_L\frac{1+\eta L_\Ssfit}{1-\eta \alpha_L}.
\end{equation}
See Lemma \ref{lemma: lipschitz drifts} in Appendix \ref{app:tool} for details.
The following lemma summarizes the key approximation property of the drift $\Tsfit_\eta$ used throughout the subsequent analysis.

\begin{lemma}\label{lemma: technical lemma drift}
Let Assumptions~\ref{ass: L and V 1} and~\ref{ass: accurate score 2}(i) hold.
Let $\eta$ be chosen to satisfy $1-\eta\alpha_L>0$ and $1-3\eta^2\beta_L^2>0$.
Then, for any fixed $\sigma>0$ and any $\xbm\in\R^d$,
\begin{equation}\label{eq: technical drift 2}
\|\Tsfit_\eta(\xbm)-\big(\nabla L(\xbm)+\nabla V(\xbm))\|^2\le\frac{2\delta^2}{1-3\eta^2\beta_L^2}+\frac{6\eta^2\beta_L^2}{1-3\eta^2\beta_L^2}\|\nabla L(\xbm)+\nabla V(\xbm)\|^2.
\end{equation}
\end{lemma}
\begin{proof}
See proof in Appendix \ref{app:tool}.
\end{proof}
Lemma~\ref{lemma: technical lemma drift} shows that $\Tsfit_\eta(\xbm)$ approximates the ideal posterior drift $\nabla L(\xbm) + \nabla V(\xbm)$, with the approximation error controlled by the score estimation error $\delta$ and the splitting error induced by the FBS-type treatment.
Adding Brownian motion $\Bt$ to the ODE in \eqref{eq: ode and T} yields the following continuous-time SDE
\begin{equation}\label{eq: sde with T}
    d\xbm_t = -\Tsfit_{\eta}(\xbm_t)dt + \sqrt{2}d\Bt.
\end{equation}
By applying the forward Euler--Maruyama discretization, we obtain the \emph{pro\textbf{X}imal \textbf{MC} (\textbf{X-MC})}
\begin{equation}\label{eq: XMC}
    \Xnpo = \Xn - \gamma \Tsfit_{\eta}(\Xn) + \sqrt{2\gamma}\Wn,
\end{equation}
which serves as the sequential baseline for the proposed {\proposed}, introduced in the next section.
It is also worth noting that, in the exact-score case $\Ssfit_\theta=\nabla V$, a change of variables~\citep{lau2022non} rewrites X-MC in a form similar to the proximal Langevin algorithm~\citep{wibisono2019proximal}; see also~\eqref{Eq:PLA}.
Let $\zbm_n=\prox_{\eta L}(\Xn-\eta\nabla V(\Xn))$.
We have
\begin{equation}
\begin{array}{rl}
&\Xnpo = \Xn - \gamma \Tsfit_{\eta}(\Xn) + \sqrt{2\gamma}\Wn\vspace{0.5ex}\\
\Leftrightarrow\; & (I-\eta\nabla V)^{-1}\circ(I+\eta\nabla L)(\zbm_{n+1}) =\displaystyle\left(1-\frac{\gamma}{\eta}\right)\Xn + \frac{\gamma}{\eta}\zbm_n+\sqrt{2\gamma}\Wn\vspace{0.5ex}\\
\Leftrightarrow\; &(I+\eta\nabla L)(\zbm_{n+1}) =\displaystyle (I-\eta\nabla V)\circ\left[\left(1-\frac{\gamma}{\eta}\right)\Xn + \frac{\gamma}{\eta}\zbm_n+\sqrt{2\gamma}\Wn\right],
\end{array}
\end{equation}
where $(I-\eta \nabla V)^{-1}$ is well-defined by setting $\eta<1/\beta_V$.
Letting $\gamma = \eta$, 
we arrive at
\begin{equation}
\zbm_{n+1} =\displaystyle \prox_{\gamma L}\left(\zbm_n-\gamma \nabla V\bigl(\zbm_n+\sqrt{2\gamma}\Wn\bigr)+\sqrt{2\gamma}\Wn\right).
\end{equation}
This formulation shows that X-MC effectively applies the proximal Langevin update to $\zbm_n$, up to the residual noise term $\sqrt{2\gamma}\Wn$ appearing inside the evaluation of $\nabla V$.
From this perspective, X-MC recovers the proximal Langevin algorithm in this special case, apart from this additional perturbation in the gradient.

\subsection{{\proposed}: Algorithm Formulation and Convergence Analysis}
\label{sec:3.2}

Starting in this section, we introduce the family of {\proposed} algorithms.
Our presentation follows a two-part structure: we first introduce the formulation of each algorithm, and then establish its companion convergence analysis, with additional assumptions and lemmas introduced as needed.
This structure will be maintained in subsequent sections.

\begin{algorithm}[t!]
\caption{{\proposed} (single block)}
\begin{algorithmic}[1]
\setstretch{1.2}
\Require Initial distribution $\mu_0$, drift $\Tsfit_\eta$, stepsize $\gamma>0$, number of time steps $N$, and number of Picard iterations $K$.
\Ensure $\xbm_N$
\State Sample $\xbm_0\sim\mu_0$, and independently of $\xbm_0$, sample $\Wbm_1,\ldots,\Wbm_N\overset{\mathrm{i.i.d.}}{\sim}\mathcal{N}(0,\Ibm)$
\State Set $\Bbm_0\leftarrow 0, \Bbm_{n\gamma}\leftarrow
\sum\limits_{j=1}^n \sqrt{\gamma}\Wbm_j$, for $n=1,\ldots, N$
\State Set $\xbm_n^{(0)}\gets \xbm_0$, for $n=0,\ldots,N$, and $\xbm_{0}^{(k)}\gets \xbm_{0}$ for $k=0,\ldots,K$
\For{$k=0,\dots,K-1$}
\State $\Tsfit_{\eta}(\xbm_{0}^{(k)}),\ldots,\Tsfit_{\eta}(\xbm_{N-1}^{(k)}) \;\leftarrow\; \mathsf{ParallelCompute}(\xbm_{0}^{(k)},\ldots,\xbm_{N-1}^{(k)})$
\For{$n=1,\dots,N$} 
\State $\xbm^{(k+1)}_{n} \;\leftarrow\; \xbm_0-\gamma\sum\limits_{i=0}^{n-1}\Tsfit_{\eta}(\xbm_{i}^{(k)})+\sqrt{2}\Bbm_{n\gamma}$
\EndFor
\EndFor
\State Return $\xbm_N=\xbm_N^{(K)}$
\end{algorithmic}
\label{alg: pix}
\end{algorithm}

\paragraph{Algorithm formulation.}
We apply Picard iteration to parallelize the computation of the X-MC trajectory across time nodes.
Specifically, the Picard iteration for the continuous-time SDE in \eqref{eq: sde with T} is given by
\begin{equation}\label{eq: picard sde}
    \xbm_t^{(k+1)} = \xbm_0-\displaystyle\int_0^t\Tsfit_\eta(\xbm_s^{(k)})ds+\sqrt{2}\Bbm_t,
\end{equation}
where we note that all trajectory iterates share the same realization of the Brownian process $(\Bbm_t)_{t\in[0,T]}$.
Let $t_n=n\gamma$ for $n=0,\ldots,N$, where $N\gamma=T$.
Discretizing \eqref{eq: picard sde} using left-point Riemann sums yields the {\proposed} update
\begin{align}\label{eq: picard discrete}
&\xbm_n^{(0)}=\xbm_0, \quad \xbm_0^{(k)}=\xbm_0, \quad \xbm_n^{(k+1)}=\xbm_0-\gamma\sum_{i=0}^{n-1}\Tsfit_\eta(\xbm_i^{(k)})+\sqrt{2}\Bbm_{n\gamma} \\
& \text{with}\quad n=1,\ldots,N,\quad k=0,\ldots,K.\nonumber
\end{align}
Here, $\Bbm_{n\gamma}=\sum_{i=1}^{n}\sqrt{\gamma}\Wbm_i$ denotes the cumulative sum of the discrete Brownian increments.
Unlike the sequential update rule of X-MC, {\proposed} has the entire previous trajectory $\{\xbm_n^{(k)}\}_{n=0}^{N}$ available when computing $\{\xbm_n^{(k+1)}\}_{n=0}^{N}$. 
Hence, the drift evaluations $\Tsfit_\eta(\xbm_0^{(k)}),\ldots, \\ \Tsfit_\eta(\xbm_{N-1}^{(k)})$ can be computed in parallel across time nodes and accumulated for the next trajectory refinement.
Suppose sufficient parallel computational resources are available such that each time node can be processed concurrently; the computational cost of {\proposed} then no longer scales with $N$ but with the number of Picard iterations $K$.
This is equivalent to viewing the entire trajectory as a single time block that is fully parallelized.
Algorithm~\ref{alg: pix} summarizes the algorithmic details of {\proposed} under this single-block condition.
In \S\ref{sec: 3.3 multi}, we will introduce a multi-block version that better balances computational resource constraints against the degree of parallelization.

\paragraph{Convergence analysis.}
The goal of the presented analysis is to quantify the distance between the evolving probability distribution of the {\proposed} iterates $\xbm_n^{(k)}$ and the target posterior distribution.
To this end, we use the relative Fisher information to measure the distance between two probability distributions.
\begin{definition}
Let $\mu$ and $\pi$ denote two probability distributions such that $\mu\ll\pi$.
The Fisher information (FI) of $\mu$ relative to $\pi$ is defined as
\begin{equation}
    \FI(\mu\|\pi)=\int \left\|\nabla\log \frac{d\mu}{d\pi}\right\|^2 d\mu.
\end{equation}
\end{definition}
If $\mu$ and $\pi$ have positive and smooth densities, $\FI(\mu \,\|\, \pi) = 0$ if and only if $\mu = \pi$, motivating its use as a criterion for measuring convergence between distributions.
From an optimization perspective, Langevin dynamics can be viewed as a gradient flow in the space of probability distributions equipped with the Wasserstein metric, with FI playing the role of the squared gradient norm.
Thus, relative FI provides a stationarity criterion for the algorithm.

\begin{lemma}\label{lemma: error decay no block maintext}
Let $\mu_0$ be the initial distribution of $\xbm_0$ and suppose Assumptions \ref{ass: L and V 1}-\ref{ass: accurate score 2} hold.
Set the parameters $\eta, \gamma>0$ sufficiently small such that
\begin{equation}
    1-\eta \alpha_L>0, \quad 1-3\eta^2\beta_L^2>0, \quad q_{\mathsf{sing}} \defn \gamma N\Lambda_\eta<1.
\end{equation}
Define the error across two successive Picard refinements at the $k$-th iteration as
\begin{equation}\varepsilon_k\defn\max\limits_{n=1,2,\dots,N}\E\left[\|\xbm^{(k)}_{n}-\xbm^{(k-1)}_{n}\|^2\right].
\end{equation}
Then, the Picard iteration error for {\proposed} satisfies
\begin{equation}
\varepsilon_K\lesssim q_{\mathsf{sing}}^{2K-2}\bigl(1+\FI(\mu_0\Vert\pi)\bigr).
\end{equation}
\end{lemma}
\begin{proof}
See Appendix \ref{app:thm1} for the detailed proof and range of the hyperparameters.
\end{proof}
Lemma~\ref{lemma: error decay no block maintext} shows that the Picard iteration error $\varepsilon_K$ decays geometrically with the total number of Picard iterations.
As shown by the numerical results in Fig.~\ref{fig:decay curves} and \S\ref{sec:experiments}, this geometric rate leads to rapid stabilization of the trajectories within only a few iterations.
We next present the main theorem for {\proposed}.

\begin{theorem}
\label{thm:main-no-block-simple}
Let $(\mu_t)_{t\geq0}$ denote the law of the continuous interpolation of $\{\Xn^{(K)}\}_{n=0}^{N}$ generated by
{\proposed}, where $N > 0$ is the total number of time steps. 
Assume Assumptions~\ref{ass: L and V 1}-\ref{ass: accurate score 2} hold, and the parameters $\eta,\gamma>0$ are chosen sufficiently small.
Then, the following inequality holds for {\proposed}:
\begin{equation}
\begin{array}{rl}
&\displaystyle\Big(0.75-\mathcal O(\eta^2)-\mathcal O(\gamma^2)\Big)\;\frac{1}{T}\int_0^T\FI(\mu_t\Vert\pi)dt \vspace{0.5ex} \\
\leq&\displaystyle\frac{1}{T}\KL(\mu_0\Vert\pi)+\underbrace{\mathcal O(\eta^2)}_{\substack{\text{Splitting}\\ \text{discrepancy}}}
+\underbrace{\mathcal O(\delta^2)}_{\substack{\text{Score mismatch}\\\text{error}}}+\underbrace{\mathcal O(\gamma)}_{\substack{\text{Discretization}\\\text{error}}}+\underbrace{\mathcal{O}(q_{\mathsf{sing}}^{2K-2})\big(1+\FI(\mu_0\Vert\pi)\big)}_{\substack{\text{Finite Picard}\\\text{refinement error}}}.
\end{array}
\end{equation}
\end{theorem}
\begin{proof}
See Appendix \ref{app:thm1} for the detailed proof and range of the hyperparameters.
\end{proof}
Theorem \ref{thm:main-no-block-simple} establishes a finite-time stationarity guarantee for PiX-MC.
The theorem also explicitly quantifies the errors stemming from the splitting discrepancy $\mathcal{O}(\eta^2)$, score mismatch $\mathcal{O}(\delta^2)$, and discretization error $\mathcal{O}(\gamma)$, which are inherited from X-MC.
The additional finite Picard-refinement error $\mathcal{O}(q_{\mathsf{sing}}^{2K-2})\big(1+\FI(\mu_0\Vert\pi)\big)$ is specific to {\proposed}.
When $\eta,\gamma$ are chosen sufficiently small such that the coefficient $\left(0.75-\mathcal{O}(\eta^2)-\mathcal{O}(\gamma^2)\right)$ is positive and $q_{\mathsf{sing}}<1$, the bound is nontrivial, and the finite Picard refinement error decays geometrically with $K$ according to Lemma~\ref{lemma: error decay no block maintext}.
Additionally, as $\FI(\cdot\|\pi)$ is convex, Jensen's inequality yields
\begin{equation}
\FI(\bar{\mu}_T\|\pi)\le\displaystyle\frac{1}{T}\int_0^T\FI(\mu_t\|\pi)dt,
\end{equation}
where $\bar{\mu}_T=(1/T)\int_0^T\mu_t\,dt$ is the time-averaged marginal law.
A possible approach for drawing samples from this distribution has been discussed in~\citep{balasubramanian2022towards}.
Theorem \ref{thm:main-no-block-simple} therefore also bounds the stationarity gap of $\bar{\mu}_T$ relative to $\pi$.

\subsection{Multi-Block {\proposed}: Partitioning the Trajectory into Multiple Blocks}
\label{sec: 3.3 multi}

As discussed in the previous section, the vanilla {\proposed} treats the entire trajectory as a single block to be fully parallelized, offering the largest possible degree of parallelism. Yet this design comes with two practical trade-offs.
First, the degree of parallelism interacts with the choice of hyperparameters. 
Recall from Lemma~\ref{lemma: error decay no block maintext} and Theorem~\ref{thm:main-no-block-simple} that controlling the Picard error requires $q_{\mathsf{sing}}<1$. 
Since $q_{\mathsf{sing}}=T\Lambda_\eta$, this contraction condition directly restricts the time horizon that can be handled in a single block.
Second, refining the full trajectory at once is more demanding on computational resources, as it requires storing and updating all time nodes simultaneously, which can be memory-intensive for large-scale imaging problems.
To better balance parallelism with these practical considerations, we introduce \emph{\textbf{multi-block {\proposed}}} that partitions the trajectory into multiple time blocks.

\paragraph{Algorithm formulation.}
We partition the full interval $[0,T]$ into $M$ successive blocks, each spanning $N$ time steps, so that
\begin{equation}
T=MN\gamma,\qquad\tau=N\gamma,
\end{equation}
where $\tau$ denotes the length of each block.
Starting from the endpoint of the preceding block, multi-block {\proposed} then performs $K$ Picard refinements only within the current block, rather than over the full trajectory.
As a result, both the trajectory storage and the concurrent drift evaluations scale with the block size $N$ rather than the total number of time steps $MN$, allowing $N$ to be chosen according to the available GPU memory and parallel computing resources.
This block-wise refinement also relaxes the condition required for convergence: the local Picard contraction factor becomes $q_{\mathsf{mult}}\defn\tau\Lambda_\eta$.
Note that when $M>1$, $q_{\mathsf{mult}}=(\tau/T)q_{\mathsf{sing}}<q_{\mathsf{sing}}$, the multi-block scheme admits a larger stepsize than its single-block counterpart.
Algorithm~\ref{alg: pix multi block} summarizes the complete procedure of multi-block {\proposed}.
As shown, parallelization occurs within each time block, while the algorithm proceeds sequentially across blocks.

\begin{algorithm}[t!]
\caption{\mbproposed}
\label{alg: pix multi block}
\begin{algorithmic}[1]
\setstretch{1.2}
\Require Initial distribution $\mu_0$, drift $\Tsfit_\eta$, stepsize $\gamma>0$, number of blocks $M$, time step number $N$ inside each block, and Picard iteration number $K$.
\State Sample $\xbm_0\sim \mu_0$ and, independently of $\xbm_0$, sample $\Wbm_1,\ldots,\Wbm_{MN}\overset{\mathrm{i.i.d.}}{\sim}\mathcal{N}(0,\Ibm)$
\State Set $\Bbm_0\leftarrow 0$, and $\Bbm_{i\gamma}\leftarrow \sum\limits_{j=1}^i \sqrt{\gamma} \Wbm_j$ for $i=1,\ldots,MN$.
\For{$m=0,\dots,M-1$}
\State Set $\xbm_{mN+i}^{(0)}\gets \xbm_{mN}$ for $i=1,\ldots,N$, and $\xbm_{mN}^{(k)}\gets \xbm_{mN}$ for $k=0,\ldots,K$
\For{$k=0,\dots,K-1$}
\State $\Tsfit_\eta(\xbm_{mN}^{(k)}),\ldots,\Tsfit_\eta(\xbm_{(m+1)N-1}^{(k)}) \;\leftarrow\; \mathsf{ParallelCompute}(\xbm_{mN}^{(k)},\ldots,\xbm_{(m+1)N-1}^{(k)})$
\For{$i=1,\ldots,N$}
\State $\xbm_{mN+i}^{(k+1)}=\xbm_{mN}-\gamma\sum\limits_{j=0}^{i-1}\Tsfit_\eta(\xbm_{mN+j}^{(k)})+\sqrt{2}\left(\Bbm_{(mN+i)\gamma}-\Bbm_{mN\gamma}\right)$
\EndFor
\EndFor
\State $\xbm_{(m+1)N}\gets \xbm_{(m+1)N}^{(K)}$
\EndFor
\State Return $\xbm_{MN}=\xbm_{MN}^{(K)}$
\end{algorithmic}
\end{algorithm}

\paragraph{Convergence analysis.}
The convergence analysis of multi-block {\proposed} requires an additional stability control compared with the single-block setting.
In single-block {\proposed}, every Picard trajectory is initialized at the same point $\xbm_0$, and the initial Picard residual can be controlled directly through the initial Fisher information.
In contrast, the starting point $\xbm_{mN}$ of the $m$-th block is itself a random variable, determined by the output of all preceding blocks.
Therefore, we impose the following stability assumption on the block starting points, which allows the initial Picard residual to be controlled uniformly across blocks.

\begin{assumption}\label{ass: multi-block stability}
The block starting points have uniformly bounded second moments. 
Specifically, there exists a constant $C_\xbm>0$, independent of the number of blocks $M$ and the total time horizon $T$, such that
\begin{equation}
\sup\limits_{0 \leq m \leq M-1} \mathbb{E}\left\|\xbm^{(0)}_{m N}\right\|^2 \leq C_\xbm.
\end{equation}
\end{assumption}
Assumption \ref{ass: multi-block stability} is a standard condition for Langevin-based methods, and can be implied by coercivity of the posterior potential or dissipativity of the drift; see e.g., \citep{raginsky2017non, renaud2026stability}.
We also provide an example condition for Assumption \ref{ass: multi-block stability} to hold in Appendix~\ref{app:diss}.
This assumption is consistent with practical imaging settings, where the image intensity typically lies in a bounded range such as $[0,1]^d$.
Under Assumption \ref{ass: multi-block stability}, the block-wise Picard initialization remains uniformly controlled over the sampling horizon.
We next establish a result analogous to Lemma~\ref{lemma: error decay no block maintext} for multi-block {\proposed}.

\begin{lemma}
\label{lemma: error decay no LSI maintext}
Let Assumptions~\ref{ass: L and V 1}--\ref{ass: multi-block stability} hold, and set the parameters $\eta,\gamma>0$ sufficiently small such that
\begin{equation}
    1-\eta \alpha_L>0, \quad 1-3\eta^2\beta_L^2>0, \quad q_{\mathsf{mult}} = \gamma N\Lambda_\eta<1.
\end{equation}
For each block $m=0,\ldots,M-1$, define the in-block Picard error as
\begin{equation}\label{eq:picarderror1}
\varepsilon_{m,k}\defn\max_{i=1,\ldots,N}\E\left[\|\xbm_{mN+i}^{(k)}-\xbm_{mN+i}^{(k-1)}\|^2\right].
\end{equation}
Then, for multi-block {\proposed}, the in-block Picard error satisfies $\varepsilon_{m,k+1}\le q_{\mathsf{mult}}^2\,\varepsilon_{m,k}$.
Consequently, for each block $m=0,\ldots,M-1$,
\begin{equation}
\varepsilon_{m,K}\le q_{\mathsf{mult}}^{2K-2}\left[2\tau d+\displaystyle\frac{3\tau^2(\delta^2+G^2)}{1-3\eta^2\beta_L^2}\right]=q_{\mathsf{mult}}^{2K-2}R.
\end{equation}
Here $G>0$ is a constant independent of the total time $T$; see Appendix \ref{app:C} for details.
\end{lemma}
\begin{proof}
See Appendix \ref{app:C} for the detailed proof and explicit expressions for the constants.
\end{proof}
Lemma \ref{lemma: error decay no LSI maintext} establishes geometric Picard contraction within each local block.
We are now ready to present the convergence result for multi-block {\proposed}.

\begin{theorem}\label{thm: main no LSI simple}
Let $(\mu_t)_{t\geq0}$ denote the law of the continuous interpolation of $\{\Xn^{(K)}\}_{n=0}^{MN}$ generated by multi-block {\proposed}, where $N > 0$ denotes the number of time steps in each time block and $M>0$ is the number of blocks.
Assume Assumptions~\ref{ass: L and V 1}-\ref{ass: multi-block stability} hold and choose sufficiently small parameters $\eta,\gamma>0$.
Then, the following inequality holds for multi-block {\proposed}:
\begin{equation}
\begin{array}{rl}
    &\Big(0.75-\mathcal O(\eta^2)-\mathcal O(\gamma^2)\Big)\;\displaystyle\frac{1}{T}\int_{0}^{T}\FI(\mu_t\|\pi)dt \\
    \le &\displaystyle\frac{\KL(\mu_{0}\|\pi)}{T} + 
    \underbrace{\mathcal O(\eta^2)}_{\substack{\text{Splitting}\\ \text{discrepancy}}}
    +
    \underbrace{\mathcal O(\delta^2)}_{\substack{\text{Score mismatch}\\\text{error}}}+\underbrace{\mathcal O(\gamma)}_{\substack{\text{Discretization}\\\text{error}}}+\underbrace{\mathcal{O}(q_{\mathsf{mult}}^{2K-2})\,R}_{\substack{\text{Finite Picard}\\\text{refinement error}}}.
\end{array}
\end{equation}
where $R$ is the constant from Lemma~\ref{lemma: error decay no LSI maintext}.
\end{theorem}
\begin{proof}
See Appendix \ref{app:C} for the detailed proof and range of the hyperparameters.
\end{proof}
Theorem~\ref{thm: main no LSI simple} extends the single-block guarantee to the multi-block case by replacing the global contraction coefficient with the in-block $q_{\mathsf{mult}}<1$.
Note that $q_{\mathsf{mult}}$ is determined by the local time horizon $\tau$ rather than the total sampling time $T$, and the finite Picard error still decays geometrically within each block.
Compared with the result in Theorem~\ref{thm:main-no-block-simple}, the associated constant is now controlled by $R$ established in Lemma~\ref{lemma: error decay no LSI maintext}, rather than by the initial Fisher information.

\subsection{Multi-Block Annealed {\proposed}: Improving Sample Quality via Annealing}\label{sec:3.4}

{\proposed} has so far used a fixed noise level in the score network $\Ssfit_\theta$.
In this section, we introduce \emph{multi-block annealed {\proposed} (\textbf{{\ambproposed}})} that employs a decreasing noise schedule to first capture coarse image structures and then progressively refine details.

\paragraph{Algorithm formulation.}
Let $n=mN+i$ denote the global time node index, where $m=0,\ldots,M-1$ is the block index and $i=0,\ldots,N-1$ is the local node index.
We use the geometric schedule
\begin{equation}\label{eq: noise schedule}
\sigma_n=\max\{\sigma_{\max}\xi^n,\sigma_{\min}\},
\end{equation}
where $0<\xi<1$ is the annealing parameter.
We further associate $\{\sigma_n\}_{n=0}^{MN-1}$ with a non-increasing score-weight schedule $\{\alpha_n\}_{n=0}^{MN-1}$ satisfying $\alpha_n\geq1$ and $\alpha_n\to 1$ when $n \to \infty$.
At the $n$-th time node, the corresponding weighted proximal-likelihood drift is defined as
\begin{equation}
\Tsfit_{\eta, \sigma_n}^{\alpha_n}\defn\frac{1}{\eta}\left[I-\prox_{\eta L} \circ\left(I-\eta \alpha_n \Ssfit_\theta\left(\cdot; \sigma_n\right)\right)\right].
\end{equation}
Rather than restarting the annealing schedule at each block boundary, the global indexing in \eqref{eq: noise schedule} allows the noise level and score weight to decrease continuously over the entire sampling trajectory.
Within the $m$-th block, the local Picard iteration thus uses the time-dependent drifts $\Tsfit^{\alpha_{mN}}_{\eta,\sigma_{mN}},\ldots,\Tsfit^{\alpha_{(m+1)N-1}}_{\eta,\sigma_{(m+1)N-1}}$, while retaining the same block-wise trajectory structure as multi-block {\proposed}.
The resulting algorithm, {\ambproposed}, is summarized in Algorithm~\ref{alg: pix annealing multi-block}.

\begin{algorithm}[t!]
\caption{{\ambproposed}}
\label{alg: pix annealing multi-block}
\begin{algorithmic}[1]
\setstretch{1.2}
\Require Initial distribution $\mu_0$, stepsize $\gamma>0$, number of blocks $M$, time step number $N$ inside each block, drift $\Tsfit^{\alpha}_{\eta,\sigma}$ with denoising strength schedule $\{\sigma_i\}_{i=0}^{MN-1}$ and score-weight schedule $\{\alpha_i\}_{i=0}^{MN-1}$, and Picard iteration number $K$.
\State Sample $\xbm_0\sim \mu_0$ and, independently of $\xbm_0$, sample $\Wbm_1,\ldots,\Wbm_{MN}\overset{\mathrm{i.i.d.}}{\sim}\mathcal{N}(0,\Ibm)$
\State Set $\Bbm_0\leftarrow 0$, and $\Bbm_{i\gamma}\leftarrow \sum\limits_{j=1}^i \sqrt{\gamma} \Wbm_j$ for $i=1,\ldots,MN$.
\For{$m=0,\dots,M-1$}
\State Set $\xbm_{mN+i}^{(0)}\gets \xbm_{mN}$ for $i=1,\ldots,N$, and $\xbm_{mN}^{(k)}\gets \xbm_{mN}$ for $k=0,\ldots,K-1$
\For{$k=0,\dots,K-1$}
\State $\Tsfit^{\alpha_{mN}}_{\eta,\sigma_{mN}}(\xbm_{mN}^{(k)}),\ldots,\Tsfit^{\alpha_{(m+1)N-1}}_{\eta,\sigma_{(m+1)N-1}}(\xbm_{(m+1)N-1}^{(k)}) \leftarrow \mathsf{ParallelCompute}(\xbm_{mN}^{(k)},\ldots,\xbm_{(m+1)N-1}^{(k)})$
\For{$i=1,\ldots,N$}
\State $\xbm_{mN+i}^{(k+1)}=\xbm_{mN}-\gamma\sum\limits_{j=0}^{i-1}\Tsfit^{\alpha_{mN+j}}_{\eta,\sigma_{mN+j}}(\xbm_{mN+j}^{(k)})+\sqrt{2}\left(\Bbm_{(mN+i)\gamma}-\Bbm_{mN\gamma}\right)$
\EndFor
\EndFor
\State $\xbm_{(m+1)N}\gets \xbm_{(m+1)N}^{(K)}$
\EndFor
\State Return $\xbm_{MN}=\xbm_{MN}^{(K)}$
\end{algorithmic}
\end{algorithm}

\paragraph{Convergence analysis.}
The convergence analysis follows the same overall argument as that of multi-block {\proposed}, with modifications accounting for the time-varying denoising-strength schedule $\sigma_n$ and score-weight schedule $\alpha_n$.
Since the proximal drift $\Tsfit_{\eta,\sigma}^{\alpha}$ now varies across time nodes, we replace the single score-mismatch error $\delta$ with the node-dependent error and adjust the uniform Lipschitz bound of the proximal drift accordingly.
To control these effects, we replace Assumption \ref{ass: accurate score 2} with the following annealed counterpart.
\begin{assumption}\label{ass:4ann}
    Given the schedules $\{\sigma_n\}_{n=0}^{MN-1}$ and $\{\alpha_n\}_{n=0}^{MN-1}$, there exist $\{\delta_n\}_{n=0}^{MN-1}$, $L_\Ssfit>0$, and $R_\Ssfit>0$ such that the score network $\Ssfit_\theta$ satisfies:
    \begin{enumerate}[label=(\roman*)]
        \item for every $n=0,\ldots,MN-1$ and $\xbm\in\R^d$,
        \begin{equation}
            \left\|\Ssfit_\theta\left(\xbm ; \sigma_n\right)-\nabla V(\xbm)\right\| \leq \delta_n ;
        \end{equation}

        \item for every $n=0,\ldots,MN-1$ and $\xbm, \zbm\in\R^d$,
        \begin{equation}
            \left\|\Ssfit_\theta\left(\xbm ; \sigma_n\right)-\Ssfit_\theta\left(\zbm ; \sigma_n\right)\right\| \leq L_\Ssfit\|\xbm-\zbm\| ;
        \end{equation}

        \item for every $n=0,\ldots,MN-1$ and $\xbm\in\R^d$,
        \begin{equation}
            \left\|\Ssfit_\theta\left(\xbm ; \sigma_n\right)\right\| \leq R_\Ssfit .
        \end{equation}
    \end{enumerate}
\end{assumption}
Assumption 4(iii) controls the additional discrepancy induced by the score weights. Indeed,
\begin{equation}
\left\|\alpha_n \Ssfit_\theta\left(\xbm; \sigma_n\right)-\nabla V(\xbm)\right\| \leq \delta_n+\left(\alpha_n-1\right) R_\Ssfit.
\end{equation}
Accordingly, define the maximal and time-averaged annealed-score errors by
\begin{equation}
\delta_{\mathsf{ann}}\defn\max\limits_{0\leq n \leq M N-1}\left\{\delta_n+\left(\alpha_n-1\right) R_\Ssfit\right\}, \quad \bar{\delta}_{\mathsf{ann}}^2\defn\frac{1}{MN} \sum\limits_{n=0}^{MN-1}\left[\delta_n+\left(\alpha_n-1\right) R_\Ssfit\right]^2 .
\end{equation}
The maximal score error $\delta_{\mathsf{ann}}$ controls the in-block Picard error, whereas $\bar\delta_{\mathsf{ann}}^2$ appears in the time-averaged convergence bound.
Since $\{\alpha_n\}_{n=0}^{MN-1}$ is non-increasing, the proximal drifts are uniformly $\Lambda_{\eta,\mathsf{ann}}$-Lipschitz, where
\begin{equation}
\Lambda_{\eta, \mathsf{ann}}\defn\alpha_0L_\Ssfit+\beta_L\frac{1+\eta\alpha_0L_\Ssfit}{1-\eta\alpha_L}, \quad q_{\mathsf{ann}}\defn\tau \Lambda_{\eta, \mathsf{ann}}.
\end{equation}
Here, $q_{\mathsf{ann}}$ is the uniform in-block Picard contraction factor for {\ambproposed}.
The following theorem states the convergence result for {\ambproposed}.

\begin{theorem}\label{thm: multi-block annealing simple}
Let $(\mu_t)_{t\geq0}$ denote the law of the continuous interpolation of $\{\Xn^{(K)}\}_{n=0}^{MN}$ generated by
{\ambproposed}, where $N > 0$ denotes the number of time steps in each time block and $M>0$ is the number of blocks.
Assume Assumptions~\ref{ass: L and V 1}, \ref{ass: multi-block stability}, and \ref{ass:4ann} hold, and choose sufficiently small parameters $\eta,\gamma>0$.
Then, the following inequality holds for multi-block {A\proposed}:
\begin{equation}
\begin{array}{rl}
    &\Big(0.75-\mathcal O(\eta^2)-\mathcal O(\gamma^2)\Big)\;\displaystyle\frac{1}{T}\int_{0}^{T}\FI(\mu_t\|\pi)dt \\
    \le &\displaystyle\frac{\KL(\mu_{0}\|\pi)}{T} + 
    \underbrace{\mathcal O(\eta^2)}_{\substack{\text{Splitting}\\ \text{discrepancy}}}
    +
    \underbrace{\mathcal O(\bar{\delta}_\mathsf{ann}^2)}_{\substack{\text{Annealed score }\\\text{mismatch error}}}+\underbrace{\mathcal O(\gamma)}_{\substack{\text{Discretization}\\\text{error}}}+\underbrace{\mathcal{O}(q_{\mathsf{ann}}^{2K-2})\,R_{\mathsf{ann}}}_{\substack{\text{Finite Picard}\\\text{refinement error}}}.
\end{array}
\end{equation}
where $R_{\mathsf{ann}}=2\tau d+\frac{3\tau^2\left(\delta_{\mathsf{ann}}^2+G^2\right)}{1-3 \eta^2 \beta_L^2}$.
\end{theorem}
\begin{proof}
See Appendix \ref{app:D} for the detailed proof and full expression of the terms.
\end{proof}
Theorem \ref{thm: multi-block annealing simple} extends the multi-block guarantee in Theorem \ref{thm: main no LSI simple} to the weighted annealing setting.
Compared with multi-block PiX-MC, the score-mismatch term is replaced by the time-averaged annealed-score discrepancy $\bar{\delta}_{\mathsf{ann}}^{2}$, while the finite Picard-refinement error is governed by the uniform contraction factor $q_{\mathsf{ann}}$.
In particular,
\begin{equation}
\bar{\delta}_{\mathsf{ann}}^2 \leq \frac{2}{MN}\sum\limits_{n=0}^{MN-1} \delta_n^2+\frac{2R_\Ssfit^2}{MN} \sum\limits_{n=0}^{MN-1}(\alpha_n-1)^2 .
\end{equation}
The first term represents the approximation error of score networks, whereas the second term is induced by the score weights.
If $\alpha_n=1$ after a finite annealing stage, which is true in our experiments, the latter contribution decreases as $MN$ increases.
Meanwhile, with $q_{\mathsf{ann}}<1$, the finite Picard-refinement error decays geometrically with $K$.
More broadly, annealing and Picard refinement act on complementary aspects of the computation: annealing improves the sequential trajectory itself, reducing the effective sequential effort required to attain a given reconstruction quality, whereas Picard refinement reduces the wall-clock cost of evaluating that trajectory through time parallelism.
Their gains can therefore compound in practice.
As demonstrated in \S\ref{sec:experiments}, {A\proposed} yields additional wall-clock acceleration beyond trajectory parallelization.

\section{Numerical Validation of the Theory}\label{sec:validations}

We consider a compressed sensing problem whose target posterior is available in closed form. 
This setting allows us to directly compare the empirical distributions generated by {\proposed} algorithms with the true posterior, providing direct empirical support for the convergence results established in our theorems.
As the single-block implementation of {\proposed} is computationally demanding, we implement all variants of {\proposed} using their multi-block schemes, namely Algorithms~\ref{alg: pix multi block}--\ref{alg: pix annealing multi-block}. For brevity, we omit the qualifier "multi-block" hereafter.
All {\proposed} algorithms use eight NVIDIA RTX PRO 6000 Blackwell GPUs.

\paragraph{Experimental Setup.}
We construct a Gaussian image prior
\begin{equation}p(\xbm)=\mathcal{N}\bigl(\xbm;\mu_{\textsf{prior}},\Sigma_{\textsf{prior}}\bigr),
\end{equation}
where $\mu_{\textsf{prior}}$ and $\Sigma_{\textsf{prior}}$ are estimated from 70,000 FFHQ images \citep{karras2019style}, after converting the images to grayscale and resizing them to $32\times32$.
A ground-truth image is drawn according to
$\xbm^\ast\sim\mathcal{N}\bigl(\mu_{\textsf{prior}},\Sigma_{\textsf{prior}}\bigr)$.
We then generate the linear measurements $\ybm = \Abm \xbm^\ast + \ebm$, where $\Abm\in\R^{100\times1024}$ is a Gaussian random matrix whose rows are normalized to unit norm, and $\ebm\sim\mathcal{N}\bigl(\zerobm,\sigma_y^2\Ibm\bigr)$, with $\sigma_y=0.001$.
Under this Gaussian prior and likelihood, the posterior distribution is also Gaussian,
\begin{equation}
    \pi(\xbm|\ybm)=\mathcal{N}\bigl(\xbm;\mu_{\textsf{post}},\Sigma_{\textsf{post}}\bigr),
\end{equation}
where
\begin{equation}\Sigma_{\textsf{post}}=\left(\Sigma_{\textsf{prior}}^{-1}+\frac{1}{\sigma_y^2}\Abm^\top\Abm\right)^{-1}, \qquad \mu_{\textsf{post}}=\Sigma_{\textsf{post}}\left(\Sigma_{\textsf{prior}}^{-1}\mu_{\textsf{prior}}+\frac{1}{\sigma_y^2}\Abm^\top\ybm\right).
\end{equation}
We train a score network following the EDM framework  \citep{karras2022elucidating} by using samples from the empirical Gaussian prior. 
We evaluate three instantiations of {\proposed}: PiX-MC and APiX-MC equipped with the learned score network, and PiX-MC with the analytical score $\nabla V(\xbm)=\Sigma_{\textsf{prior}}^{-1}\bigl(\xbm-\mu_{\textsf{prior}}\bigr)$.
Comparing {\proposed} and {\aproposed} with the learned score network additionally illustrates the effect of annealing on the evolution of the sampled distribution.

\begin{figure}[t!]
\centering
\includegraphics[width=0.95\textwidth]{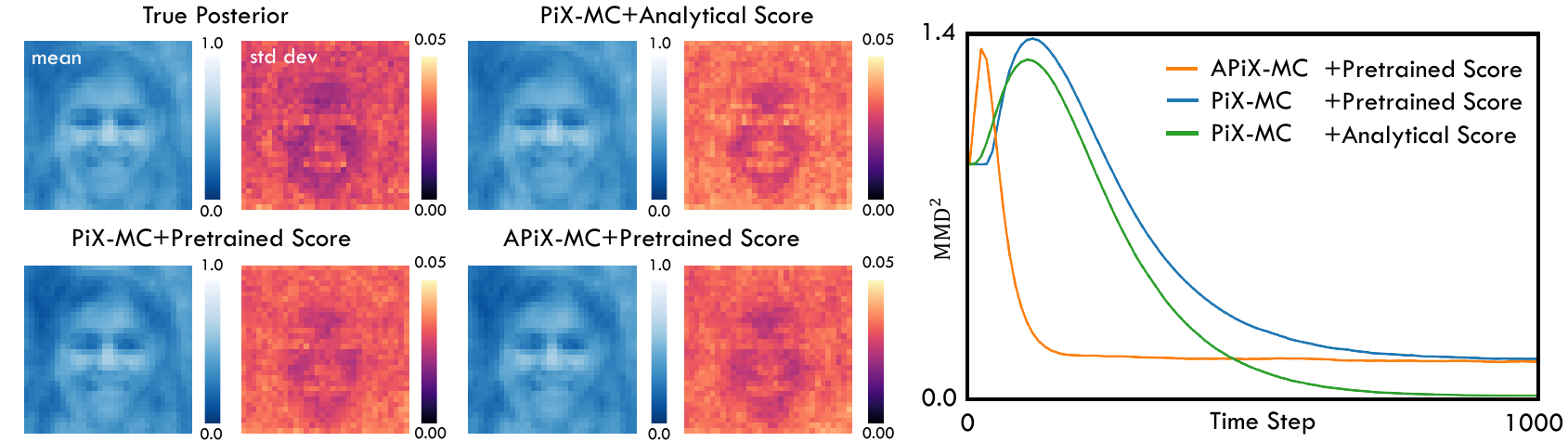}
\caption{
Distribution-level comparison of PiX-MC and APiX-MC against the closed-form Gaussian posterior. 
Each method runs 150 independent chains, with samples recorded every 10 time steps. 
The left panels show the posterior mean and pixel-wise standard deviation at the final time step, with the exact posterior shown for reference. 
The right panels plot the squared MMD to the exact posterior over the sampling trajectory. 
Note how the analytical-score variant achieves the lowest $\mathsf{MMD}^2$, while APiX-MC reduces $\mathsf{MMD}^2$ more rapidly than PiX-MC with the learned score.
}
\label{fig:face}
\end{figure}

\paragraph{Distributional Evaluation.}
For each algorithm, we run 150 independent chains and record the samples every 10 time steps. 
We also draw 150 independent samples directly from the closed-form posterior.
At the final time step, we compute the empirical posterior mean and pixelwise standard deviation from the samples.
To assess the full distribution beyond these first two moments, we compute the squared maximum mean discrepancy (MMD) \citep{jayasumana2024rethinking} at every recorded time step.
Given samples $\widetilde{\Xbf}=\{\widetilde{\xbm}_1,\ldots,\widetilde{\xbm}_m\}$ and exact posterior samples $\Xbf=\{\xbm_1,\ldots,\xbm_n\}$, we use the empirical estimator
\begin{equation}
\mathsf{MMD}^2(\widetilde{\Xbf},\Xbf)=\frac{1}{m^2}\sum_{i=1}^{m}\sum_{j=1}^{m}k(\widetilde{\xbm}_i,\widetilde{\xbm}_j)+\frac{1}{n^2}\sum_{i=1}^{n}\sum_{j=1}^{n}k(\xbm_i,\xbm_j)-\frac{2}{mn}\sum_{i=1}^{m}\sum_{j=1}^{n}k(\widetilde{\xbm}_i,\xbm_j).
\end{equation}
We use the Gaussian kernel $k(\xbm_1,\xbm_2)=\exp\left(-\frac{\|\xbm_1-\xbm_2\|_2^2}{2\sigma_{\textsf{MMD}}^2}\right)$ with $\sigma_{\textsf{MMD}}=3.0.$

\paragraph{Distribution-Level Results}
Figure~\ref{fig:face} reports the empirical posterior mean and pixelwise standard deviation at the final time step, together with the evolution of $\mathsf{MMD}^2$ over the sampling trajectory. All three {\proposed} instantiations accurately recover the posterior mean and uncertainty maps. As expected, {\proposed} with the analytical score most closely matches the true posterior, owing to the absence of score approximation error.
The $\mathsf{MMD}^2$ curves show that the empirical distributions produced by all three methods move toward the target posterior, providing empirical support for the finite-horizon stationarity bounds established in Theorems~\ref{thm: main no LSI simple}--\ref{thm: multi-block annealing simple}. 
At the final time step, {\proposed} with the analytical score achieves the lowest $\mathsf{MMD}^2$ of $0.0036$, while {\aproposed} and {\proposed} with the learned score attain $0.1361$ and $0.1480$, respectively. The remaining gap between the learned-score variants and the analytical variant is consistent with the score approximation error analyzed in Section~\ref{sec:method}.
Compared with {\proposed}, {\aproposed} reduces $\mathsf{MMD}^2$ more rapidly and reaches a lower final value, corroborating the discussion in Section~\ref{sec:3.4} that annealing improves the underlying sampling trajectory while Picard refinement enables its parallel evaluation.

\section{Real-World Imaging Experiments}\label{sec:experiments}
We now present our experimental results, which are organized around three goals.
First, we verify the behavior of Picard refinement on MRI and Rician noise removal, covering both linear and nonlinear inverse problem setups.
Second, we evaluate practical acceleration on large-scale $1024 \times 1024$ image deblurring, where we vary the GPU budget and the block size to examine how the multi-block implementation affects wall-clock runtime. 
Third, we apply {\proposed} to a large-scale $512 \times 512 \times 80$ sparse-view CT problem to demonstrate its scalability on high-dimensional imaging tasks.
All hyper-parameters are tuned to maximize the \textit{peak signal-to-noise ratio (PSNR)}. Additional implementation details and hyper-parameter settings are provided in Appendix \ref{app:ex}.

\paragraph{Experimental Setup.}
All experiments are conducted on a single compute node equipped with eight NVIDIA RTX PRO 6000 Blackwell GPUs and two AMD EPYC 9334 32-Core processors. All Picard-based methods are implemented using the multi-block scheme and executed on all eight GPUs, whereas sequential methods are executed on a single GPU.
We use DPS \citep{chung2022diffusion} and DAPS \citep{zhang2025improving} as external baselines. For the deblurring and 3D CT experiments, we additionally include sequential and ablation variants to isolate the effects of Picard time parallelization, the proximal-likelihood drift, and the annealing strategy.
We use the following nomenclature throughout this section. L-MC denotes the standard sequential Langevin sampler with score network $\Ssfit_\theta$ (Eq.~\eqref{Eq:LMC}), and X-MC denotes its proximal-likelihood counterpart (Eq.~\eqref{eq: XMC}). PiL-MC is obtained by applying Picard time parallelization to L-MC by replacing the drift $\Tsfit_\eta$ in Algorithm~\ref{alg: pix multi block} with $\Ssfit_\theta+\nabla L$. The prefix ``A'' denotes the corresponding annealed variants, yielding AL-MC, AX-MC, and APiL-MC.
Unless otherwise specified, all methods use the same score network for each imaging task. Task-specific score network choices and multi-block configurations are described in the corresponding subsections, while complete hyperparameter settings are provided in Appendix~\ref{app:parameters}.

\subsection{Behavior of Picard Refinement}

We first present results on linear MRI reconstruction and nonlinear Rician denoising. These experiments empirically verify the geometric decay of the Picard error predicted by Lemmas~\ref{lemma: error decay no block maintext} and~\ref{lemma: error decay no LSI maintext}, demonstrating that the predicted convergence behavior is agnostic to the linearity of the inverse problem.

\begin{figure}[t!]
\centering
\includegraphics[width=\textwidth]{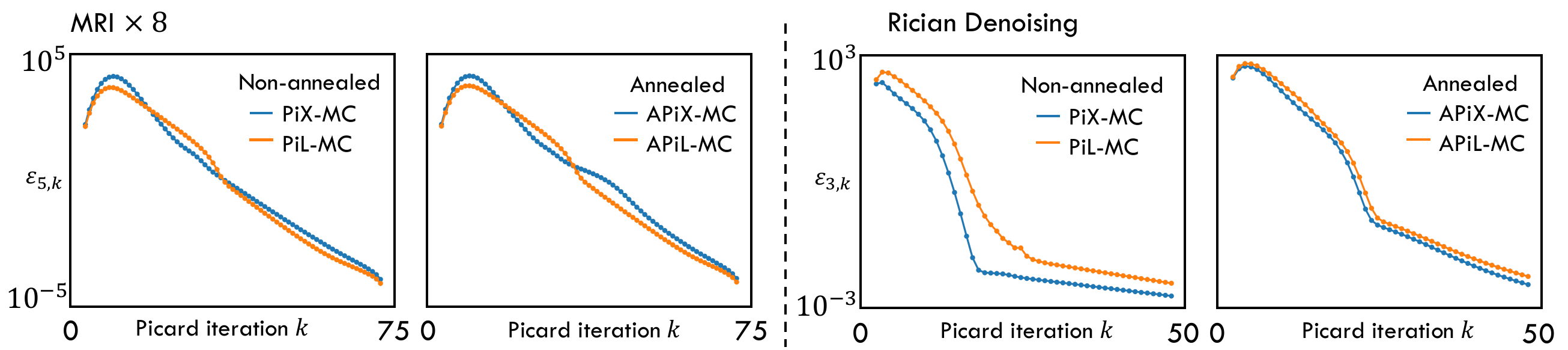}
\caption{
Visualization of the in-block Picard error $\varepsilon_{m,k}$ for {\proposed} methods on accelerated MRI reconstruction and Rician denoising. Langevin-based methods are included for reference. The error is plotted against the Picard iteration $k$ for representative blocks, using the block indexed by $m=5$ for MRI reconstruction and the block indexed by $m=3$ for Rician denoising. After a short initial transient, all {\proposed} methods exhibit geometric decay, consistent with Lemma~\ref{lemma: error decay no LSI maintext}.
}
\label{fig:decay curves}
\end{figure}

\paragraph{Implementation Details.}
For MRI, we use the pretrained score network provided by \citep{sun2024provable}\footnote{The pretrained score network is available at \url{https://github.com/sunyumark/PnP-MonteCarlo}}. 
We use ten brain images with size $320\times 320$ from the FastMRI dataset \citep{zbontar2018fastmri} as the test set. 
We use a radial subsampling mask corresponding to $8\times$ acceleration. 
We additionally add additive white Gaussian noise (AWGN) with standard deviation $\sigma_{\ybm}=0.05$. 
For Rician denoising, we use the pretrained $256\times 256$ unconditioned diffusion model\footnote{The pretrained diffusion model is available at \url{https://github.com/openai/guided-diffusion}}.
The tested dataset (CBSD10) is composed of a subset of 10 images of the CBSD68 dataset \citep{martin2001database}. 
We add Rician noise with $\sigma_{\ybm}=0.1$. 
We set $M=10, N=75, K=100$ for MRI, and
$M=5, N=100, K=50$ for Rician denoising. 
We refer to Appendix \ref{app:E2 likelihood} for more technical details on these two tasks.

\paragraph{Geometric Decay of the Picard Error.}
Lemma~\ref{lemma: error decay no LSI maintext} establishes geometric contraction of the in-block Picard error $\varepsilon_{m,k}$ defined in \eqref{eq:picarderror1}. Figure~\ref{fig:decay curves} reports the empirical evolution of $\varepsilon_{m,k}$ on a logarithmic scale for the MRI reconstruction and Rician denoising tasks. We report the block indexed by $m=5$ for MRI reconstruction and the block indexed by $m=3$ for Rician denoising, with $\varepsilon_{m,k}$ averaged over the test dataset.
After a short initial transient, the Picard error decreases approximately linearly on the logarithmic scale for all methods, including both proximal and non-proximal variants as well as their annealed counterparts. This behavior indicates rapid convergence of the Picard iteration in practice and is empirically consistent with the geometric contraction established in Lemma~\ref{lemma: error decay no LSI maintext}.

\begin{figure}[t!]
\centering
\includegraphics[width=\textwidth]{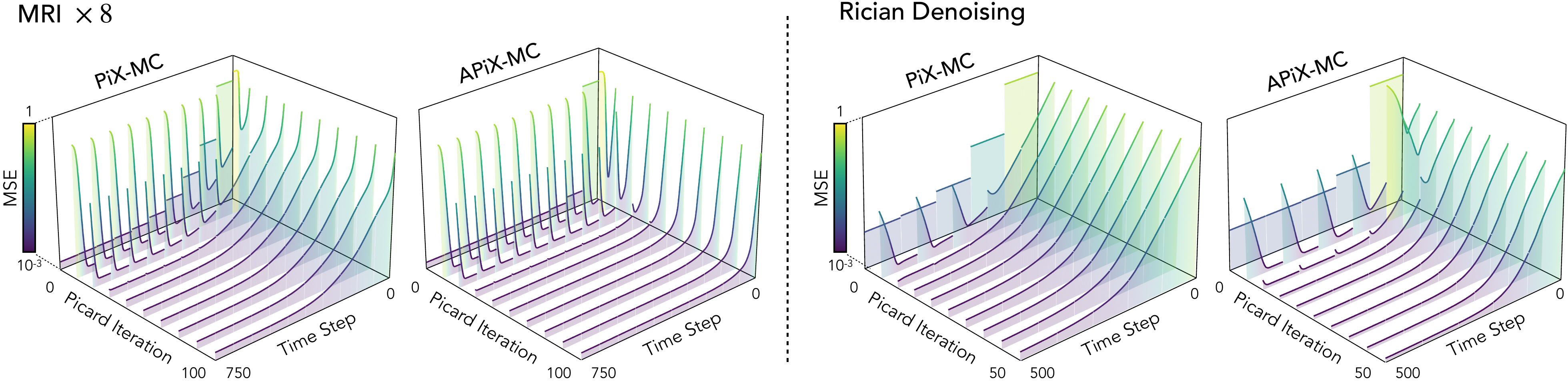}
\caption{
MSE evolution of {\proposed} and {\aproposed} over time steps and Picard iterations for accelerated MRI reconstruction (\textit{left}) and Rician denoising (\textit{right}).
At $k=0$, all time nodes within each block share the same initialization, resulting in a piecewise-constant MSE profile. As the Picard iteration proceeds, the block-wise trajectories are progressively refined into smooth curves.
}
\label{fig:mse_surface}
\end{figure}

\begin{table}[t!]
\centering
\scriptsize
\setlength{\tabcolsep}{3pt}
\renewcommand{\arraystretch}{1.15}
\begin{minipage}{0.49\linewidth}
\centering
\textbf{MRI (Accel.$=8\times$ with $\sigma_y=0.05$)}
\begin{tabular}{l|cc}
\hline\hline
\textbf{Method} & \textbf{PSNR $\uparrow$} & \textbf{SSIM $\uparrow$} \\
\hline
TV                         & 27.64 & 0.7678 \\
DAPS             & 29.66 & 0.7598 \\
\hline
PiL-MC               & 30.21 & 0.8201 \\
APiL-MC     & 31.30 & 0.8353 \\
PiX-MC              & 31.41 & 0.8439 \\
APiX-MC    & \textbf{31.69} & \textbf{0.8472} \\
\hline\hline
\end{tabular}
\end{minipage}
\hfill
\begin{minipage}{0.49\linewidth}
\centering
\textbf{Rician Denoising ($\sigma_y=0.1$)}
\begin{tabular}{l|ccc}
\hline\hline
\textbf{Method} & \textbf{PSNR $\uparrow$} & \textbf{LPIPS $\downarrow$} & \textbf{SSIM $\uparrow$} \\
\hline
 DPS              & 26.02 & 0.1429 & 0.7865 \\
 DAPS             & 27.98 & 0.1723 & 0.7533 \\
\hline
PiL-MC             	 & 30.51 & 0.1139 & 0.8657 \\
APiL-MC      & \textbf{30.62} & 0.1025 & \textbf{0.8710} \\
PiX-MC               & 29.75 & 0.1632 & 0.8363 \\
APiX-MC     & {30.43} & \textbf{0.0826} & {0.8677} \\
\hline\hline
\end{tabular}
\end{minipage}
\caption{Average quantitative results of PiX-MC and its variants for Rician noise removal ($\sigma_{\ybm}=0.1$) on the CBSD10 dataset and MRI reconstruction ($8\times,\sigma_{\ybm}=0.05$) on 10 FastMRI brain images.
The best values are shown in \textbf{bold}.
}
\label{tab: rician and mri}
\end{table}

\paragraph{Trajectory Evolution.}
We further visualize the evolution of the trajectories during the Picard iteration.
Given an image $\xbm$ and the ground-truth image $\xbm_{\mathsf{GT}}$, we compute the mean-squared error (MSE) $\text{MSE}:=\frac{1}{d}\left\|\xbm-\xbm_{\mathsf{GT}}\right\|_2^2$, and Figure~\ref{fig:mse_surface} visualizes the resulting MSE curves.
At the initial Picard iteration ($k=0$), all time nodes within each block are initialized to the same point, 
yielding piecewise-constant MSE curves. 
As the Picard iteration proceeds, the trajectories are progressively refined, and the piecewise-constant curves evolve into smooth trajectories across successive blocks. 
This evolution provides an intuitive illustration of the Picard refinement mechanism and is consistent with the convergence guarantees established in Lemma~\ref{lemma: error decay no LSI maintext} and Theorems~\ref{thm: main no LSI simple}--\ref{thm: multi-block annealing simple}.

\begin{figure}[t!]
\centering
\includegraphics[width=0.9\textwidth]{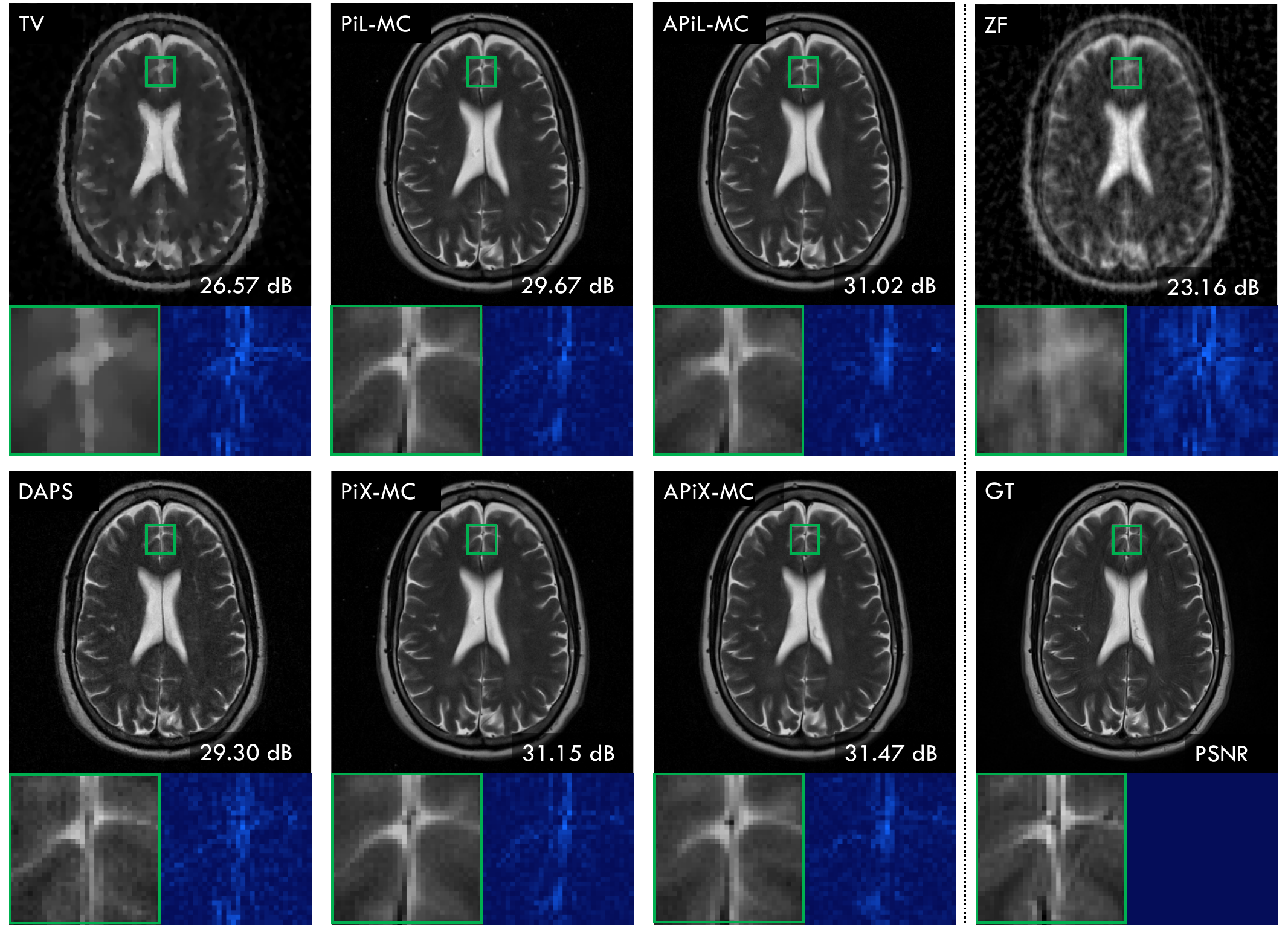}
\caption{
Visual comparison for $8\times$ accelerated MRI reconstruction with measurement noise level $\sigma_{\ybm}=0.05$. 
Green boxes highlight the zoomed-in regions, and blue panels show the corresponding residual maps. 
Note the sharper local structures and fewer residual artifacts recovered by PiX-MC and APiX-MC compared with the baseline methods.
}
\label{fig:mri_visual}
\end{figure}

\begin{figure}[t!]
\centering
\includegraphics[width=0.95\textwidth]{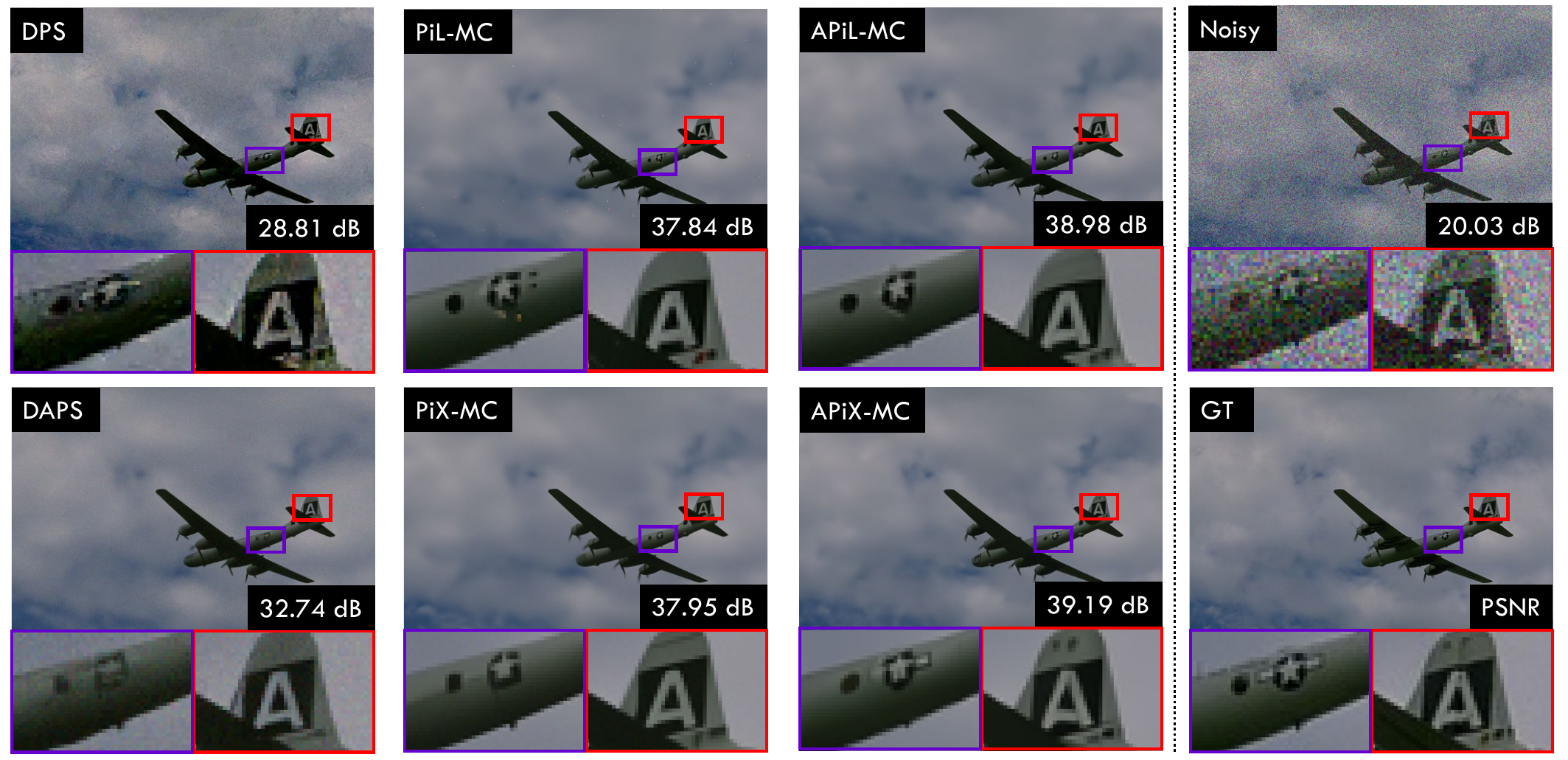}
\caption{
Visual comparison for Rician denoising with measurement noise level $\sigma_{\ybm}=0.1$. 
Purple and red boxes highlight the zoomed-in regions. 
Note the stronger Rician noise suppression and better preservation of fine markings achieved by PiX-MC and APiX-MC compared with the baseline methods.
}
\label{fig:rician_visual}
\end{figure}

\paragraph{Quantitative and Qualitative Comparison.}
Now we evaluate the final reconstruction quality of the proposed methods. Table~\ref{tab: rician and mri} reports the average quantitative results over ten test images for each task. For Rician denoising, APiL-MC achieves the highest PSNR and SSIM, while {\aproposed} attains the lowest LPIPS. For accelerated MRI, {\aproposed} achieves the best PSNR and SSIM among all compared methods. Overall, the proposed methods are competitive with diffusion-based baselines, and annealing consistently improves reconstruction quality over the corresponding non-annealed variants.

Figures~\ref{fig:mri_visual}--\ref{fig:rician_visual} provide visual comparisons. For accelerated MRI, TV exhibits staircase artifacts, while DAPS leaves noticeable residual noise. Both PiL-MC and APiL-MC remove most blur and noise but lose fine image details. In contrast, {\proposed} and {\aproposed} recover sharper edges with fewer residual artifacts, as highlighted in the zoomed-in regions of Figure~\ref{fig:mri_visual}.
A similar trend is observed for Rician denoising. DPS exhibits noticeable color shifts, and DAPS leaves residual noise in the regions highlighted in purple. While PiL-MC and APiL-MC effectively suppress the Rician noise, {\proposed} and {\aproposed} preserve finer structures. In particular, {\aproposed} successfully reconstructs the letter ``A'' together with the surrounding dot-like markings (red box), whereas these details are blurred into a thin horizontal line by {\proposed} and are absent from the remaining reconstructions. 
These visual observations are consistent with the quantitative results in Table~\ref{tab: rician and mri}.

\subsection{Practical Acceleration on Large-Scale Deblurring}
In this section, we investigate whether Picard time parallelization can accelerate posterior sampling without sacrificing reconstruction quality, and how this acceleration scales with the GPU budget and the number of time nodes per block.

\paragraph{Implementation Details.}
For the score network, we use the same diffusion model as in the Rician noise removal experiment. 
We use ten natural images from the DIV2K validation set from \citep{Agustsson_2017_CVPR_Workshops}\footnote{The DIV2K dataset is available at \url{https://data.vision.ee.ethz.ch/cvl/DIV2K/}}. We center-crop each image to size $1024\times1024$. 
We apply a camera-shake blur from \citep{levin2009understanding} to the test images, and add AWGN noise with standard deviation $\sigma_{\ybm}=0.05$. 
For the proximal methods, the proximal operator on likelihood can be solved efficiently in the Fourier domain according to \citep{pan2016l_0}.
For each block of the Picard-based methods, we adopt an adaptive stopping criterion: the Picard refinement is terminated when either the maximum number of Picard iterations $K=20$ is reached, or the normalized Picard error satisfies $\varepsilon_{m,k}/d<3\times10^{-4}$.
More implementation details are provided in Appendix \ref{app:ex}.

\paragraph{Acceleration under Different Parallelization Configurations.}
We first investigate the practical acceleration under different numbers of GPUs and time nodes $N$ within each block.
Table \ref{tab:large_scale_deblurring_acceleration} reports the runtime for reaching an average PSNR of $26$ dB and the corresponding speedup over each sequential counterpart under different configurations.
It can be seen that increasing the number of GPUs improves the speedup across all tested methods and block configurations.
In particular, using eight GPUs with $N=8$ yields speedups between $2.87\times$ and $2.99\times$ across the four Picard-based methods, as all eight time nodes can be evaluated concurrently.
Moderate oversubscription with $N=16$ remains effective, but the benefit decreases as $N$ increasingly exceeds the number of available GPUs.
For example, with four GPUs and $N=32$, all four Picard-based methods are slower than their sequential counterparts.
These results demonstrate the practical trade-off between the block length and the available computational resources.

\begin{figure}[t!]
\centering
\includegraphics[width=0.9\textwidth]{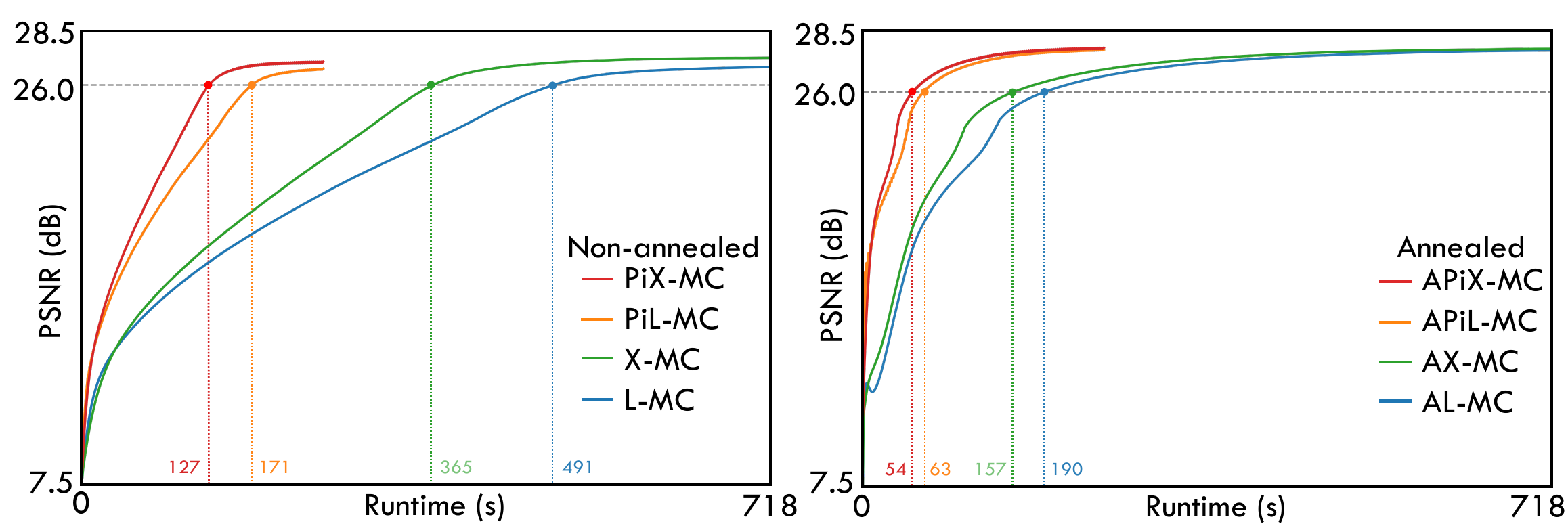}
\caption{
Average PSNR against wall-clock time for large-scale deblurring.
The left and right panels show the non-annealed and annealed methods, respectively.
The horizontal dashed line marks the reference PSNR of $26$ dB, and the vertical dotted lines indicate the corresponding threshold-crossing time in seconds.
Note how {\proposed} and {\aproposed} substantially accelerate convergence without sacrificing reconstruction quality.
}
\label{fig: deblur curves}
\end{figure}

\begin{table}[t!]
\centering
\scriptsize
\begin{tabular}{cccccccc}
\toprule
Methods & \multicolumn{3}{c}{PiL-MC vs. L-MC (491s)} & & \multicolumn{3}{c}{APiL-MC vs. AL-MC (190s)} \\
\cmidrule(lr){2-4}\cmidrule(lr){6-8}
 & \underline{$N=8$} & \underline{$N=16$} & \underline{$N=32$} &  & \underline{$N=8$} & \underline{$N=16$} & \underline{$N=32$}\\
4 GPUs & 264s / $1.86\times$ & 328s / $1.50\times$ & 550s / $0.89\times$ & & 99s / $1.92\times$ & 139s /$1.37\times$ & 250s / $0.76\times$\\
8 GPUs & 171s / $2.87\times$ & 175s / $2.81\times$ & 295s / $1.66\times$ &  & 63s / $2.99\times$ & 81s / $2.35\times$ & 151s / $1.26\times$\\
\midrule
Methods & \multicolumn{3}{c}{PiX-MC vs. X-MC (365s)} & & \multicolumn{3}{c}{APiX-MC vs. AX-MC (157s)} \\
\cmidrule(lr){2-4}\cmidrule(lr){6-8}
 & \underline{$N=8$} & \underline{$N=16$} & \underline{$N=32$} &  & \underline{$N=8$} & \underline{$N=16$} & \underline{$N=32$}\\
 4 GPUs & 195s / $1.87\times$ & 230s / $1.59\times$ & 440s / $0.83\times$ & & 84s / $1.87\times$ & 120s / $1.31\times$ & 194s / $0.81\times$  \\
8 GPUs & 127s / $2.87\times$ & 138s / $2.64\times$ & 188s / $1.94\times$  & & 54s / $2.91\times$ & 70s / $2.24\times$ & 94s / $1.67\times$\\
\bottomrule
\end{tabular}
\caption{Runtime and speedup required to reach an average PSNR of $26$ dB under different GPU budgets and numbers of time nodes $N$ per block. 
Each speedup is computed relative to the corresponding sequential method shown in the column header.
}
\label{tab:large_scale_deblurring_acceleration}
\end{table}

\begin{figure}[t!]
\centering
\includegraphics[width=\textwidth]{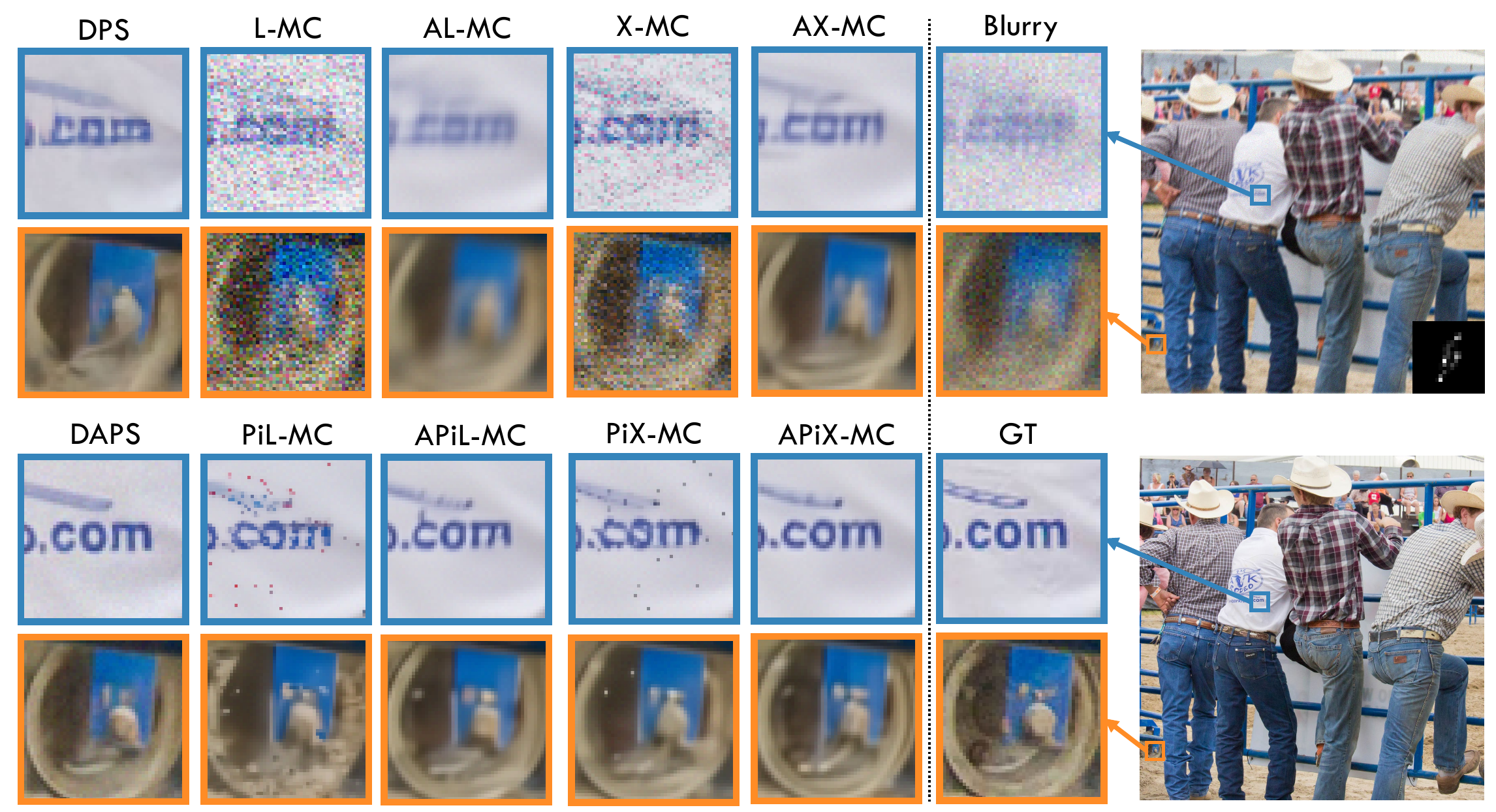}
\caption{Visual comparison of the reconstructions by {\proposed} and baselines for the large-scale deblurring task, where the underlying image is cropped into a size of $1024\times 1024$ pixels. 
DPS (334 seconds) and DAPS (741 seconds) methods are run until convergence. 
The other methods are run for 200 seconds. 
Within this limited runtime budget, the proposed methods better preserve local structures with fewer noise residuals. 
}
\label{fig:deblur_visual}
\end{figure}

Figure \ref{fig: deblur curves} plots the average PSNR against runtime in seconds under the configuration of eight GPUs and $N=8$.
As shown, Picard parallelization consistently accelerates convergence. 
In particular, {A\proposed} reaches an average PSNR of $26$ dB in only $54$ seconds, compared with $157$ seconds for AX-MC, $365$ seconds for X-MC, and $491$ seconds for L-MC.
Beyond the pairwise speedups, this comparison demonstrates that Picard time parallelization can compound with mechanisms that accelerate the underlying sequential trajectory.
For example, the proximal update and annealing together reduce the runtime from $491$ seconds for L-MC to $157$ seconds for AX-MC, corresponding to a $3.13\times$ sequential-side speedup.
Picard parallelization further reduces this runtime to $54$ seconds, providing an additional 
$2.91\times$ parallelization-side speedup.
Together, they yield an overall $9\times$ speedup over L-MC.
These results empirically support the discussion at the end of \S\ref{sec:3.4} that Picard parallelization provides a complementary direction that can be combined with sequential strategies.
This is particularly valuable for large-scale imaging tasks, where sequential trajectory evaluation is computationally expensive.

\paragraph{Quantitative and Qualitative Comparison.}
We next show that PiX-MC methods accelerate posterior sampling without compromising reconstruction quality. 
Table~\ref{tab:large_scale_deblurring_quality} reports the average final PSNR, SSIM, and LPIPS over the test set. The L-MC family is also included as references.
PiX-MC and APiX-MC closely match the reconstruction quality of their corresponding sequential counterparts, with PSNR differences of at most $0.01$ dB, while remaining competitive with the external baselines DPS and DAPS.
\begin{wraptable}[14]{r}{0.5\textwidth}
\centering
\scriptsize
\setlength{\tabcolsep}{4.5pt}
\renewcommand{\arraystretch}{1.10}
\begin{tabular}{lccc}
\toprule
\textbf{Method}
& \textbf{PSNR (dB) $\uparrow$}
& \textbf{LPIPS $\downarrow$}
& \textbf{SSIM $\uparrow$} \\
\midrule
DPS     & 26.43 & 0.2498 & 0.7142 \\
DAPS    & \textbf{28.17} & \textbf{0.1632} & 0.7444 \\
\midrule
L-MC     & 26.82 & 0.2605 & 0.7273 \\
AL-MC    & 27.92 & 0.2042 & 0.7737 \\
X-MC    & 27.24 & 0.2686 & 0.7342 \\
AX-MC   & 27.98 & 0.2138 & 0.7719 \\
\midrule
PiL-MC   & 26.82 & 0.2604 & 0.7268 \\
APiL-MC  & 27.92 & 0.2048 & \textbf{0.7738} \\
PiX-MC  & 27.23 & 0.2726 & 0.7343 \\
APiX-MC & 27.98 & 0.2120 & 0.7718 \\
\bottomrule
\end{tabular}
\caption{
Average quantitative results for large-scale deblurring on ten
$1024\times1024$ test images from DIV2K validatation set.
}
\label{tab:large_scale_deblurring_quality}
\vspace{-2pt}
\end{wraptable}

Figure~\ref{fig:deblur_visual} compares visual results under a fixed runtime budget of $200$ seconds. While DPS and DAPS require $334$ and $741$ seconds, respectively, to converge, PiX-MC and APiX-MC already recover sharp edges and fine structures with substantially reduced blur and noise within the 200-second runtime budget. Compared with their sequential counterparts, Picard time parallelization consistently produces higher-quality reconstructions under the same wall-clock budget by advancing further along the sampling trajectory. These observations, together with the quantitative results, show that Picard time parallelization achieves substantial acceleration without sacrificing the final reconstruction quality.

\begin{figure}[t!]
\centering
\includegraphics[width=\textwidth]{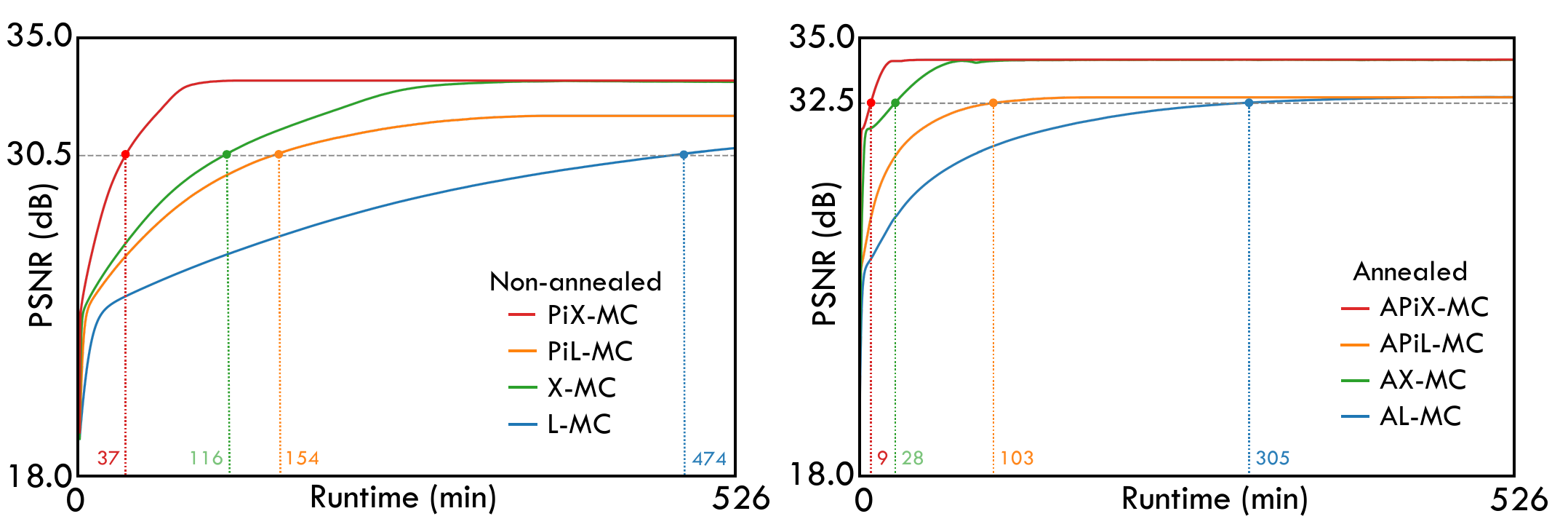}
\caption{
Average PSNR against wall-clock time for 3D CT reconstruction. The left and right panels show the non-annealed and annealed methods, respectively. 
The horizontal dashed lines mark the reference PSNR levels of $30.5$ dB and $32.5$ dB, and the vertical dotted lines indicate the corresponding threshold-crossing times. 
Note how {\proposed} and {\aproposed} attain even higher reconstruction quality in substantially less wall-clock time than their sequential counterparts. In particular, {\aproposed} achieves the 32.5 dB reference PSNR in only $9$ minutes, delivering approximately a $50\times$ wall-clock speedup over L-MC.
}
\label{fig:CT_curve}
\end{figure}

\subsection{Scaling to 3D CT Reconstruction}
Finally, we evaluate PiX-MC on a large-scale sparse-view CT reconstruction problem with volume size $512\times512\times80$, demonstrating its practical scalability to high-dimensional 3D imaging.

\paragraph{Implementation Details.}
For the score network, we trained a diffusion model based on the EDM framework \citep{karras2022elucidating}, using CT images from the AMOS \citep{ji2022amos} and LDCT \citep{moen2021low} datasets. The image values in Hounsfield units were clamped to the range $[-1000, 1000]$, then linearly rescaled to $[-1, 1]$.
We consider a sparse-view parallel-beam CT reconstruction problem with 30 uniformly spaced projection views over $180^\circ$. The target volume consists of 80 axial slices, each of size $512\times512$. Measurements are generated by applying the parallel-beam projection operator and adding i.i.d. Gaussian noise with standard deviation $\sigma_{\ybm}=0.5$ in the sinogram domain. Filtered back-projection (FBP) is used for initialization. Forward projection, back projection, and FBP are implemented using LEAP \citep{kim2023differentiable}\footnote{LEAP is available at \url{https://github.com/LLNL/LEAP}}.

To encourage consistency across adjacent slices, we additionally incorporate a one-dimensional inter-slice Huber-TV (HTV) regularizer $\Hbm(\xbm)$ into the proximal drift. The detailed formulation of the HTV regularizer and the resulting proximal drift are provided in Appendix~\ref{app:E2 likelihood}. 
Accordingly, the proximal drift $\Tsfit_{\eta,\sigma}$ is given by
\begin{equation}
    \Tsfit_{\eta,\sigma}\defn\frac{1}{\eta}\Big(I-\prox_{\eta L}\circ\big(I-\eta \Ssfit_\theta(\cdot;\sigma)-\eta\beta_H\nabla \Hbm\big)\Big),
\end{equation}
where $\beta_H>0$ is a balancing parameter for $\Hbm$. 
All Picard-based methods use eight GPUs with $N=8$ time nodes per block. 

\begin{figure}[t!]
\centering
\includegraphics[width=\textwidth]{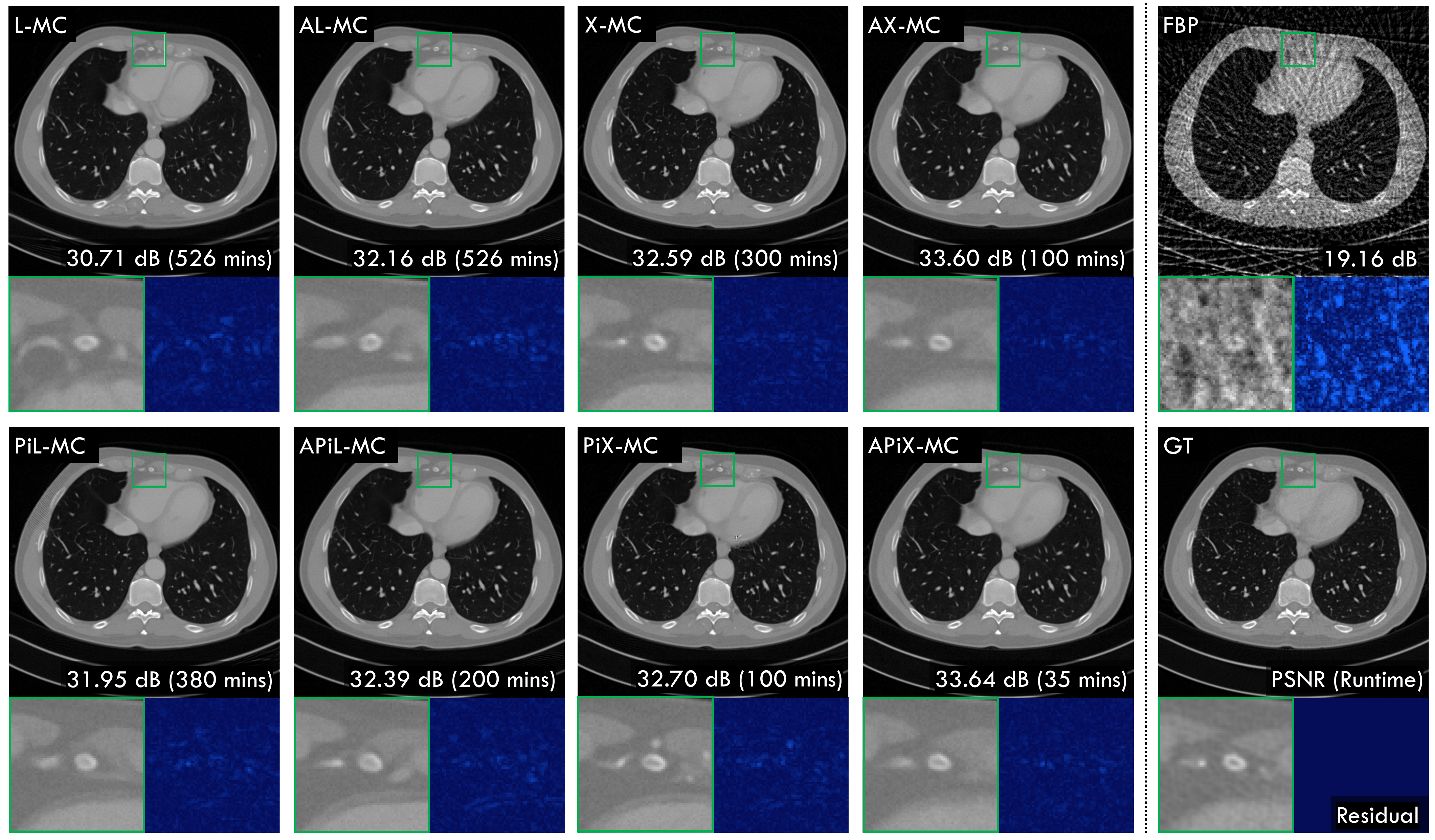}
\caption{Visual comparison of {\proposed} and baseline methods for large-scale sparse-view 3D CT reconstruction. 
The 20th slice of a $512\times512\times80$ CT volume is shown. Green boxes indicate the zoomed-in regions, together with the corresponding residual maps. 
The PSNR values and runtime (minutes) after convergence are reported. 
Note how {\aproposed} achieves the best visual quality and the highest PSNR of $33.64$ dB in only $35$ minutes, by combining Picard iteration, annealing,
and the proximal-likelihood drift.
}
\label{fig:CT_visual}
\end{figure}

\paragraph{Runtime Acceleration.}
Figure~\ref{fig:CT_curve} plots the mean slice-wise PSNR over the 80 axial slices against wall-clock time. The curves distinguish the effects of the three algorithmic components: the proximal-likelihood drift, Picard time parallelization, and annealing.
Compared with their likelihood-gradient counterparts, X-MC and AX-MC converge substantially faster and attain higher final PSNR, demonstrating that the proximal-likelihood drift provides a more effective reconstruction trajectory for this CT task. 
Building upon these improved trajectories, PiX-MC and APiX-MC further reduce the wall-clock time while preserving the attainable reconstruction quality. 
Specifically, PiX-MC reaches its reference PSNR in only $37$ minutes, compared with $116$ minutes for X-MC, corresponding to a $3.14\times$ speedup. 
Likewise, APiX-MC requires only $9$ minutes, compared with $28$ minutes for AX-MC, yielding a $3.11\times$ speedup. The same pattern is observed for the likelihood-gradient methods, where PiL-MC and APiL-MC accelerate L-MC and AL-MC by $3.08\times$ and $2.96\times$, respectively.
Annealing further improves the attainable reconstruction quality while preserving the approximately $3\times$ speedup provided by Picard time parallelization. 
Overall, APiX-MC achieves 32.5 dB in only $9$ minutes, demonstrating how the proximal-likelihood drift, annealing, and Picard time parallelization compound to deliver approximately a $50\times$ wall-clock speedup over the standard L-MC.

\paragraph{Axial-Slice and Volumetric Reconstruction.}
Figures~\ref{fig:CT_visual} and~\ref{fig:CT_3D} compare the reconstructed axial slice after convergence and the reconstructed volume under a fixed wall-clock budget of $35$ minutes, respectively. Across all four sequential/Picard pairs, Picard time parallelization achieves comparable or better reconstruction quality in substantially less wall-clock time. The visual comparisons further separate the effects of the remaining algorithmic components: compared with the likelihood-gradient methods, the proximal-likelihood variants more effectively suppress sparse-view streak artifacts and recover finer anatomical structures, while annealing further reduces residual artifacts and improves visual fidelity. These improvements are consistently observed in both the selected axial slice and the full 3D volume. Overall, APiX-MC achieves the best visual quality, and the two figures together demonstrate that the proximal-likelihood drift, annealing, and Picard time parallelization complement each other to enable high-quality large-scale 3D CT reconstruction.

\begin{figure}[t!]
\centering
\includegraphics[width=\textwidth]{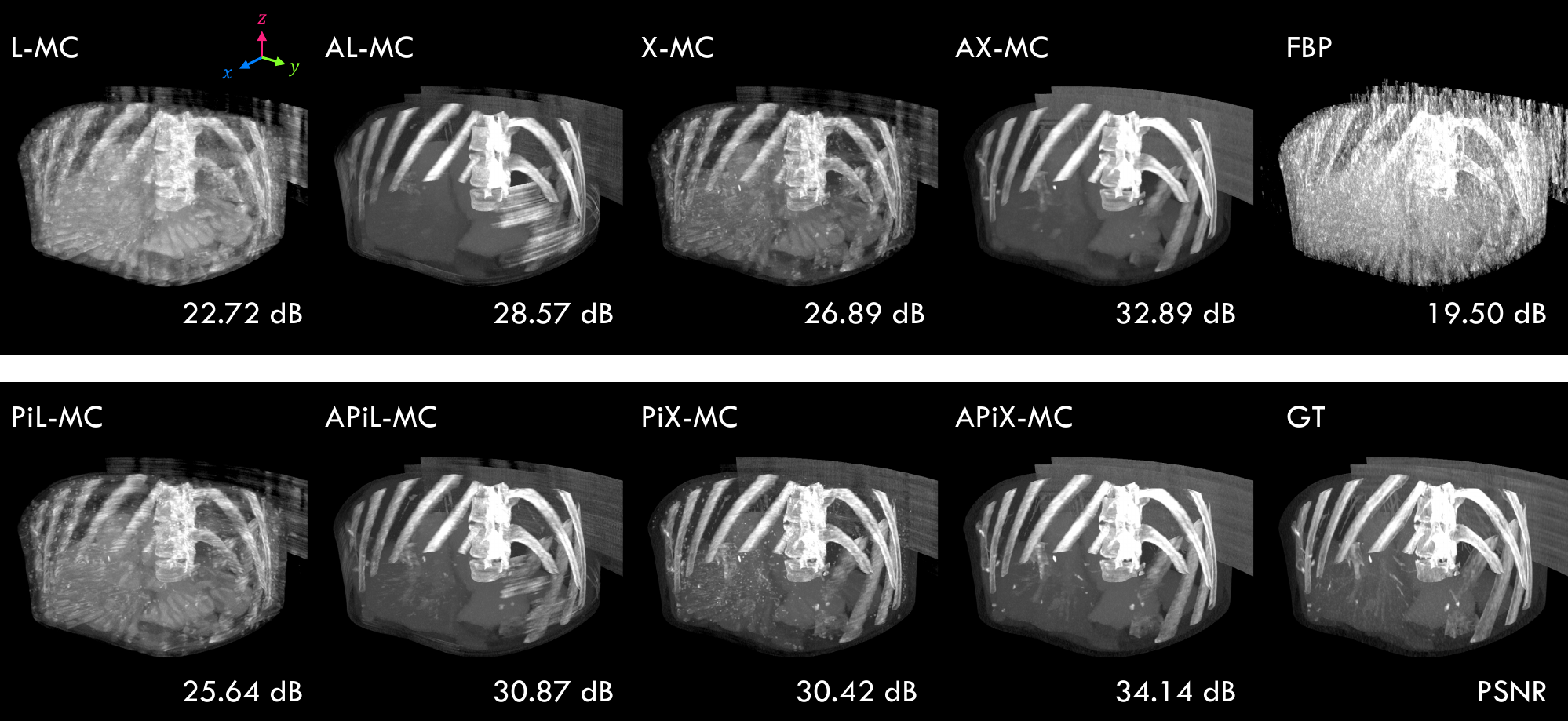}
\caption{Visual comparison of PiX-MC and baseline methods for large-scale sparse-view 3D CT reconstruction under a 35-minute time budget. 
Note how APiX-MC substantially suppresses the streak artifacts and noise while preserving clear structures.
}
\label{fig:CT_3D}
\end{figure}

\section{Conclusion}\label{sec:conclusion}
This paper introduces {\proposed}, a principled time-parallel posterior sampling framework for Bayesian imaging.
Motivated by forward-backward splitting, we construct a proximal drift that exploits efficient proximal operators available for imaging likelihoods.
Instead of advancing the samples sequentially along a long trajectory, PiX-MC parallelizes the trajectory across time nodes, enabling multiple time nodes to be processed concurrently across GPUs. 
To improve practical scalability under limited GPU budgets, we further develop annealed and multi-block variants. 
The annealed variant empirically accelerates the convergence to reach a target reconstruction quality,
while the multi-block variant partitions the trajectory into blocks, allowing PiX-MC to operate efficiently under limited GPU budgets.
We establish finite-time stationarity guarantees for all variants of PiX-MC under mild assumptions.
Extensive experiments across linear, nonlinear, and large-scale 2D and 3D imaging problems demonstrate practical scalability of PiX-MC and substantial wall-clock acceleration while maintaining reconstruction quality.

\section{Acknowledgements}
Y. Sun acknowledges support from the U.S. National Science Foundation (NSF) under Grant CCF-2542022. E. Bell is supported by the U.S. Department of Energy Computational Science Graduate Fellowship under Award No. DE-SC0026073. Y. Chen acknowledges support from the AMS-Simons Travel Grant and National Science Foundation award DMS-2608264.

\appendix

\section{Tool Lemmas}\label{app:tool}

\begin{lemma}[Lipschitz continuity of the proximal drifts]\label{lemma: lipschitz drifts}
Suppose Assumption \ref{ass: L and V 1} holds and $1-\eta\alpha_L>0$.
\begin{enumerate}[label=(\roman*)]
    \item If Assumption \ref{ass: accurate score 2} (ii) holds, the proximal drift $\Tsfit_{\eta,\sigma}$ is $\Lambda_\eta$-Lipschitz, where
    \begin{equation}
    \Lambda_\eta\defn L_\Ssfit+\beta_L \frac{1+\eta L_\Ssfit}{1-\eta \alpha_L} .
    \end{equation}
    \item If Assumption \ref{ass:4ann} (ii) holds and $\{\alpha_n\}_{n=0}^{MN-1}$ is non-increasing, then every weighted proximal drift $\Tsfit_{\eta,\sigma_n}^{\alpha_n}$ is uniformly $\Lambda_{\eta,\mathsf{ann}}$-Lipschitz, where
    \begin{equation}
    \Lambda_{\eta, \mathsf{ann}}\defn\alpha_0 L_\Ssfit+\beta_L \frac{1+\eta \alpha_0 L_\Ssfit}{1-\eta \alpha_L} .
    \end{equation}
\end{enumerate}
\end{lemma}
\begin{proof}
We prove (i) first.
    For $\xbm_1,\xbm_2\in\R^d$, define
    \begin{equation}
        \zbm_1=\prox_{\eta L}(\xbm_1 - \eta \Ssfit_\theta(\xbm_1;\sigma)),\quad \zbm_2=\prox_{\eta L}(\xbm_2 - \eta \Ssfit_\theta(\xbm_2;\sigma)).
    \end{equation}
Since $L$ is $\alpha_L$-weakly convex and $1-\eta\alpha_L>0$, $\prox_{\eta L}$ is $(1-\eta\alpha_L)^{-1}$-Lipschitz. 
Thus,
\begin{equation}
    \|\zbm_1-\zbm_2\|\le \displaystyle\frac{1}{1-\eta\alpha_L}\left\|\xbm_1-\xbm_2 -\eta\left(\Ssfit_\theta(\xbm_1;\sigma)-\Ssfit_\theta(\xbm_2;\sigma)\right)\right\|\le \frac{1+\eta L_\Ssfit}{1-\eta\alpha_L}\|\xbm_1-\xbm_2\|.
\end{equation}
The proximal optimal condition yields
\begin{equation}
    \Tsfit_\eta(\xbm_1)=\nabla L(\zbm_1)+\Ssfit_\theta(\xbm_1;\sigma),\quad \Tsfit_\eta(\xbm_2)=\nabla L(\zbm_2)+\Ssfit_\theta(\xbm_2;\sigma).
\end{equation}
Hence,
\begin{equation}
\begin{array}{rl}
\left\|\Tsfit_{\eta, \sigma}(\xbm_1)-\Tsfit_{\eta, \sigma}(\xbm_2)\right\| & \leq\left\|\nabla L\left(\zbm_1\right)-\nabla L\left(\zbm_2\right)\right\|+\left\|\Ssfit_\theta(\xbm_1; \sigma)-\Ssfit_\theta(\xbm_2; \sigma)\right\| \\
& \leq \beta_L\left\|\zbm_1-\zbm_2\right\|+L_\Ssfit\|\xbm_1-\xbm_2\| \\
& \displaystyle\leq\left(L_\Ssfit+\beta_L \frac{1+\eta L_\Ssfit}{1-\eta \alpha_L}\right)\|\xbm_1-\xbm_2\|.
\end{array}
\end{equation}

For (ii), replace $\Ssfit_\theta(\cdot;\sigma)$ by $\alpha_n \Ssfit_\theta(\cdot;\sigma_n)$.
Since $\alpha_n\leq\alpha_0$, the weighted score is uniformly $\alpha_0L_\Ssfit$-Lipschitz.
Repeating the argument yields (ii). 
\end{proof}

The following lemma from \cite[Lemma 16]{chewi2025analysis} is essential for the proof. 
\begin{lemma}\label{lemma: Lipschitz and FI}
Under Assumption \ref{ass: L and V 1}, $\nabla L(\xbm) + \nabla V (\xbm)$ is $(\beta_L+\beta_V)$-Lipschitz.
For any probability measure $\mu\ll \pi$ and $\xbm\sim \mu$, we have
\begin{equation}
    \mathbb{E}\left[\|\nabla L(\xbm) + \nabla V(\xbm)\|^2\right] \le \mathbb{E}\left[\left\|\nabla \ln \frac{\mathrm{d} \mu }{\mathrm{d} \pi}\right\|^2\right]+2 d (\beta_L+\beta_V)= \FI(\mu\|\pi)+2d(\beta_L+\beta_V).
\end{equation}
\end{lemma}

\noindent
We introduce the following technical lemma.
\begin{lemma}\label{lemma: technical lemma FBS}
    Assume that Assumptions \ref{ass: L and V 1} and \ref{ass: accurate score 2} (i) are satisfied, and $\eta>0$ is chosen sufficiently small such that $1-\eta \alpha_L, 1-3\eta^2\beta_L^2>0$.
    Then, for any $\xbm\in\mathbb{R}^d$, 
    \begin{equation}
        \|\Tsfit_\eta(\xbm)\|^2 \le \displaystyle\frac{3\delta^2}{1-3\eta^2\beta^2_L}+\frac{3}{1-3\eta^2\beta^2_L}\|\nabla L(\xbm)+\nabla V(\xbm)\|^2.
    \end{equation}
\end{lemma}
\begin{proof}
    Let 
    \begin{equation}
        \zbm=(I+\eta\nabla L)^{-1} (\xbm - \eta\Ssfit_\theta(\xbm;\sigma)).
    \end{equation}
    Then $\Tsfit_\eta(\xbm) = \frac{1}{\eta}(\xbm-\zbm)$. We have
    \begin{equation}
        \zbm+\eta\nabla L(\zbm)= \xbm-\eta \Ssfit_\theta(\xbm;\sigma).
    \end{equation}
    That is, 
    \begin{equation}
        \zbm-\xbm = -\eta\nabla L(\zbm)-\eta \Ssfit_\theta(\xbm).
    \end{equation}
    Thus, 
    \begin{equation}
        \begin{array}{rl}
            &\|\zbm-\xbm\|^2=\eta^2 \|\nabla L(\zbm) + \Ssfit_\theta(\xbm)\|^2  \\
             =&\eta^2 \|\nabla L(\zbm)-\nabla L(\xbm)+\Ssfit_\theta(\xbm;\sigma) - \nabla V(\xbm) + \nabla L(\xbm) + \nabla V(\xbm)\|^2\\
             \le & 3\eta^2 \|\nabla L(\zbm)-\nabla L(\xbm)\|^2 + 3\eta^2 \|\Ssfit_\theta(\xbm;\sigma)-\nabla V(\xbm)\|^2 + 3\eta^2 \|\nabla L(\xbm)+\nabla V(\xbm)\|^2\\
             \le & 3\eta^2\beta_L^2\|\zbm-\xbm\|^2 + 3\eta^2\delta^2 + 3\eta^2 \|\nabla L(\xbm) + \nabla V(\xbm)\|^2.
        \end{array}
    \end{equation}
    This gives us
    \begin{equation}
        \begin{array}{rl}
             &(1-3\eta^2\beta_L^2)\|\zbm-\xbm\|^2 \le 3\eta^2\delta^2+3\eta^2\|\nabla L(\xbm)+\nabla V(\xbm)\|^2  \\
             \iff & \|\Tsfit_\eta(\xbm)\|^2 \le \displaystyle\frac{3\delta^2}{1-3\eta^2\beta_L^2}+\frac{3}{1-3\eta^2\beta_L^2}\|\nabla L(\xbm)+\nabla V(\xbm)\|^2. 
        \end{array}
    \end{equation}
\end{proof}

\noindent
We now prove Lemma \ref{lemma: technical lemma drift}.
\begin{proof}
    Let $\zbm = \prox_{\eta L}(\xbm-\eta \Ssfit_\theta (\xbm;\sigma))$. 
    Thus $\zbm+\eta \nabla L(\zbm) = \xbm - \eta \Ssfit_\theta (\xbm;\sigma)$.
    Therefore $\Tsfit_\eta (\xbm) = \nabla L(\zbm)+\Ssfit_\theta (\xbm;\sigma)$. 
    It follows that 
    \begin{equation}
        \Tsfit_\eta(\xbm) - \bigl(\nabla L(\xbm)+\nabla V(\xbm)\bigr)=\nabla L(\zbm) -\nabla L(\xbm) + \Ssfit_\theta (\xbm;\sigma) -\nabla V(\xbm).
    \end{equation}
    Then we have
    \begin{equation}
    \begin{array}{rl}
         &\|\Tsfit_\eta(\xbm) - \bigl(\nabla L(\xbm)+\nabla V(\xbm)\bigr)\|^2\\
          \le&  2\|\nabla L(\zbm)-\nabla L(\xbm)\|^2+2\|\Ssfit_\theta (\xbm;\sigma)-\nabla V(\xbm)\|^2\\
         \le& 2\beta_L^2 \|\zbm-\xbm\|^2 + 2\delta^2.
    \end{array}
    \end{equation}
    Since $\zbm-\xbm = -\eta \Tsfit_\eta (\xbm)$, this gives
    \begin{equation}
        \|\Tsfit_\eta(\xbm)-\bigl(\nabla L(\xbm)+\nabla V(\xbm)\bigr)\|^2\le 2\eta^2 \beta_L^2 \|\Tsfit_\eta(\xbm)\|^2 + 2\delta^2.
    \end{equation}
    Applying Lemma \ref{lemma: technical lemma FBS} yields 
    \begin{equation}
        \|\Tsfit_\eta(\xbm) - \bigl(\nabla L(\xbm)+\nabla V(\xbm)\bigr)\|^2 \le 2\delta^2 + \displaystyle \frac{6\eta^2 \beta_L^2 \delta^2}{1-3\eta^2\beta_L^2}+\frac{6\eta^2 \beta_L^2 }{1-3\eta^2\beta_L^2}\|\nabla L(\xbm)+\nabla V(\xbm)\|^2.
    \end{equation}
    Finally, note that $\displaystyle 2\delta^2 + \frac{6\eta^2 \beta_L^2 \delta^2}{1-3\eta^2\beta_L^2} = \frac{2\delta^2}{1-3\eta^2\beta_L^2}$. 
    Hence, 
    \begin{equation}
        \|\Tsfit_\eta(\xbm) - \bigl(\nabla L(\xbm)+\nabla V(\xbm)\bigr)\|^2 \le \displaystyle \frac{2\delta^2}{1-3\eta^2\beta_L^2}+\frac{6\eta^2 \beta_L^2 }{1-3\eta^2\beta_L^2}\|\nabla L(\xbm)+\nabla V(\xbm)\|^2.
    \end{equation}
\end{proof}

\section{Proof of single-block PiX-MC}\label{app:thm1}
We present the convergence result of single-block \proposed.
Before that, we define some essential constants used throughout the proof. 
Define the constants $c_{\mathsf{score}},c_{\mathsf{split}}$:
\begin{equation}
    c_{\mathsf{score}} = \displaystyle\frac{2\delta^2}{1-3\eta^2\beta_L^2}, \quad c_{\mathsf{split}}=\frac{6\eta^2\beta_L^2}{1-3\eta^2\beta_L^2}.
\end{equation}
Define the constant $\rho$:
\begin{equation}
    \rho=\displaystyle \left(1-\frac{18\gamma^2(\beta_L+\beta_V)^2}{1-3\eta^2\beta_L^2}\right)^{-1}.
\end{equation}
Define the constant $A_1$:
\begin{equation}
    A_1=2(\beta_L+\beta_V)^2(1+2c_{\mathsf{split}})=\displaystyle\frac{2(\beta_L+\beta_V)^2(1+9\eta^2\beta^2_L)}{1-3\eta^2\beta_L^2}.
\end{equation}
Define the constants $A_4, A_5$:
\begin{equation}
    A_4=4c_{\mathsf{split}}+\frac{18\rho A_1 \gamma^2}{1-3\eta^2\beta_L^2}, \quad A_5=2c_{\mathsf{score}}+A_1\rho\left(\frac{9\gamma^2\delta^2}{1-3\eta^2\beta_L^2}+6\gamma d\right).
\end{equation}
Define the constants $B_2, B_3$:
\begin{equation}
    B_2=\displaystyle q_{\mathsf{sing}}^{2K-2}
    \cdot \left\{2Td+\frac{1}{1-3\eta^2\beta_L^2}\left[3T^2\delta^2+6T^2d(\beta_L+\beta_V)\right]\right\},
\end{equation}
and
\begin{equation}
    B_3=q_{\mathsf{sing}}^{2K-2} \cdot\frac{3T^2}{1-3\eta^2\beta_L^2}.
\end{equation}
Finally, we define the constant $A$:
\begin{equation}
    A:=3\rho A_1 B_2+A_5+2A_4d(\beta_L+\beta_V).
\end{equation}
Recall that $q_{\mathsf{sing}}=T\Lambda_\eta$ is the contraction factor introduced in \S \ref{sec:3.2}.
We now state a detailed version of Theorem \ref{thm:main-no-block-simple} with explicit constants.

\begin{theorem}[Detailed version of Theorem \ref{thm:main-no-block-simple}]\label{thm: main no block}
Let $(\mu_t)_{t\geq0}$ denote the law for the continuous interpolation of $\{\Xn^{(K)}\}_{n=0}^{N}$ generated by
{\proposed}, where $N > 0$ is the total number of time steps. 
Assume Assumptions~\ref{ass: L and V 1}-\ref{ass: accurate score 2} hold, and the parameters $\eta,\gamma>0$ are chosen sufficiently small such that 
\begin{equation}
    \displaystyle
    1-\eta\alpha_L>0, 1-3\eta^2\beta_L^2>0, 1-\frac{18 \gamma^2\left(\beta_L+\beta_V\right)^2}{1-3 \eta^2 \beta_L^2}>0,  q_{\mathsf{sing}}<1, A_4<0.75.
\end{equation}
Then,
\begin{equation}
\displaystyle (0.75-A_4)\frac{1}{T}\int_0^T \FI(\mu_t\|\pi)dt 
    \le \frac{1}{T} \KL (\mu_0\|\pi) + A + 3\rho A_1 B_3 \cdot\FI(\mu_0\|\pi).
\end{equation}
\end{theorem}
Note that 
\begin{equation}
    A_4=\displaystyle\frac{24 \eta^2 \beta_L^2}{1-3 \eta^2 \beta_L^2}+\frac{36 \gamma^2\left(\beta_L+\beta_V\right)^2\left(1+9 \eta^2 \beta_L^2\right)}{\left(1-3 \eta^2 \beta_L^2\right)^2\left(1-\frac{18 \gamma^2\left(\beta_L+\beta_V\right)^2}{1-3 \eta^2 \beta_L^2}\right)}.
\end{equation}
Hence, $A_4\to0$ as $(\eta,\gamma)\to(0,0)$, and therefore $A_4<0.75$ for sufficiently small $\eta$ and $\gamma$. 
More explicitly, direct simplification gives
\begin{equation}
    \displaystyle
    0.75-A_4=0.75\cdot \frac{1-35 \eta^2 \beta_L^2-66 \gamma^2\left(\beta_L+\beta_V\right)^2}{1-3 \eta^2 \beta_L^2-18 \gamma^2\left(\beta_L+\beta_V\right)^2} .
\end{equation}
Consequently, the sufficient conditions
\begin{equation}\label{eq:example for condition}
    \displaystyle
    \eta^2\beta_L^2<\frac{1}{35}, \quad \gamma^2(\beta_L+\beta_V)^2<\frac{1-35\eta^2\beta_L^2}{66}
\end{equation}
ensure that $A_4<0.75$.

Before the proof of Theorem \ref{thm: main no block}, we first introduce a technical lemma, which shows that the error decays exponentially fast in the parallel refinement, if $q_{\mathsf{sing}}<1$.

\begin{lemma}[Detailed version of Lemma \ref{lemma: error decay no block maintext}]\label{lemma: error decay no block}
    Let $\mu_0$ be the initial distribution of $\xbm_0$ and assume Assumptions \ref{ass: L and V 1}-\ref{ass: accurate score 2} hold. Choose the parameters $\eta,\gamma >0$ sufficiently small such that 
    \begin{equation}
        1-\eta \alpha_L>0, \quad 1-3\eta^2\beta_L^2>0, \quad q_{\mathsf{sing}}<1.
    \end{equation}
    Define the error across two successive Picard refinements at the $k$-th iteration as
    \begin{equation}\varepsilon_k\defn\max\limits_{n=1,2,\dots,N}\E\left[\|\xbm^{(k)}_{n}-\xbm^{(k-1)}_{n}\|^2\right].
    \end{equation}
    Then, the Picard iteration error for {\proposed} satisfies
    \begin{equation}
    \varepsilon_K \le  B_2+B_3\cdot \FI (\mu_{0}\|\pi).
    \end{equation}
\end{lemma}

\begin{proof} For each $ n=1,2,\dots,N$,
    \begin{equation}
        \begin{array}{rl}
             \mathbb{E}\|\xbm^{(k+1)}_{n}-\xbm^{(k)}_{n}\|^2 
             =& \mathbb{E}\left\|\gamma\sum\limits_{i=0}^{n-1}\left[\Tsfit_\eta(\xbm^{(k)}_{i})-\Tsfit_\eta(\xbm^{(k-1)}_{i})\right]\right\|^2  \\
             \le & n\gamma^2 \sum\limits_{i=0}^{n-1}\mathbb{E}\|\Tsfit_\eta(\xbm^{(k)}_{i})-\Tsfit_\eta(\xbm^{(k-1)}_{i})\|^2\\
             \le & n\gamma^2 \Lambda_\eta^2 \sum\limits_{i=0}^{n-1}\mathbb{E}\|\xbm^{(k)}_{i}-\xbm^{(k-1)}_{i}\|^2\\
             \le & n^2\gamma^2\Lambda_\eta^2\varepsilon_k\le \left(n\gamma\Lambda_\eta\right)^2 \varepsilon_k\le q_{\mathsf{sing}}^2 \varepsilon_k.
        \end{array}
    \end{equation}
    Thus, $\varepsilon_{k+1}\le q_{\mathsf{sing}}^2 \varepsilon_k$.
    Note that according to Algorithm \ref{alg: pix}, $\xbm^{(0)}_{n}\equiv \xbm_{0}$. In particular, we have
    \begin{equation}
    \begin{array}{rl}
         & \mathbb{E}\|\xbm^{(1)}_{n}-\xbm^{(0)}_{n}\|^2=\mathbb{E}\left\|\gamma\sum\limits_{i=0}^{n-1}\Tsfit_\eta(\xbm_0)+\sqrt{2}\Bbm_{n\gamma }\right\|^2\\
         =& n^2\gamma^2\mathbb{E}\|\Tsfit_\eta(\xbm_0)\|^2+2n\gamma d\le T^2\mathbb{E}\|\Tsfit_\eta(\xbm_0)\|^2+2Td.
    \end{array}
    \end{equation}
    Note that since $\xbm_0$ is independent of $\Bbm_{n\gamma}$ and $\mathbb E[\Bbm_{n\gamma}]=\zerobm$, the cross term vanishes.
    Thus $\varepsilon_1\le T^2\mathbb{E}\|\Tsfit_\eta(\xbm_0)\|^2+2Td$. 
    Since $\varepsilon_{k+1}\le q_{\mathsf{sing}}^2\varepsilon_k$ and $\varepsilon_1\le T^2\mathbb{E}\|\Tsfit_\eta(\xbm_0)\|^2+2Td$, by induction we have
    \begin{equation}\label{eq: temp0* no block}
        \varepsilon_K\le q_{\mathsf{sing}}^{2K-2}\left( T^2 \mathbb{E}\|\Tsfit_\eta (\xbm_0)\|^2+2Td\right).
    \end{equation}
    We bound $\mathbb{E}\|\Tsfit_\eta (\xbm_0)\|^2$ using Lemma \ref{lemma: Lipschitz and FI}. By Lemma \ref{lemma: technical lemma FBS}, we have that 
    \begin{equation}
        \mathbb{E}\|\Tsfit_\eta (\xbm_0)\|^2\le \displaystyle\frac{3\delta^2}{1-3\eta^2\beta_L^2}+\frac{3}{1-3\eta^2\beta_L^2}\mathbb{E}\|\nabla L(\xbm_0)+\nabla V(\xbm_0)\|^2.
    \end{equation}
    Since $\nabla L(\xbm)+\nabla V(\xbm)$ is $(\beta_L+\beta_V)$-smooth, by Lemma \ref{lemma: Lipschitz and FI}, we have
    \begin{equation}
        \mathbb{E}\| \nabla L(\xbm_0)+ \nabla V(\xbm_0)\|^2\le \displaystyle 2d(\beta_L+\beta_V)+\FI(\mu_{0}\| \pi).
    \end{equation}
    Therefore, 
    \begin{equation}\label{eq: temp1* no block}
        \mathbb{E}\|\Tsfit_\eta(\xbm_0)\|^2\le \displaystyle \frac{3\delta^2}{1-3\eta^2\beta_L^2}+\frac{3}{1-3\eta^2\beta_L^2}\left[ 2d(\beta_L+\beta_V)+\FI (\mu_{0}\| \pi)\right].
    \end{equation}
    Substituting \eqref{eq: temp1* no block} into \eqref{eq: temp0* no block} completes the proof.
\end{proof}

\noindent
Now we are ready to prove Theorem \ref{thm: main no block}.
\begin{proof}
    We will use the interpolation method.
    Let $\Xbm_{t\in[0,T]}$ denote the continuous-time interpolation of the final Picard trajectory $\{\xbm_n^{(K)}\}_{n=0}^N$. 
    Let $\Xbm_{n\gamma} = \xbm_n^{(K)}$.
    For any $t\in[n\gamma,(n+1)\gamma]$, define
    \begin{equation}\label{eq: temp2 no block}
        \Xbm_t:=\xbm_{n}^{(K)}-(t-n\gamma)\Tsfit_\eta(\xbm^{(K-1)}_{n})+\sqrt{2}(\Bbm_t-\Bbm_{n\gamma}).
    \end{equation}
    Now, let $\mu_t$ be the law of $\Xbm_t$, $\mu_t\equiv\mathsf{law}(\Xbm_t)$.
    With respect to the natural filtration, $\Tsfit_{\eta}(\xbm_n^{(K-1)})$ is $\mathcal{F}_{n\gamma}$- measurable and constant in time on this interval. The marginal process therefore has conditional drift $-\mathbb{E}\left[\Tsfit_\eta(\xbm_n^{(K-1)}) | \Xbm_t\right]$. 
    Applying the entropy-dissipation identity and conditional Jensen's inequality, as in \cite[proof of Lemma 3]{vempala2019rapid}, gives:
    \begin{equation}
        \partial_t \KL(\mu_t\|\pi) = -\FI(\mu_t\|\pi) + \mathbb{E}\left\langle \nabla L(\Xbm_t)+\nabla V(\Xbm_t)-\Tsfit_\eta(\xbm^{(K-1)}_{n}), \nabla \log \displaystyle \frac{\mu_t(\Xbm_t)}{\pi(\Xbm_t)}\right\rangle.
    \end{equation}
    Thus we have that 
    \begin{equation}\label{eq: temp* no block}
        \partial_t \KL(\mu_t\|\pi) \le -\displaystyle\frac{3}{4}\FI(\mu_t\|\pi) + \mathbb{E}\left\| \nabla L(\Xbm_t)+\nabla V(\Xbm_t)-\Tsfit_\eta(\xbm^{(K-1)}_{n})\right\|^2.
    \end{equation}
    We need to bound the last term. 
    \begin{equation}
        \begin{array}{rl}
             &\mathbb{E}\|\nabla L(\Xbm_t)+\nabla V(\Xbm_t)-\Tsfit_\eta (\xbm^{(K-1)}_{n})\|^2  \\
             \le & 2\mathbb{E}\| \nabla L(\Xbm_t)+\nabla V(\Xbm_t)-\nabla L(\xbm^{(K-1)}_{n})-\nabla V(\xbm^{(K-1)}_{n})\|^2 \\
             &+ 2\mathbb{E}\|\nabla L(\xbm^{(K-1)}_{n})+\nabla V(\xbm^{(K-1)}_{n})-\Tsfit_\eta(\xbm^{(K-1)}_{n})\|^2.
        \end{array}
    \end{equation}
    We bound each of the two terms.
    For the first term,
    \begin{equation}
        \mathbb{E}\|\nabla L(\Xbm_t)+\nabla V(\Xbm_t)-\nabla L(\xbm^{(K-1)}_{n})-\nabla V(\xbm^{(K-1)}_{n})\|^2\le (\beta_L+\beta_V)^2\mathbb{E}\|\Xbm_t-\xbm^{(K-1)}_{n}\|^2.
    \end{equation}
    For the second term, note that
    \begin{equation}
    \nabla L(\xbm) + \nabla V(\xbm)
        -\Tsfit_\eta(\xbm)=\nabla V(\xbm)-\Ssfit_\theta(\xbm;\sigma)+\nabla L(\xbm)-\nabla L(\zbm),
    \end{equation}
    where $\zbm=(I+\eta \nabla L)^{-1}\bigl(\xbm-\eta \Ssfit_\theta(\xbm;\sigma)\bigr)$. 
    Thus, 
    \begin{equation}
    \begin{array}{rl}
         & \|\nabla L(\xbm) + \nabla V(\xbm)-\Tsfit_\eta(\xbm)\|^2 \\
         \le &2\|\nabla V(\xbm)-\Ssfit_\theta(\xbm;\sigma)\|^2 + 2\|\nabla L(\xbm)-\nabla L(\zbm)\|^2 \\
         \le & 2\delta^2 + 2\beta_L^2 \|\xbm-\zbm\|^2.
    \end{array}
    \end{equation}
    By Lemma \ref{lemma: technical lemma FBS}, 
    \begin{equation}
        \displaystyle \|\Tsfit_\eta(\xbm)\|^2 \le \frac{3\delta^2 + 3\|\nabla L(\xbm) + \nabla V(\xbm)\|^2}{1-3\eta^2\beta_L^2}.
    \end{equation}
    Thus,
    \begin{equation}
    \begin{array}{rl}
         \left\|\nabla L(\xbm) + \nabla V(\xbm)-\Tsfit_\eta(\xbm)\right\|^2
         \leq & \displaystyle\frac{2 \delta^2}{1-3 \eta^2 \beta_L^2}+\frac{6 \eta^2 \beta_L^2}{1-3 \eta^2 \beta_L^2}\|\nabla L(\xbm) + \nabla V(\xbm)\|^2 \\
         = &c_{\mathsf{score}} + c_{\mathsf{split}} \|\nabla L(\xbm) + \nabla V(\xbm)\|^2. 
    \end{array}
    \end{equation}
    We have
    \begin{equation}
    \begin{array}{rl}
         &\mathbb{E}\|\nabla L(\Xbm_t)+\nabla V(\Xbm_t)-\Tsfit_\eta (\xbm^{(K-1)}_{n})\|^2 \vspace{1ex} \\
         \le & 2(\beta_L+\beta_V)^2\mathbb{E}\|\Xbm_t-\xbm^{(K-1)}_{n}\|^2 + 2c_{\mathsf{score}} + 2c_{\mathsf{split}}\mathbb{E}\|\nabla L(\xbm^{(K-1)}_{n})+\nabla V(\xbm^{(K-1)}_{n})\|^2.
    \end{array}
    \end{equation}
    Now we bound $\mathbb{E}\|\nabla L(\xbm^{(K-1)}_{n})+\nabla V(\xbm^{(K-1)}_{n})\|^2$:
    \begin{equation}
    \begin{array}{rl}
        &\mathbb{E}\|\nabla L(\xbm^{(K-1)}_{n})+\nabla V(\xbm^{(K-1)}_{n})\|^2 \\
        
        \le & 2\mathbb{E}\|\nabla L(\xbm^{(K-1)}_{n})+\nabla V(\xbm^{(K-1)}_{n})-\nabla L(\Xbm_t)-\nabla V(\Xbm_t)\|^2+2\mathbb{E}\|\nabla L(\Xbm_t)+\nabla V(\Xbm_t)\|^2\\
        \le & 2(\beta_L+\beta_V)^2\mathbb{E}\|\xbm^{(K-1)}_{n}-\Xbm_t\|^2+2\mathbb{E}\|\nabla L(\Xbm_t)+\nabla V(\Xbm_t)\|^2.
    \end{array}
    \end{equation}
    and
    \begin{equation}
    \begin{array}{rl}
         &\mathbb{E}\|\nabla L(\Xbm_t)+\nabla V(\Xbm_t)-\Tsfit_\eta (\xbm^{(K-1)}_{n})\|^2 \vspace{1ex} \\
         \le & \left[2(\beta_L+\beta_V)^2+4c_{\mathsf{split}}(\beta_L+\beta_V)^2\right]\mathbb{E}\|\Xbm_t-\xbm^{(K-1)}_{n}\|^2 \\
         &+ 4c_{\mathsf{split}}\mathbb{E}\|\nabla L(\Xbm_t)+\nabla V(\Xbm_t)\|^2 + 2c_{\mathsf{score}}.
    \end{array}
    \end{equation}
    By the definition of $A_1$, we have
    \begin{equation}
    \begin{array}{rl}
         & \mathbb{E}\|\nabla L(\Xbm_t)+\nabla V(\Xbm_t)-\Tsfit_\eta (\xbm^{(K-1)}_{n})\|^2 \\
         \le &A_1\mathbb{E}\|\Xbm_t-\xbm^{(K-1)}_{n}\|^2 + 4c_{\mathsf{split}}\mathbb{E}\|\nabla L(\Xbm_t)+\nabla V(\Xbm_t)\|^2 + 2c_{\mathsf{score}}. 
    \end{array}
    \end{equation}
    Now we consider $\mathbb{E}\|\Xbm_t-\xbm^{(K-1)}_{n}\|^2$.
    \begin{equation}
        \Xbm_t-\xbm^{(K-1)}_{n}=\xbm^{(K)}_{n}-\xbm^{(K-1)}_{n}-(t-n\gamma)\Tsfit_\eta(\xbm^{(K-1)}_{n})+\sqrt{2}(\Bbm_t-\Bbm_{n\gamma}).
    \end{equation}
    \begin{equation}
        \mathbb{E}\|\Xbm_t-\xbm^{(K-1)}_{n}\|^2\le 3\mathbb{E}\|\xbm^{(K)}_{n}-\xbm^{(K-1)}_{n}\|^2+3\gamma^2\mathbb{E}\|\Tsfit_\eta(\xbm^{(K-1)}_{n})\|^2+6\gamma d.
    \end{equation}
    By Lemma \ref{lemma: technical lemma FBS},
    \begin{equation}
        \begin{array}{rl}
            \mathbb{E}\|\Tsfit_\eta(\xbm^{(K-1)}_{n})\|^2\le&\displaystyle \frac{3\delta^2+3\mathbb{E}\|\nabla L(\xbm^{(K-1)}_{n})+\nabla V(\xbm^{(K-1)}_{n})\|^2}{1-3\eta^2\beta_L^2}\vspace{1ex}\\
            \le &\displaystyle\frac{3\delta^2+6(\beta_L+\beta_V)^2\mathbb{E}\|\xbm^{(K-1)}_{n}-\Xbm_t\|^2+6\mathbb{E}\|\nabla L(\Xbm_t)+\nabla V(\Xbm_t)\|^2}{1-3\eta^2\beta_L^2};
        \end{array}
    \end{equation}
    \begin{equation}
    \begin{array}{rl}
        &\mathbb{E}\|\Xbm_t-\xbm^{(K-1)}_{n}\|^2\\
        \le & 3\mathbb{E}\|\xbm^{(K)}_{n}-\xbm^{(K-1)}_{n}\|^2+6\gamma d\\
        &\displaystyle+\frac{9\gamma^2\delta^2+18\gamma^2(\beta_L+\beta_V)^2\mathbb{E}\|\xbm^{(K-1)}_{n}-\Xbm_t\|^2+18\gamma^2\mathbb{E}\|\nabla L(\Xbm_t)+\nabla V(\Xbm_t)\|^2}{1-3\eta^2\beta_L^2}.
    \end{array}
    \end{equation}
    Rearranging this, we obtain
    \begin{equation}
        \begin{array}{rl}
             & \mathbb{E}\|\Xbm_t-\xbm^{(K-1)}_{n}\|^2\\
             \le &3\rho\mathbb{E}\|\xbm^{(K)}_{n}-\xbm^{(K-1)}_{n}\|^2+\rho\left(\frac{9\gamma^2\delta^2}{1-3\eta^2\beta_L^2}+6\gamma d\right)+\frac{18\rho \gamma^2}{1-3\eta^2\beta_L^2}\mathbb{E}\|\nabla L(\Xbm_t) + \nabla V(\Xbm_t)\|^2. 
        \end{array}
    \end{equation}
    Thus,
    \begin{equation}
    \begin{array}{rl}
         &\mathbb{E}\|\nabla L(\Xbm_t) + \nabla V(\Xbm_t)-\Tsfit_\eta (\xbm^{(K-1)}_{n})\|^2 \vspace{1ex} \\
         \le & \displaystyle3\rho A_1\mathbb{E}\|\xbm^{(K)}_{n}-\xbm^{(K-1)}_{n}\|^2 + A_1\rho\left(\frac{9\gamma^2\delta^2}{1-3\eta^2\beta_L^2}+6\gamma d\right) \\
         &\displaystyle+ \frac{18\rho A_1 \gamma^2}{1-3\eta^2\beta_L^2}\mathbb{E}\|\nabla L(\Xbm_t) + \nabla V(\Xbm_t)\|^2+ 4c_{\mathsf{split}}\mathbb{E}\|\nabla L(\Xbm_t) + \nabla V(\Xbm_t)\|^2 + 2c_{\mathsf{score}}.
    \end{array}
    \end{equation}
    The coefficient of $\mathbb{E}\|\nabla L(\Xbm_t) + \nabla V(\Xbm_t)\|^2$ is $A_4$, the constant term is $A_5$. That is,
    \begin{equation}
    \begin{array}{rl}
         & \mathbb{E}\|\nabla L(\Xbm_t) + \nabla V(\Xbm_t)-\Tsfit_\eta (\xbm^{(K-1)}_{n})\|^2 \\ \le &3\rho A_1\mathbb{E}\|\xbm^{(K)}_{n}-\xbm^{(K-1)}_{n}\|^2 + A_4\mathbb{E}\|\nabla L(\Xbm_t) + \nabla V(\Xbm_t)\|^2 + A_5. \\
         & 
    \end{array}
    \end{equation}
    By Lemma \ref{lemma: Lipschitz and FI},
    \begin{equation}
        \mathbb{E}\|\nabla L(\Xbm_t) + \nabla V(\Xbm_t)\|^2\le \FI(\mu_t\|\pi)+2d(\beta_L+\beta_V).
    \end{equation}
    Thus we have
    \begin{equation}
    \begin{array}{rl}
         &  \mathbb{E}\|\nabla L(\Xbm_t) + \nabla V(\Xbm_t)-\Tsfit_\eta (\xbm^{(K-1)}_{n})\|^2 \\
         \le &3\rho A_1\mathbb{E}\|\xbm^{(K)}_{n}-\xbm^{(K-1)}_{n}\|^2 + A_4\FI(\mu_t\|\pi) + A_5 + 2A_4d(\beta_L+\beta_V). 
    \end{array}
    \end{equation}
    Back to \eqref{eq: temp* no block},
    \begin{equation}
    \begin{array}{rl}
        \partial_t\KL(\mu_t\|\pi) &\le -\frac{3}{4}\FI(\mu_t\|\pi) + 3\rho A_1\mathbb{E}\|\xbm^{(K)}_{n}-\xbm^{(K-1)}_{n}\|^2 + A_4\FI(\mu_t\|\pi) + A_5 + 2A_4d(\beta_L+\beta_V) \\
        &= \left(A_4-\frac{3}{4}\right)\FI(\mu_t\|\pi) + 3\rho A_1\mathbb{E}\|\xbm^{(K)}_{n}-\xbm^{(K-1)}_{n}\|^2 + A_5 + 2A_4d(\beta_L+\beta_V).
    \end{array}
    \end{equation}
    By Lemma \ref{lemma: error decay no block},
    \begin{equation}
        \mathbb{E}\|\xbm^{(K)}_{n}-\xbm^{(K-1)}_{n}\|^2\le B_2 + B_3\cdot\FI(\mu_0\|\pi).
    \end{equation}
    Thus,
    \begin{equation}
        \partial_t\KL(\mu_t\|\pi) \le (A_4-0.75)\FI(\mu_t\|\pi) + A + 3\rho A_1 B_3\cdot\FI(\mu_0\|\pi),
    \end{equation}
    where $A = 3\rho A_1 B_2 + A_5 + 2A_4d(\beta_L+\beta_V)$.
    Now we integrate over $t\in[0, T]$:
    \begin{equation}
        \KL(\mu_T\|\pi)-\KL(\mu_0\|\pi) \le (A_4-0.75)\int_0^T\FI(\mu_t\|\pi)dt + AT + 3\rho T A_1 B_3\cdot\FI(\mu_0\|\pi).
    \end{equation}
    Note that $\KL(\mu_T\|\pi)\ge 0$, thus
    \begin{equation}
        (0.75-A_4)\frac{1}{T}\int_0^T\FI(\mu_t\|\pi)dt \le \frac{1}{T}\KL(\mu_0\|\pi) + A + 3\rho A_1 B_3\cdot\FI(\mu_0\|\pi).
    \end{equation}
    This completes the proof.
\end{proof}

\section{Sufficient conditions for Assumption \ref{ass: multi-block stability}}\label{app:diss}
We present Lemma \ref{lemma: sufficient multi-block stability} as an example for Assumption \ref{ass: multi-block stability} to hold.
\begin{lemma}[Sufficient condition for Assumption \ref{ass: multi-block stability}]
\label{lemma: sufficient multi-block stability}
Suppose Assumptions \ref{ass: L and V 1}-\ref{ass: accurate score 2} hold. Also suppose that $\mathbb{E}\|\xbm_0\|^2<\infty$.
Assume additionally that there exist constants $\kappa>\alpha_L+\alpha_V$ and $b\geq0$ such that
\begin{equation}\label{eq: quadratic coercivity}
    L(\xbm)+V(\xbm)\geq \frac{\kappa}{2}\|\xbm\|^2-b,\qquad \forall \xbm\in\mathbb{R}^d.
\end{equation}
Then, there exist $\eta_0,\gamma_0>0$ such that, for any $0<\eta\leq\eta_0$ and $0<\gamma\leq\gamma_0$ satisfying $q_{\mathsf{mult}}=\tau\Lambda_\eta<1$, there exists $K_0<\infty$ such that, for every $K\geq K_0$, the block starting points generated by multi-block {\proposed} satisfy
\begin{equation}
    \sup_{m\geq0}\mathbb{E}\|\xbm^{(0)}_{mN}\|^2\leq C_\xbm
\end{equation}
for some finite constant $C_\xbm>0$ independent of the number of blocks $M$ and the total time horizon $T$.
Hence, Assumption~\ref{ass: multi-block stability} holds.
\end{lemma}

\begin{proof}
We first show that the quadratic coercivity condition implies dissipativity of the ideal posterior drift.
Since $L+V$ is $(\alpha_L+\alpha_V)$-weakly convex, for any $\xbm\in\mathbb{R}^d$,
\begin{equation}
    L(\zerobm)+V(\zerobm) \geq L(\xbm)+V(\xbm)-\langle \nabla L(\xbm)+\nabla V(\xbm),\xbm\rangle-\frac{\alpha_L+\alpha_V}{2}\|\xbm\|^2.
\end{equation}
Combining this inequality with \eqref{eq: quadratic coercivity} gives
\begin{equation}
\begin{array}{rl}
    \langle \xbm,\nabla L(\xbm)+\nabla V(\xbm)\rangle
    &\geq L(\xbm)+V(\xbm)-L(\zerobm)-V(\zerobm)-\displaystyle\frac{\alpha_L+\alpha_V}{2}\|\xbm\|^2\\[0.5ex]
    &\geq\displaystyle\frac{\kappa-\alpha_L-\alpha_V}{2}\|\xbm\|^2-b-|L(\zerobm)+V(\zerobm)|.
\end{array}
\end{equation}
Then,
\begin{equation}\label{eq: ideal drift dissipativity}
    \langle \xbm,\nabla L(\xbm)+\nabla V(\xbm)\rangle
    \geq a\|\xbm\|^2-b_0, \ a\defn\frac{\kappa-\alpha_L-\alpha_V}{2}>0,\ b_0\defn b+|L(\zerobm)+V(\zerobm)|.
\end{equation}
We next transfer this dissipativity property to the proximal drift $\Tsfit_\eta$.
By Lemma~\ref{lemma: technical lemma drift},
\begin{equation}
    \left\|\Tsfit_\eta(\xbm)-\big[\nabla L(\xbm)+\nabla V(\xbm)\big]\right\|^2\le c_{\mathsf{score}}+c_{\mathsf{split}}\|\nabla L(\xbm)+\nabla V(\xbm)\|^2.
\end{equation}
Since $\nabla L+\nabla V$ is $(\beta_L+\beta_V)$-Lipschitz,
\begin{equation}
    \|\nabla L(\xbm)+\nabla V(\xbm)\|^2\leq 2\|\nabla L(\zerobm)+\nabla V(\zerobm)\|^2+2(\beta_L+\beta_V)^2\|\xbm\|^2.
\end{equation}
Therefore,
\begin{equation}\label{eq: approximation linear growth}
    \left\|
    \Tsfit_\eta(\xbm)
    -\big[\nabla L(\xbm)+\nabla V(\xbm)\big]\right\|^2 \leq A_\eta+B_\eta\|\xbm\|^2,
\end{equation}
where
\begin{equation}
    A_\eta \defn c_{\mathsf{score}}+2c_{\mathsf{split}}\|\nabla L(\zerobm)+\nabla V(\zerobm)\|^2, \qquad B_\eta \defn 2c_{\mathsf{split}}(\beta_L+\beta_V)^2.
\end{equation}
Since $c_{\mathsf{split}}=O(\eta^2)$, we can choose $\eta_0>0$ sufficiently small such that
$B_\eta<a^2$ whenever $0<\eta\leq\eta_0$.
Using \eqref{eq: ideal drift dissipativity},
\eqref{eq: approximation linear growth}, and Young's inequality, we obtain
\begin{equation}
\begin{array}{rl}
    \langle \xbm,\Tsfit_\eta(\xbm)\rangle=&\langle \xbm,\nabla L(\xbm)+\nabla V(\xbm)\rangle+\left\langle\xbm,\Tsfit_\eta(\xbm)-\big[\nabla L(\xbm)+\nabla V(\xbm)\big]\right\rangle\\[0.5ex]
    \geq&a\|\xbm\|^2-b_0-\|\xbm\|\left\|\Tsfit_\eta(\xbm)-\big[\nabla L(\xbm)+\nabla V(\xbm)\big]\right\|\\[0.5ex]
    \geq&
    a\|\xbm\|^2-b_0-\displaystyle\frac{a}{2}\|\xbm\|^2-\displaystyle\frac{1}{2a}\left\|\Tsfit_\eta(\xbm)-\big[\nabla L(\xbm)+\nabla V(\xbm)\big]\right\|^2\\[0.5ex]
    \geq&\displaystyle\frac{a^2-B_\eta}{2a}\|\xbm\|^2-\left(b_0+\frac{A_\eta}{2a}\right).
\end{array}
\end{equation}
Hence, we have
\begin{equation}\label{eq: proximal drift dissipativity}
    \langle \xbm,\Tsfit_\eta(\xbm)\rangle
    \geq a_\eta\|\xbm\|^2-b_\eta, \ a_\eta\defn\frac{a^2-B_\eta}{2a}>0, \ b_\eta\defn b_0+\frac{A_\eta}{2a}.
\end{equation}

We also note that the Lipschitz continuity of $\Tsfit_\eta$ directly implies a linear-growth bound.
Indeed, by Lemma~\ref{lemma: lipschitz drifts},
\begin{equation}\label{eq: proximal drift linear growth}
    \|\Tsfit_\eta(\xbm)\|^2\leq2\|\Tsfit_\eta(\zerobm)\|^2+2\Lambda_\eta^2\|\xbm\|^2.
\end{equation}

We now consider one block of multi-block {\proposed}.
Fix $m\geq0$, and define a sequential reference trajectory
$\{\widetilde{\xbm}_{m,i}\}_{i=0}^{N}$ using the same Brownian increments as the $m$-th Picard block.
Set $\widetilde{\xbm}_{m,0}=\xbm_{mN}=\xbm_{mN}^{(0)}$, and, for $i=0,\ldots,N-1$,
\begin{equation}
    \widetilde{\xbm}_{m,i+1}=\widetilde{\xbm}_{m,i}-\gamma\Tsfit_\eta(\widetilde{\xbm}_{m,i})+\sqrt{2}\left(\Bbm_{(mN+i+1)\gamma}-\Bbm_{(mN+i)\gamma}\right).
\end{equation}
This is precisely the sequential X-MC trajectory over the current block.

Since the Brownian increment over
$[(mN+i)\gamma,(mN+i+1)\gamma]$
is independent of $\widetilde{\xbm}_{m,i}$ and has zero mean, by
\eqref{eq: proximal drift dissipativity} and
\eqref{eq: proximal drift linear growth},
\begin{equation}
    \mathbb{E}\|\widetilde{\xbm}_{m,i+1}\|^2\leq\left(1-2a_\eta\gamma+2\gamma^2\Lambda_\eta^2\right)\mathbb{E}\|\widetilde{\xbm}_{m,i}\|^2+\gamma\left(2b_\eta+2\gamma\|\Tsfit_\eta(\zerobm)\|^2+2d\right).
\end{equation}
Choose $\gamma_0>0$ sufficiently small such that, for
$0<\gamma\leq\gamma_0$, define
\begin{equation}
    r_{\eta,\gamma}\defn1-2a_\eta\gamma+2\gamma^2\Lambda_\eta^2\in(0,1), \ C_{\mathsf{seq}}\defn\frac{\gamma\left(2b_\eta+2\gamma\|\Tsfit_\eta(\zerobm)\|^2+2d\right)}{1-r_{\eta,\gamma}}.
\end{equation}
Iterating the preceding inequality gives, for every $i=0,\ldots,N$,
\begin{equation}\label{eq: sequential reference moment}
    \mathbb{E}\|\widetilde{\xbm}_{m,i}\|^2\leq r_{\eta,\gamma}^{\,i}\mathbb{E}\|\xbm_{mN}\|^2+C_{\mathsf{seq}}\leq\mathbb{E}\|\xbm_{mN}\|^2
    +C_{\mathsf{seq}}.
\end{equation}

It remains to control the finite-Picard approximation of this sequential block.
For convenience, set
$\xbm_{mN}^{(k)}=\xbm_{mN}$ for all $k=0,\ldots,K$, and define
\begin{equation}
    D_{m,k}\defn\max_{0\leq i\leq N}\mathbb{E}\left\|\xbm_{mN+i}^{(k)}-\widetilde{\xbm}_{m,i}\right\|^2.
\end{equation}
The sequential reference trajectory satisfies the fixed-point representation
\begin{equation}
    \widetilde{\xbm}_{m,i}=\xbm_{mN}-\gamma\sum_{j=0}^{i-1}\Tsfit_\eta(\widetilde{\xbm}_{m,j})+\sqrt{2}\left(\Bbm_{(mN+i)\gamma}-\Bbm_{mN\gamma}\right).
\end{equation}
Comparing this identity with the Picard update in
Algorithm~\ref{alg: pix multi block}, we obtain, for $i=1,\ldots,N$,
\begin{equation}
\begin{array}{rl}
    \mathbb{E}\left\|\xbm_{mN+i}^{(k+1)}-\widetilde{\xbm}_{m,i}\right\|^2\leq&i\gamma^2\displaystyle\sum_{j=0}^{i-1}\mathbb{E}\left\|\Tsfit_\eta(\xbm_{mN+j}^{(k)})-\Tsfit_\eta(\widetilde{\xbm}_{m,j})\right\|^2\\[0.5ex]
    \leq&i\gamma^2\Lambda_\eta^2\displaystyle\sum_{j=0}^{i-1}\mathbb{E}\left\|\xbm_{mN+j}^{(k)}-\widetilde{\xbm}_{m,j}\right\|^2
    \leq\tau^2\Lambda_\eta^2D_{m,k}=q_{\mathsf{mult}}^2D_{m,k}.
\end{array}
\end{equation}
Taking the maximum over $i$ yields the following contraction that comes directly from the Lipschitz continuity of $\Tsfit_\eta$:
\begin{equation}\label{eq: fixed point picard contraction}
    D_{m,k+1}\leq q_{\mathsf{mult}}^2D_{m,k}.
\end{equation}
Since the zeroth Picard trajectory satisfies
$\xbm_{mN+i}^{(0)}=\xbm_{mN}$,
\eqref{eq: sequential reference moment} gives
\begin{equation}
    D_{m,0}\leq 2\mathbb{E}\|\xbm_{mN}\|^2+2\displaystyle\max_{0\leq i\leq N}\mathbb{E}\|\widetilde{\xbm}_{m,i}\|^2\leq 4\mathbb{E}\|\xbm_{mN}\|^2+2C_{\mathsf{seq}}.
\end{equation}
Therefore, by \eqref{eq: fixed point picard contraction},
\begin{equation}\label{eq: finite picard fixed point error}
    D_{m,K}\leq q_{\mathsf{mult}}^{2K}\left(4\mathbb{E}\|\xbm_{mN}\|^2+2C_{\mathsf{seq}}\right).
\end{equation}
Finally, Algorithm~\ref{alg: pix multi block} sets
$\xbm_{(m+1)N}=\xbm_{(m+1)N}^{(K)}$.
For any $\upsilon>0$, Young's inequality, together with
\eqref{eq: sequential reference moment} for $i=N$ and
\eqref{eq: finite picard fixed point error}, gives
\begin{equation}
\begin{array}{rl}
    \mathbb{E}\|\xbm_{(m+1)N}\|^2\leq&(1+\upsilon)\mathbb{E}\|\widetilde{\xbm}_{m,N}\|^2+(1+\upsilon^{-1})\mathbb{E}\left\|\xbm_{(m+1)N}^{(K)}-\widetilde{\xbm}_{m,N}\right\|^2\\[0.5ex]
    \leq&\left[(1+\upsilon)r_{\eta,\gamma}^{\,N}+4(1+\upsilon^{-1})q_{\mathsf{mult}}^{2K}\right]\mathbb{E}\|\xbm_{mN}\|^2\\
    &+ \left[(1+\upsilon)+2(1+\upsilon^{-1})q_{\mathsf{mult}}^{2K}\right]C_{\mathsf{seq}}.
\end{array}
\end{equation}
Since $r_{\eta,\gamma}^{\,N}<1$, choose $\upsilon>0$ sufficiently small such that $(1+\upsilon)r_{\eta,\gamma}^{\,N}<1$.
Since $q_{\mathsf{mult}}<1$, there exists a finite $K_0$ such that $(1+\upsilon)r_{\eta,\gamma}^{\,N}+4(1+\upsilon^{-1})q_{\mathsf{mult}}^{2K_0}<1$.
Hence, for every $K\geq K_0$, there exist constants
$\chi\in(0,1)$ and $C<\infty$, independent of the block index $m$, such that
\begin{equation}\label{eq: block moment contraction}
    \mathbb{E}\|\xbm_{(m+1)N}\|^2\leq\chi\mathbb{E}\|\xbm_{mN}\|^2+C.
\end{equation}
Iterating \eqref{eq: block moment contraction} over the block index gives
\begin{equation}
    \mathbb{E}\|\xbm_{mN}\|^2\leq\chi^m\mathbb{E}\|\xbm_0\|^2+C\displaystyle\sum_{j=0}^{m-1}\chi^j\leq\mathbb{E}\|\xbm_0\|^2+\displaystyle\frac{C}{1-\chi}.
\end{equation}
Therefore,
\begin{equation}
    \sup_{m\geq0}\mathbb{E}\|\xbm_{mN}\|^2\leq\mathbb{E}\|\xbm_0\|^2+\frac{C}{1-\chi}\defn C_\xbm<\infty.
\end{equation}
The constant $C_\xbm$ does not depend on $M$ or $T=M\tau$.
This proves Assumption \ref{ass: multi-block stability}.
\end{proof}

\section{Proof of multi-block PiX-MC}\label{app:C}
We present the convergence result of multi-block PiX-MC.
Before that, we define some essential constants.
Recall the constants $\rho,A_1,A_4$, and $A_5$ defined in Appendix \ref{app:thm1}. 
Under Assumption \ref{ass: multi-block stability}, define 
\begin{equation}
    G^2\defn2\|\nabla L(\zerobm)+\nabla V(\zerobm)\|^2+2(\beta_L+\beta_V)^2 C_{\xbm}.
\end{equation}
Since $\nabla L+\nabla V$ is $(\beta_L+\beta_V)$-Lipschitz, Assumption \ref{ass: multi-block stability} yields
\begin{equation}
\sup\limits_{0 \leq m \leq M-1} \mathbb{E}\left[\left\|\nabla L\left(\xbm_{m N}\right)+\nabla V\left(\xbm_{m N}\right)\right\|^2\right] \leq G^2 .
\end{equation}
We define the constant $D$.
\begin{equation}
    D= 3\rho A_1 q_{\mathsf{mult}}^{2K-2}R+A_5+2A_4 d(\beta_L+\beta_V).
\end{equation}
Recall that $q_{\mathsf{mult}}=\tau\Lambda_\eta$ is the in-block contraction factor introduced in Section \ref{sec: 3.3 multi}. 
\begin{theorem}[Detailed version of Theorem \ref{thm: main no LSI simple}]\label{thm: main no LSI}
Let $(\mu_t)_{t\geq0}$ denote the law of the continuous interpolation of $\{\Xn^{(K)}\}_{n=0}^{N}$ generated by multi-block {\proposed}.
Suppose Assumptions \ref{ass: L and V 1}-\ref{ass: multi-block stability} hold.
Let the parameters $\eta,\gamma$ be chosen sufficiently small such that
\begin{equation}
    1-\eta \alpha_L>0, 1-3\eta^2\beta_L^2>0,
     1-\frac{18 \gamma^2\left(\beta_L+\beta_V\right)^2}{1-3 \eta^2 \beta_L^2}>0, 
    q_{\mathsf{mult}}<1, A_4<0.75.
\end{equation}
Then,
\begin{equation}
    \left(0.75-A_4\right)
    \frac{1}{T}\int_{0}^{T}\FI(\mu_t\|\pi)dt
    \le
    \frac{\KL(\mu_{0}\|\pi)}{T}
    +
    D.
\end{equation}
\end{theorem}
The conditions involving $\rho$ and $A_4$ are identical to those in Theorem \ref{thm: main no block}.
Before proving Theorem \ref{thm: main no LSI}, we establish the following detailed version of Lemma \ref{lemma: error decay no LSI maintext}, which shows that the error decays geometrically fast in the parallel refinement.
\begin{lemma}[Detailed version of Lemma \ref{lemma: error decay no LSI maintext}]\label{lemma: error decay no LSI}
    Let Assumptions \ref{ass: L and V 1}-\ref{ass: multi-block stability} hold.
    Choose the parameter $\eta,\gamma>0$ sufficiently small such that 
    \begin{equation}
        1-\eta\alpha_L>0, \quad 1-3\eta^2\beta_L^2>0, \quad q_{\mathsf{mult}}=\tau \Lambda_\eta <1.
    \end{equation}
    For each block $m=0,\ldots,M-1$ and each Picard refinement $k=1,\ldots,K$, define
    \begin{equation}
        \varepsilon_{m,k}\defn \max\limits_{1\le i\le N}\mathbb{E}\left\| \xbm_{mN+i}^{(k)} - \xbm_{mN+i}^{(k-1)}\right\|^2.
    \end{equation}
    Then for $k=1,\ldots,K-1$, $\varepsilon_{m,k+1}\le q_{\mathsf{mult}}^2 \varepsilon_{m,k}$. 
    Moreover
    \begin{equation}
    \varepsilon_{m, K} \leq \displaystyle q_{\mathsf{mult}}^{2 K-2}\left[2 \tau d+\frac{3 \tau^2\left(\delta^2+G^2\right)}{1-3 \eta^2 \beta_L^2}\right]=q_{\mathsf{mult}}^{2 K-2} R.
    \end{equation}
    Consequently, 
    \begin{equation}
    \max\limits_{\substack{0 \leq m \leq M-1 \\ 1 \leq i \leq N}} \mathbb{E}\left\|\xbm_{m N+i}^{(K)}-\xbm_{m N+i}^{(K-1)}\right\|^2 \leq q_{\mathsf{mult}}^{2K-2}R.
    \end{equation}
\end{lemma}

\begin{proof}
    Fix a block index $m\in \{0,1,\dots,M-1\}$. 
    For any $k=1,\ldots, K-1$ and $i=1,\ldots,N$, multi-block PiX-MC gives
    \begin{equation}
        \begin{array}{rl}
             \mathbb{E}\|\xbm^{(k+1)}_{mN+i}-\xbm^{(k)}_{mN+i}\|^2
             
             =& \mathbb{E}\left\|\gamma\sum\limits_{j=0}^{i-1}\left[\Tsfit_\eta(\xbm^{(k)}_{mN+j})-\Tsfit_\eta(\xbm^{(k-1)}_{mN+j})\right]\right\|^2  \\
             \le & i\gamma^2 \sum\limits_{j=0}^{i-1}\mathbb{E}\|\Tsfit_\eta(\xbm^{(k)}_{mN+j})-\Tsfit_\eta(\xbm^{(k-1)}_{mN+j})\|^2\\
             \le & i\gamma^2 \Lambda_\eta^2\sum\limits_{j=0}^{i-1}\mathbb{E}\|\xbm^{(k)}_{mN+j}-\xbm^{(k-1)}_{mN+j}\|^2\\
             \le & \tau^2\Lambda_\eta^2 \varepsilon_{m,k}=q_{\mathsf{mult}}^2\varepsilon_{m,k}.
        \end{array}
    \end{equation}
    Thus, $\varepsilon_{m,k+1}\le q_{\mathsf{mult}}^2 \varepsilon_{m,k}$.
    It remains to bound the initial Picard error $\epsilon_{m,1}$. 
    Since the zeroth Picard trajectory is initialized at the block starting point, $\xbm_{mN+i}^{(0)}=\xbm_{mN}$ for $i=0,\ldots,N$, we have, for $i=1,\ldots,N$,
    \begin{equation}
    \xbm_{m N+i}^{(1)}-\xbm_{m N+i}^{(0)}=-\gamma \sum\limits_{j=0}^{i-1} \Tsfit_\eta\left(\xbm_{m N}\right)+\sqrt{2}\left(\Bbm_{(m N+i) \gamma}-\Bbm_{m N \gamma}\right).
    \end{equation}
    Note that 
    \begin{equation}
        \mathbb{E}\left\|\xbm_{m N+i}^{(1)}-\xbm_{m N+i}^{(0)}\right\|^2  =i^2 \gamma^2 \mathbb{E}\left\|\Tsfit_\eta\left(\xbm_{m N}\right)\right\|^2+2 i \gamma d\leq \tau^2 \mathbb{E}\left\|\Tsfit_\eta\left(\xbm_{m N}\right)\right\|^2+2 \tau d .
    \end{equation}
    By Lemma \ref{lemma: technical lemma FBS}, we have
    \begin{equation}
\mathbb{E}\left\|\Tsfit_\eta\left(\xbm_{m N}\right)\right\|^2 \leq \frac{3 \delta^2}{1-3 \eta^2 \beta_L^2}+\frac{3}{1-3 \eta^2 \beta_L^2} \mathbb{E}\left\|\nabla L\left(\xbm_{m N}\right)+\nabla V\left(\xbm_{m N}\right)\right\|^2,
\end{equation}
and that $
\mathbb{E}\left\|\nabla L\left(\xbm_{m N}\right)+\nabla V\left(\xbm_{m N}\right)\right\|^2 \leq G^2$.
    As a result, 
    \begin{equation}
    \mathbb{E}\left\|\Tsfit_\eta\left(\xbm_{m N}\right)\right\|^2 \leq \frac{3\left(\delta^2+G^2\right)}{1-3 \eta^2 \beta_L^2} .
    \end{equation}
    Thus $\epsilon_{m,1}\le 2\tau d+\frac{3\tau^2(\delta^2+G^2)}{1-3\eta^2\beta_L^2}=R$.
    We finally obtain
    \begin{equation}
        \epsilon_{m,K}\le q_{\mathsf{mult}}^{2K-2}\epsilon_{m,1}\le q_{\mathsf{mult}}^{2K-2}R.
    \end{equation}
    Since $R$ is independent of the total time horizon $T$ and the bound is uniform over $m=0,\ldots,M-1$, the global estimate follows.
    This completes the proof.
\end{proof}

\noindent
Now we are ready to prove Theorem \ref{thm: main no LSI}.
\begin{proof}
    Let $m\in \{0,\ldots,M-1\}$ and $i\in\{0,\ldots,N-1\}$. For $t\in \left[(mN+i)\gamma, (mN+i+1)\gamma\right]$, define the continuous-time interpolation by
    \begin{equation}
        \Xbm_t=\xbm^{(K)}_{mN+i}-\left[t-(mN+i)\gamma\right]\Tsfit_\eta\left(\xbm^{(K-1)}_{mN+i}\right)+\sqrt{2}(\Bbm_t-\Bbm_{(mN+i)\gamma}).
    \end{equation}
    Let $\mu_t=\mathsf{Law}(\Xbm_t)$.
    By multi-block PiX-MC, this interpolation satisfies $\Xbm_{(mN+i)\gamma}=\xbm^{(K)}_{mN+i}$ and $\Xbm_{(mN+i+1)\gamma}=\xbm^{(K)}_{mN+i+1}$.

    Applying the same conditional-drift argument as in the proof of Theorem \ref{thm: main no block}, we obtain, for almost every $t\in \left[(mN+i)\gamma, (mN+i+1)\gamma\right]$, 
    \begin{equation}
        \partial_t \KL(\mu_t\|\pi) 
        \le 
        -\frac{3}{4}\FI(\mu_t\|\pi) 
        + 
        \mathbb{E}\left\|\nabla L(\Xbm_t) + \nabla V(\Xbm_t)-\Tsfit_\eta(\xbm^{(K-1)}_{mN+i})\right\|^2.
    \end{equation}
    The local interpolation estimate established in the proof of Theorem \ref{thm: main no block} also applies here without modification, and yields
    \begin{equation}
    \begin{array}{rl}
         &\mathbb{E}\left\|\nabla L(\Xbm_t)+\nabla V(\Xbm_t)-\Tsfit_\eta \left(\xbm^{(K-1)}_{mN+i}\right)\right\|^2 \vspace{1ex} \\
         \le & 3\rho A_1 \mathbb{E}\left\|\xbm^{(K)}_{mN+i}-\xbm^{(K-1)}_{mN+i}\right\|^2+A_4\mathbb{E}\|\nabla L(\Xbm_t)+\nabla V(\Xbm_t)\|^2 + A_5.
    \end{array}
    \end{equation}    
    By Lemma \ref{lemma: Lipschitz and FI}, we have
    \begin{equation}
        \mathbb{E}\|\nabla L(\Xbm_t)+\nabla V(\Xbm_t)\|^2 \le \FI(\mu_t\|\pi)+2d(\beta_L+\beta_V).
    \end{equation}
    For $i=0$, we have $\xbm_{mN}^{(K)} = \xbm_{mN}^{(K-1)}=\xbm_{mN}$.
    Hence for $i=0,\ldots, N-1$, together with Lemma \ref{lemma: error decay no LSI}, we obtain
        \begin{equation}
            \mathbb{E}\left\| \xbm_{mN+i}^{(K)} -\xbm_{mN+i}^{(K-1)}\right\|^2\le q_{\mathsf{mult}}^{2K-2} R.
        \end{equation}  
    Substituting this back, we arrive at
    \begin{equation}\label{eq: temp no LSI closed 2}
         \partial_t \KL(\mu_t\|\pi) \le \left(A_4-0.75\right)\FI(\mu_t\|\pi)+ 3\rho A_1 q_{\mathsf{mult}}^{2K-2} R+A_5 + 2A_4 d (\beta_L+\beta_V).
    \end{equation}
    By the definition of $D$, 
    \begin{equation}\label{eq: temp no LSI closed 3}
         \partial_t \KL(\mu_t\|\pi) \le \left(A_4-0.75\right)\FI(\mu_t\|\pi)+ D.
    \end{equation}
    Integrating \eqref{eq: temp no LSI closed 3} over each interval
    $[(mN+i)\gamma,(mN+i+1)\gamma]$ and summing over $m=0,\ldots,M-1$ and $i=0,\ldots,N-1$, we obtain
    \begin{equation}
        \KL(\mu_T\|\pi)-\KL(\mu_0\|\pi)
        \le 
        \left(A_4-0.75\right)\int_0^T\FI(\mu_t\|\pi)dt
        +
        DT.
    \end{equation}
    Since $\KL(\mu_T\|\pi)\ge 0$, rearranging and dividing by $T$ gives
    \begin{equation}
        \left(0.75-A_4\right)\frac{1}{T}\int_0^T\FI(\mu_t\|\pi)dt 
        \le 
        \frac{\KL(\mu_0\|\pi)}{T} + D.
    \end{equation}
    This completes the proof.
\end{proof}

\section{Proof of multi-block APiX-MC}\label{app:D}
We present the detailed convergence result for multi-block APiX-MC with weighted annealing.
We first introduce several quantities specific to the annealed analysis. Recall the constants $\rho$, $A_1$, and $A_4$ defined in Appendix \ref{app:thm1}.
Define the maximal and averaged score error by
\begin{equation}
    \delta_{\mathsf{ann}}\defn \max\limits_{0\le n\le MN-1} \{\delta_n+(\alpha_n-1)R_\Ssfit\}, \quad \bar{\delta}_{\mathsf{ann}}^2\defn \displaystyle\frac{1}{MN} \sum\limits_{n=0}^{MN-1} [\delta_n+(\alpha_n-1)R_\Ssfit]^2.
\end{equation}
Assumption \ref{ass:4ann} gives
\begin{equation}
     \|\alpha_n \Ssfit_\theta (\xbm;\sigma_n)-\nabla V(\xbm)\|\le \|\Ssfit_\theta (\xbm;\sigma_n)-\nabla V(\xbm)\|+(\alpha_n-1)\|\Ssfit_\theta (\xbm;\sigma_n)\| 
     \le  \delta_n+(\alpha_n-1)R_\Ssfit.
\end{equation}
Since $\nabla L+\nabla V$ is $(\beta_L+\beta_V)$-Lipschitz, Assumption \ref{ass: multi-block stability} gives
\begin{equation}\label{eq:dear 149}
\sup\limits_{0\leq m\leq M-1} \mathbb{E}\left[\left\|\nabla L\left(\xbm_{m N}\right)+\nabla V\left(\xbm_{m N}\right)\right\|^2\right] \leq G^2.
\end{equation}
Since $\{\alpha_n\}_{n=0}^{MN-1}$ is non-increasing, define
\begin{equation}
    \Lambda_{\eta,\mathsf{ann}}\defn \displaystyle \alpha_0 L_\Ssfit + \beta_L \frac{1+\eta \alpha_0 L_\Ssfit}{1-\eta \alpha_L}, \quad q_{\mathsf{ann}}\defn \tau \Lambda_{\eta,\mathsf{ann}},\quad
    R_{\mathsf{ann}} \defn 2\tau d+ \displaystyle\frac{3\tau^2(\delta_{\mathsf{ann}}^2+G^2)}{1-3\eta^2\beta_L^2},
\end{equation}
where $G$ is defined above through the uniform moment bound in Assumption \ref{ass: multi-block stability}.
Finally define the constant
\begin{equation}
    D_{\mathsf{ann}}= 3\rho A_1 q_{\mathsf{ann}}^{2K-2}R_{\mathsf{ann}}+
    \displaystyle\left(\frac{4+9\rho A_1 \gamma^2 }{1-3\eta^2\beta_L^2}\right)\bar{\delta}_{\mathsf{ann}}^2 + 6\rho A_1 \gamma d
    +2A_4 d(\beta_L+\beta_V).
\end{equation}

\begin{theorem}[Detailed version of Theorem \ref{thm: multi-block annealing simple}]\label{thm: multi-block annealing}
Let $(\mu_t)_{t\geq0}$ denote the law of the continuous interpolation of $\{\Xn^{(K)}\}_{n=0}^{N}$ generated by
multi-block {A\proposed}.
Suppose Assumptions \ref{ass: L and V 1}, \ref{ass: multi-block stability}, and \ref{ass:4ann} hold. 
Choose parameters $\eta,\gamma>0$ sufficiently small such that
\begin{equation}
    1-\eta\alpha_L>0,\quad 1-3\eta^2\beta_L^2 > 0,
    \quad 
     1-\frac{18 \gamma^2\left(\beta_L+\beta_V\right)^2}{1-3 \eta^2 \beta_L^2}>0, \quad  q_{\mathsf{ann}}< 1, \quad A_4<0.75.
\end{equation}
Then, 
\begin{equation}
    (0.75-A_4)\frac{1}{T}\int_0^T \FI(\mu_t\|\pi)dt 
    \le 
    \frac{\KL(\mu_0\|\pi)}{T} + D_{\mathsf{ann}}.
\end{equation}
\end{theorem}
The effective discrepancy satisfies
\begin{equation}
\bar{\delta}_{\mathsf{ann}}^2 \leq \frac{2}{M N} \sum_{n=0}^{M N-1} \delta_n^2+\frac{2 R_\Ssfit^2}{MN} \sum_{n=0}^{MN-1}\left(\alpha_n-1\right)^2 .
\end{equation}
Hence, the first term on the right captures the score approximation error, while the second summation represents the transient error introduced by weighted annealing.
In particular, this annealing error vanishes as the post-annealing sampling horizon increases, provided that $\alpha_n=1$ after a finite number of iterations.

Before proving Theorem \ref{thm: multi-block annealing}, we establish the following annealed counterpart of Lemma \ref{lemma: error decay no LSI}.
It shows that the in-block Picard refinement error continues to decay geometrically.

\begin{lemma}\label{lemma: error decay multi-block annealing}
Let Assumptions \ref{ass: L and V 1}, \ref{ass: multi-block stability}, and \ref{ass:4ann} hold.
Choose $\eta,\gamma$ small enough such that 
\begin{equation}
1-\eta\alpha_L>0, \quad 1-3\eta^2\beta_L^2>0, \quad
q_{\mathsf{ann}}=\tau \Lambda_{\eta,\mathsf{ann}}<1 .
\end{equation}
For each block $m=0,\ldots,M-1$ and each Picard refinement $k=1,\ldots,K$, define
\begin{equation}
\varepsilon_{m, k}:=\max\limits_{1 \leq i \leq N} \mathbb{E}\left\|\xbm_{m N+i}^{(k)}-\xbm_{m N+i}^{(k-1)}\right\|^2.
\end{equation}
Then, for $k=1,\ldots,K-1$, $\varepsilon_{m, k+1} \leq q_{\mathsf{ann}}^2 \varepsilon_{m, k}$. Moreover,
\begin{equation}
\varepsilon_{m, K} \leq q_{\mathsf{ann}}^{2 K-2}\left[2 \tau d+\frac{3 \tau^2\left(\delta_{\mathsf{ann}}^2+G^2\right)}{1-3 \eta^2 \beta_L^2}\right]=q_{\mathsf{ann}}^{2 K-2} R_{\mathsf{ann}} .
\end{equation}
Consequently,
\begin{equation}
\max\limits_{\substack{0 \leq m \leq M-1 \\ 1 \leq i \leq N}} \mathbb{E}\left\|\xbm_{m N+i}^{(K)}-\xbm_{m N+i}^{(K-1)}\right\|^2 \leq q_{\mathsf{ann}}^{2 K-2} R_{\mathsf{ann}}.
\end{equation}
At each block starting point, $\xbm_{mN}^{(K)}=\xbm_{mN}^{(K-1)}=\xbm_{mN}$.
Hence, the same bound also holds over all global nodes $n=0,\ldots,MN-1$.
\begin{equation}
\max\limits_{0 \leq n \leq M N-1} \mathbb{E}\left\|\xbm_n^{(K)}-\xbm_n^{(K-1)}\right\|^2 \leq q_{\mathsf{ann}}^{2K-2} R_{\mathsf{ann}}.
\end{equation}
\end{lemma}

\begin{proof}
Fix a block index $m\in\{0,\ldots,M-1\}$. For any $k=1,\ldots,K-1$ and $i=1,\ldots,N$, multi-block APiX-MC gives
\begin{equation}
 \xbm_{m N+i}^{(k+1)}-\xbm_{m N+i}^{(k)} 
 =-\gamma \sum\limits_{j=0}^{i-1}\left[\Tsfit_{\eta, \sigma_{m N+j}}^{\alpha_{m N+j}}\left(\xbm_{m N+j}^{(k)}\right)-\Tsfit_{\eta, \sigma_{m N+j}}^{\alpha_{m N+j}}\left(\xbm_{m N+j}^{(k-1)}\right)\right].
\end{equation}
By the Cauchy-Schwarz inequality and the uniform Lipschitz bound of the weighted drifts,
\begin{equation}
\begin{array}{rl}
     \mathbb{E}\left\|\xbm_{m N+i}^{(k+1)}-\xbm_{m N+i}^{(k)}\right\|^2  \leq & i \gamma^2 \sum\limits_{j=0}^{i-1} \mathbb{E}\left\|\Tsfit_{\eta, \sigma_{m N+j}}^{\alpha_{m N+j}}\left(\xbm_{m N+j}^{(k)}\right)-\Tsfit_{\eta, \sigma_{m N+j}}^{\alpha_{m N+j}}\left(\xbm_{m N+j}^{(k-1)}\right)\right\|^2 \\
\leq &i \gamma^2 \Lambda_{\eta, \mathsf{ann}}^2 \sum\limits_{j=0}^{i-1} \mathbb{E}\left\|\xbm_{m N+j}^{(k)}-\xbm_{m N+j}^{(k-1)}\right\|^2.
\end{array}
\end{equation}
Since $\xbm_{mN}^{(k)}=\xbm_{mN}$ holds for all $k$, 
\begin{equation}
 \mathbb{E}\left\|\xbm_{m N+i}^{(k+1)}-\xbm_{m N+i}^{(k)}\right\|^2 
 \leq i^2 \gamma^2 \Lambda_{\eta, \mathsf{ann}}^2 \varepsilon_{m, k} 
 \leq \tau^2 \Lambda_{\eta, \mathsf{ann}}^2 \varepsilon_{m, k} 
 =q_{\mathsf{ann}}^2 \varepsilon_{m,k} .
\end{equation}
Taking the maximum over $i=1,\ldots,N$ yields $\varepsilon_{m, k+1} \leq q_{\mathsf{ann}}^2 \varepsilon_{m, k}$.
It remains to bound $\varepsilon_{m,1}$. Since the zeroth Picard trajectory is initialized by $\xbm_{mN+i}^{(0)}=\xbm_{mN}$ for $i=0,\ldots,N$, we have
\begin{equation}
\xbm_{m N+i}^{(1)}-\xbm_{m N+i}^{(0)}=-\gamma \sum_{j=0}^{i-1} \Tsfit_{\eta, \sigma_{m N+j}}^{\alpha_{m N+j}}\left(\xbm_{m N}\right)+\sqrt{2}\left(\Bbm_{(m N+i) \gamma}-\Bbm_{m N \gamma}\right).
\end{equation}
Since $\xbm_{mN}$ is independent of the Brownian increments within the current block, the cross term vanishes. Hence,
\begin{equation}
 \mathbb{E}\left\|\xbm_{m N+i}^{(1)}-\xbm_{m N+i}^{(0)}\right\|^2 =\gamma^2 \mathbb{E}\left\|\sum\limits_{j=0}^{i-1} \Tsfit_{\eta, \sigma_{m N+j}}^{\alpha_{m N+j}}\left(\xbm_{m N}\right)\right\|^2+2i\gamma d.
\end{equation}
Applying Cauchy-Schwarz inequality gives
\begin{equation}
\gamma^2 \mathbb{E}\left\|\sum\limits_{j=0}^{i-1} \Tsfit_{\eta, \sigma_{m N+j}}^{\alpha_{m N+j}}\left(\xbm_{m N}\right)\right\|^2\leq i \gamma^2 \sum\limits_{j=0}^{i-1} \mathbb{E}\left\|\Tsfit_{\eta, \sigma_{m N+j}}^{\alpha_{m N+j}}\left(\xbm_{m N}\right)\right\|^2.
\end{equation}
Applying the argument of Lemma \ref{lemma: technical lemma FBS} to the weighted score $\alpha_{mN+j}\Ssfit_\theta(\cdot;\sigma_{mN+j})$, we obtain
\begin{equation}
 \mathbb{E}\left\|\Tsfit_{\eta, \sigma_{m N+j}}^{\alpha_{m N+j}}(\xbm_{m N})\right\|^2 \leq \frac{3 \delta_{\mathsf{ann}}^2}{1-3 \eta^2 \beta_L^2}+\frac{3}{1-3 \eta^2 \beta_L^2} \mathbb{E}\left\|\nabla L(\xbm_{mN})+\nabla V(\xbm_{mN})\right\|^2.
\end{equation}
By \eqref{eq:dear 149}
\begin{equation}
    \mathbb{E}\left\|\nabla L\left(\xbm_{mN}\right)+\nabla V\left(\xbm_{mN}\right)\right\|^2\le G^2.
\end{equation}
Thus 
\begin{equation}
\displaystyle
\mathbb{E}\left\|\Tsfit_{\eta, \sigma_{m N+j}}^{\alpha_{m N+j}}\left(\xbm_{m N}\right)\right\|^2 \leq \frac{3\left(\delta_{\mathsf{ann}}^2+G^2\right)}{1-3 \eta^2\beta_L^2}.
\end{equation}
Consequently, since $i\gamma\leq\tau$,
\begin{equation}
\varepsilon_{m,1} \leq 2 \tau d+\frac{3 \tau^2\left(\delta_{\mathsf{ann}}^2+G^2\right)}{1-3 \eta^2 \beta_L^2}=R_{\mathsf{ann}}.
\end{equation}
Combining this estimate with the geometric recursion gives
\begin{equation}
\varepsilon_{m, K} \leq q_{\mathsf{ann}}^{2 K-2} \varepsilon_{m, 1} \leq q_{\mathsf{ann}}^{2 K-2} R_{\mathsf{ann}}.
\end{equation}
Since the bound is uniform over $m=0,\ldots,M-1$, the global estimate follows.
This completes the proof.
\end{proof}

\noindent
Now we are ready to prove Theorem \ref{thm: multi-block annealing}.
\begin{proof}
For each global time node $n=0,\ldots,MN-1$ and $t\in[n\gamma,(n+1)\gamma]$, define the continuous-time interpolation by 
\begin{equation}
\Xbm_t=\xbm_n^{(K)}-(t-n \gamma) \Tsfit_{\eta, \sigma_n}^{\alpha_n}\left(\xbm_n^{(K-1)}\right)+\sqrt{2}\left(\Bbm_t-\Bbm_{n \gamma}\right).
\end{equation}
Let $\mu_t=\mathsf{Law}(\Xbm_t)$. By multi-block APiX-MC, $\Xbm_{n\gamma}=\xbm_{n}^{(K)}$, $\Xbm_{(n+1)\gamma}=\xbm_{n+1}^{(K)}$.
Applying the same conditional-drift argument as in the proofs of Theorems \ref{thm: main no block} and \ref{thm: main no LSI}, we obtain, for almost every $t\in[n\gamma,(n+1)\gamma]$,
\begin{equation}
    \partial_t\KL(\mu_t\|\pi) \le -\frac{3}{4}\FI(\mu_t\|\pi) + \mathbb{E}\left\|\nabla L(\Xbm_t)+\nabla V(\Xbm_t) - \Tsfit^{\alpha_n}_{\eta,\sigma_n}(\xbm^{(K-1)}_{n})\right\|^2.
\end{equation}
At the $n$-th node, Assumption \ref{ass:4ann} gives
\begin{equation}
\left\|\alpha_n \Ssfit_\theta\left(\xbm ; \sigma_n\right)-\nabla V(\xbm)\right\| \leq \delta_n+\left(\alpha_n-1\right) R_\Ssfit.
\end{equation}
Therefore, repeating the local interpolation estimate in the proof of Theorem \ref{thm: main no block}, with $\delta$ replaced by $\delta_n+\left(\alpha_n-1\right) R_\Ssfit$ gives us
\begin{equation}
\begin{array}{rl}
&\mathbb{E}\left\|\nabla L\left(\Xbm_t\right)+\nabla V\left(\Xbm_t\right)-\Tsfit_{\eta, \sigma_n}^{\alpha_n}\left(\xbm_n^{(K-1)}\right)\right\|^2 \\
\leq & 3 \rho A_1 \mathbb{E}\left\|\xbm_n^{(K)}-\xbm_n^{(K-1)}\right\|^2 +A_4 \mathbb{E}\left\|\nabla L\left(\Xbm_t\right)+\nabla V\left(\Xbm_t\right)\right\|^2 \\
& +\displaystyle\left(\frac{4+9 \rho A_1 \gamma^2}{1-3 \eta^2 \beta_L^2}\right)\left[\delta_n+\left(\alpha_n-1\right) R_\Ssfit\right]^2 +6 \rho A_1 \gamma d.
\end{array}
\end{equation}
By Lemma \ref{lemma: Lipschitz and FI} and Lemma \ref{lemma: error decay multi-block annealing},
\begin{equation}
\begin{array}{rl}
     \partial_t\KL(\mu_t\|\pi) 
     \le &\left(A_4-\frac{3}{4}\right)\FI(\mu_t\|\pi)+3 \rho A_1 q_{\mathsf{ann}}^{2 K-2} R_{\mathsf{ann}} \\
& +\displaystyle\left(\frac{4+9 \rho A_1 \gamma^2}{1-3 \eta^2 \beta_L^2}\right)\left[\delta_n+\left(\alpha_n-1\right) R_\Ssfit\right]^2+6 \rho A_1 \gamma d+2 A_4 d\left(\beta_L+\beta_V\right)  
\end{array}
\end{equation}
Integrating over $[n\gamma,(n+1)\gamma]$, and summing over $n=0,\ldots,MN-1$ yields
\begin{equation}
\begin{array}{rl}
&\KL\left(\mu_T \| \pi\right)-\KL\left(\mu_0 \| \pi\right) \\
\leq & \left(A_4-\frac{3}{4}\right) \int_0^T \FI\left(\mu_t \| \pi\right) dt +3 \rho A_1 q_{\mathsf{ann}}^{2 K-2} R_{\mathsf{ann}} T \vspace{0.5ex} \\
&\displaystyle +\left(\frac{4+9 \rho A_1 \gamma^2}{1-3 \eta^2 \beta_L^2}\right) \gamma \sum\limits_{n=0}^{M N-1}\left[\delta_n+\left(\alpha_n-1\right) R_\Ssfit\right]^2  +\left[6 \rho A_1 \gamma d+2 A_4 d\left(\beta_L+\beta_V\right)\right] T.
\end{array}
\end{equation}
By the definition of $\bar{\delta}_{\mathsf{ann}}^2$, we have $\gamma \sum\limits_{n=0}^{M N-1}\left[\delta_n+\left(\alpha_n-1\right) R_\Ssfit\right]^2=T \bar{\delta}_{\mathsf{ann}}^2$. 
Further, by the definition of $D_{\mathsf{ann}}$,
\begin{equation}
\KL\left(\mu_T \| \pi\right)-\KL\left(\mu_0 \| \pi\right) \leq\left(A_4-\frac{3}{4}\right) \int_0^T \FI\left(\mu_t \| \pi\right)  dt+T D_{\mathsf{ann}} .
\end{equation}
Finally by the non-negativity of $\KL$, the proof is complete.
\end{proof}

\section{Additional experimental details}\label{app:ex}

\subsection{Common Implementation Details and Hyperparameter Selections}\label{app:parameters}
In all experiments, the images are scaled to $[0,1]$. 
Before evaluating a score network, we rescale the input image to $[-1,1]$ and map the network output back to $[0,1]$ accordingly.
The pretrained models used for natural images, MRI, and CT adopt different output parameterizations, whose model-specific implementations are described below.

\paragraph{Score-network implementation.}
For the natural-image experiments, we employ a pretrained diffusion model as in DDPM \citep{ho2020denoising}.
Let $\bar{\alpha}_t$ denote the cumulative diffusion coefficient.
The connection between the score network $\Ssfit_\theta$ and noise predictor $\epsilon_\theta$ in DDPM can be formulated approximately as
\begin{equation}
    \Ssfit_\theta(\xbm;\sigma)=\displaystyle\frac{2}{\sigma}\epsilon_\theta\left(\frac{2\xbm-\onebm }{\sqrt{1+\sigma^2}};\sigma\right),
\end{equation}
 where $\sigma=\displaystyle\sqrt{\frac{1-\bar{\alpha}_t}{\bar{\alpha}_t}}$.
For MRI reconstruction, we employ the pretrained score network from PMC \citep{sun2024provable}.
Denote its output in the network coordinates by $\Ssfit^{\mathsf{MRI}}_\theta(2\xbm-\onebm;\sigma)$. 
Since the network directly estimates the positive score, the negative-score function used in our formulation is
\begin{equation}
    \Ssfit_\theta(\xbm;\sigma)=-2\Ssfit^{\mathsf{MRI}}_\theta(2\xbm-\onebm;\sigma).
\end{equation}
For CT images, we train an MMSE denoiser $\Dsfit_\theta$ using the standard EDM framework \citep{karras2022elucidating}. 
The relation between $\Ssfit_\theta$ and $\Dsfit_\theta$ is
\begin{equation}
    \Ssfit_\theta(\xbm;\sigma)=\displaystyle\frac{2}{\sigma^2}\left[2\xbm-\onebm - \Dsfit_\theta(2\xbm-\onebm;\sigma)\right].
\end{equation}

\paragraph{Balancing parameter for likelihood.}
To account for the different numerical scales of the likelihood and score networks, we introduce a task-dependent balancing parameter $\lambda>0$ for the likelihood.
For example, MRI, deblurring, and CT, the likelihood is 
\begin{equation}
    L(\xbm)=\displaystyle\frac{\lambda}{2}\|\Abm\xbm-\ybm\|^2,
\end{equation}
where $\Abm$ is a task-specific forward operator. 

\paragraph{Annealing Schedule.}
Let $n=mN+i$ denote the global time node index, where $m=0,\ldots,M-1$ is the block index and $i=0,\ldots,N-1$ is the local node index.
Recall that we use the geometric schedule
\begin{equation}
\sigma_n=\max\{\sigma_{\max}\xi^n,\sigma_{\min}\},
\end{equation}
where $0<\xi<1$ is the annealing parameter.
For the score network weight schedule $\alpha_n$, we also use a corresponding geometric sequence 
\begin{equation}
    \alpha_n=(\sigma_n/\sigma_{\min})^a
\end{equation}
This schedule is non-increasing and becomes one once $\sigma_n$ reaches $\sigma_{\min}$.
For the MRI and CT experiments, we set $a=2$; for the deblurring and Rician denoising experiments, we set $a=5/2$.

\paragraph{Picard stopping criteria.} 
For deblurring and CT experiments, we employ adaptive Picard refinement. We compute the empirical in-block residual
\begin{equation}
\varepsilon_{m,k}=\max\limits_{1 \leq i \leq N}\left\|\xbm_{mN+i}^{(k)}- \xbm_{m N+i}^{(k-1)}\right\| .
\end{equation}
The refinement within a block is terminated when either $k$ reaches $K=20$ or the normalized residual $\varepsilon_{m,k}/d<3\times 10^{-4}$.

\paragraph{Hyperparameter selection.}
The hyperparameters are tuned separately for each method and imaging task to obtain the best reconstruction results.
The gradient-based and proximal variants use different drift parameterizations and are therefore assigned method-specific stepsizes.
In particular, the implicit treatment of the likelihood in the proximal variants may improve numerical stability and permit larger stepsizes in practice.
For each paired sequential and Picard-based method, we use the same Langevin stepsize, score schedule, likelihood weight, and total number of time nodes.
The complete hyperparameter settings are reported in Table~\ref{tab:algorithmic_hyperparameters}.

\begin{table}[ht]
\centering
\scriptsize
\setlength{\tabcolsep}{3.5pt}
\renewcommand{\arraystretch}{1.08}
\begin{tabular}{llccccccccc}
\toprule
Inverse problem
& Method
& $\lambda$
& $\gamma$
& $\eta$
& $\xi$
& $\sigma_{\min}$
& $\sigma_{\max}$
& $M$
& $N$
& $K$ \\
\midrule

\multirow{4}{*}{MRI}
& PiL-MC  & $1.6\times 10^3$ & $6.25\times10^{-5}$ & -- & $1$ & $0.05$ & $0.05$ & $10$ & $75$ & $100$ \\
& APiL-MC & $1.6\times 10^3$ & $6.25\times10^{-5}$ & -- &$0.98$ & $0.05$ & $0.2$& $10$ & $75$ & $100$ \\
& PiX-MC  & $3.5\times10^3$ & $8.75\times10^{-5}$ & $4.375\times 10^{-4}$ & $1$& $0.05$ & $0.05$ & $10$ & $75$ & $100$ \\
& APiX-MC & $3.5\times10^3$ & $8.75\times10^{-5}$ & $4.375\times 10^{-4}$ & $0.98$ & $0.05$ & $0.2$&  $10$ & $75$ & $100$ \\
\midrule

\multirow{4}{*}{Rician}
& PiL-MC  & $2.2$ & $1\times 10^{-4}$ & -- & $1$ & $0.1$ & $0.1$ & $5$ & $100$ & $50$ \\
& APiL-MC & $2.2$ & $1\times 10^{-4}$ & -- & $0.98$ & $0.1$ & $0.4$ & $5$ & $100$ & $50$ \\
& PiX-MC  & $2.2$ & $1\times 10^{-4}$ & $1\times 10^{-3}$ & $1$& $0.1$ &$0.1$ &  $5$ & $100$ & $50$ \\
& APiX-MC & $2.2$ & $1\times 10^{-4}$ & $1\times 10^{-3}$ & $0.98$& $0.1$& $0.4$& $5$ & $100$ & $50$ \\
\midrule

\multirow{4}{*}{Deblurring}
& PiL-MC  & $800$ & $1\times 10^{-4}$ & -- & $1$ & $0.1$ & $0.1$ & $125$ &$8$ &  $20$ \\
& APiL-MC & $1100$ & $1\times 10^{-4}$ & -- & $0.98$ & $0.1$ & $2.0$ & $125$ &$8$ &  $20$ \\
& PiX-MC  & $900$ & $1\times 10^{-4}$ & $1\times 10^{-3}$ & $1$ & $0.1$ & $0.1$& $125$ & $8$ & $20$ \\
& APiX-MC & $1150$ & $1\times 10^{-4}$ & $1\times 10^{-3}$ & $0.98$ & $0.1$ & $2.0$ & $125$& $8$ & $20$ \\
\midrule

\multirow{4}{*}{3D CT}
& PiL-MC  & $1.11$ & $1\times 10^{-5}$ & -- & $1$ & $0.03$ &$ 0.03$ & $3875$ & $8 $ & $20$ \\
& APiL-MC & $1.11$ & $1\times 10^{-5}$ & -- & $0.995$ &$0.03$ & $0.5$ & $3875$ & $8 $ & $20$  \\
& PiX-MC  & $2.22$ & $4.5\times 10^{-5}$ & $9\times 10^{-4}$ & $1$ & $0.085$ & $0.085$ & $3875$ & $8 $ & $20$  \\
& APiX-MC & $2.22$ & $4.5\times 10^{-5}$ & $9\times 10^{-4}$ & $0.999$ & $0.085$ & $0.5$ & $3875$ & $8 $ & $20$ \\
\bottomrule
\end{tabular}
\caption{Algorithmic hyperparameter settings for the real-world imaging experiments.}
\label{tab:algorithmic_hyperparameters}
\end{table}

\subsection{Likelihood and Proximal Implementations}\label{app:E2 likelihood}
We next provide the task-specific implementations of the likelihood gradient and proximal operator.
The datasets, score networks, and hyperparameters are reported in \$ \ref{sec:experiments} and Appendix \ref{app:parameters}.

\paragraph{MRI reconstruction.}
The forward operator $\Abm=\Mbm\mathcal{F}$ consists of the discrete Fourier transform $\mathcal{F}$ followed by a radial subsampling mask. The likelihood and its gradient are 
\begin{equation}
    L(\xbm)=\displaystyle\frac{\lambda}{2}\|\Mbm\mathcal{F}\xbm - \ybm\|^2, \quad \nabla L(\xbm) = \lambda\mathcal{F}^\mathrm{H} \Mbm^\mathrm{T} \left( \Mbm \mathcal{F} \xbm-\ybm\right).
\end{equation}
The proximal operator has the closed form
\begin{equation}
    \prox_{\eta L}(\zbm) = \left(\eta\lambda \mathcal{F}^\mathrm{H}\Mbm^\mathrm{T}\Mbm \mathcal{F} + I\right)^{-1}\left(\zbm+\eta\lambda \mathcal{F}^\mathrm{H} \Mbm^\mathrm{T} \ybm\right).
\end{equation}
Since $\mathcal{F}$ is unitary and $\Mbm^\mathrm{T}\Mbm$ is diagonal in the Fourier domain, the inverse above can be evaluated elementwise in k-space, without the need to solve a large-scale linear system.

\paragraph{Rician denoising.}
Let $\ebm_{\mathsf{real}}$ and $\ebm_{\mathsf{imag}}$ be two i.i.d. Gaussian noise vectors with zero mean and standard deviation $\sigma_{\ybm}$. The Rician-noise-corrupted observation $\ybm$ is obtained by:
\begin{equation}
    \ybm = \sqrt{(\xbm+\ebm_{\mathsf{real}})^2+\ebm_{\mathsf{imag}}^2},
\end{equation}
where theoperations are elementwise.
According to \citep{gudbjartsson1995rician}, the conditional probability density function (PDF) of $\ybm$ is
\begin{equation}
    p(\ybm|\xbm)=\prod_{r=1}^d\frac{\ybm_r}{\sigma_{\ybm}^2}\exp\left(-\frac{\xbm_r^2+\ybm_r^2}{2\sigma_{\ybm}^2}\right)I_0\left(\frac{\xbm_r \ybm_r}{\sigma_{\ybm}^2}\right),
\end{equation}
which is known as the \textit{Rice} or \textit{Rician} distribution. Here $I_0$ is the modified Bessel function of the first kind of order zero.
Consider the modified Bessel's differential equations:
\begin{align*}
x^2 y'' +xy' -(x^2 + n^2)y=0, \quad n\ge 0.
\end{align*}
The solutions of the previous equation $I_n(x)=y(x)$ continuous in zero are called the modified Bessel function of the first kind of order $n$. 

After omitting terms independent of $\xbm$, the negative log-likelihood $L$ for Rician noise removal is 
\begin{equation}
    L(\xbm) = \lambda \displaystyle \left\langle \onebm, \frac{\xbm^2}{2\sigma^2_{\ybm}}-\log I_0 \left( \frac{\xbm\odot\ybm}{\sigma^2_{\ybm}}\right) \right\rangle.
\end{equation}
Note that according to Lemma 12 in \citep{renaud2026provably}, $L$ is Lipschitz smooth, thus satisfies Assumption \ref{ass: L and V 1}.
Following \citep{wei2023nonconvex}, 
its gradient is evaluated using the ratio $B(\xbm)=I_1(\xbm)/I_0(\xbm)$, and the pixel-wise proximal subproblem is solved using the iterative reweighted $\ell_1$ (IRL1) method. 
The IRL1 algorithm for solving $\prox_{\eta L}(\bm{z})$ is given as follows:
\begin{equation}
\begin{array}{rl}
\bm{x}^{t+\frac{1}{2}} &=\displaystyle\frac{\bm{y}}{\sigma_y^2}B\left(\frac{\bm{x}^t\bm{ \odot y}}{\sigma_y^2}\right)  \\
\bm{x}^{t+1} &= \displaystyle\arg\min\limits_{\xbm} \eta\lambda\left\langle \textbf{1}, \frac{\bm{x}^2}{2\sigma_{\ybm}^2}-\frac{\bm{x\odot y}}{\sigma_{\ybm}^2}\right\rangle+\frac{1}{2}\|\bm{x-z}\|^2-\eta\lambda\langle \bm{x}^{t+\frac{1}{2}}, \bm{x}\rangle\vspace{0.5ex}\vspace{0.5ex}\\
&=\displaystyle\frac{ \displaystyle\bm{z}+\frac{\eta\lambda}{\sigma_{\ybm}^2}\bm{y}+\eta\lambda \bm{x}^{t+\frac{1}{2}}}{1+\displaystyle\frac{\eta\lambda}{\sigma_{\ybm}^2}}.
\end{array}
\end{equation}
In experiments, we set the maximum inner iteration number for IRL1 to be $10$ because it typically solves the proximal sub-problem within $5$ iterations.

\paragraph{Image deblurring.}
Following \citep{renaud2026provably}, we model the forward operator as a circular convolution, which is diagonalized by the discrete Fourier transform. 
The likelihood is $L(\xbm)=\frac{\lambda}{2}\|\Kbm\xbm-\ybm\|^2$, where $\Kbm$ is the circular convolution operator associated with a blur kernel $\kbm$. 
The likelihood gradient and proximal operator are
\begin{equation}
    \nabla L(\xbm)=\lambda \Kbm^{\mathrm{T}}(\Kbm\xbm-\ybm), \quad \prox_{\eta L}(\zbm)=\left(\eta\lambda\Kbm^{\mathrm{T}}\Kbm + I\right)^{-1} (\zbm+\eta\lambda \Kbm^{\mathrm{T}}\ybm).
\end{equation}
The proximal operator can be evaluated efficiently in the Fourier domain through elementwise multiplication and division:
\begin{equation}
        \prox_{\eta L}(\zbm) = \displaystyle\mathcal{F}^{\mathrm{H}}\left(
        \frac{\mathcal{F}(\zbm+\eta\lambda \Kbm^{\mathrm{T}}\ybm)}{\mathcal{F}(\eta\lambda\Kbm^\mathrm{T}\Kbm + \onebm)}\right).
\end{equation}

\paragraph{3D CT reconstruction.} 
The quadratic likelihood potential and its gradient are given by
\begin{equation}
    L(\xbm)=\frac{\lambda}{2}\|\Abm\xbm-\ybm\|^2, \quad \nabla L(\xbm) = \lambda \Abm^\mathrm{T}(\Abm\xbm-\ybm).
\end{equation}
The proximal operator is evaluated by solving the following linear system:
\begin{equation}
    \xbm=\prox_{\eta L}(\zbm)\iff (I+\eta\lambda \Abm^\mathrm{T}\Abm) \xbm = \zbm + \eta\lambda \Abm^\mathrm{T} \ybm.
\end{equation}
We solve this linear system using five conjugate gradient iterations \citep{shewchuk1994introduction}.
To promote consistency between adjacent slices, we additionally use a one-dimensional inter-slice Huber-TV (HTV) regularizer.
For $\xbm\in\R^{512\times512\times80}$, HTV regularizer $\Hbm(\xbm)$ with smoothing parameter $\delta_{\textsf{H}}>0$ is
\begin{equation}
\Hbm(\xbm)=\sum\limits_{i=1}^{512}\sum\limits_{j=1}^{512}\sum\limits_{s=1}^{79}H_{\delta_{\textsf{H}}}( \xbm(i,j,s+1)-\xbm(i,j,s)),
\end{equation}
where $\xbm(i,j,s)$ denotes the element located at $i$-th row, $j$-th column in $s$-th slice.
For any $x\in \R$, $H_{\delta_{\textsf{H}}}(x)$ is defined as
\begin{equation}
    H_{\delta_{\textsf{H}}}(x):=\left\{
    \begin{array}{lll}
           &\displaystyle \frac{x^2}{2\delta_{\textsf{H}}}, & \text{if } |x|\le\delta_{\textsf{H}};
           \\ 
         &|x| -\displaystyle\frac{\delta_{\textsf{H}}}{2},& \text{if }|x|>\delta_{\textsf{H}}.
    \end{array}\right.
\end{equation}
We set $\delta_{\textsf{H}}=10^{-8}$ in the experiments. 
We view HTV as an inter-slice prior for CT volumes characterizing correlations between adjacent slices.
Accordingly, the proximal drift $\Tsfit_{\eta,\sigma}$ is
\begin{equation}
    \Tsfit_{\eta,\sigma}\defn\frac{1}{\eta}\Big(I-\prox_{\eta L}\circ\big(I-\eta \Ssfit_\theta(\cdot;\sigma)-\eta\beta_H\nabla \Hbm\big)\Big),
\end{equation}
where $\beta_H>0$ is a balancing parameter for $\Hbm$. We set $\beta_H=0.28$ for all methods.

\vskip 0.2in
\bibliographystyle{plain}
\bibliography{references}

@article{yang2022convergence,
  title={Convergence of the inexact Langevin algorithm and score-based generative models in KL divergence},
  author={Yang, Kaylee Yingxi and Wibisono, Andre},
  journal={arXiv preprint arXiv:2211.01512},
  year={2022}
}

@article{chen2023probability,
  title={The probability flow ode is provably fast},
  author={Chen, Sitan and Chewi, Sinho and Lee, Holden and Li, Yuanzhi and Lu, Jianfeng and Salim, Adil},
  journal={Advances in Neural Information Processing Systems},
  volume={36},
  pages={68552--68575},
  year={2023}
}

@article{lee2022convergence,
  title={Convergence for score-based generative modeling with polynomial complexity},
  author={Lee, Holden and Lu, Jianfeng and Tan, Yixin},
  journal={Advances in Neural Information Processing Systems},
  volume={35},
  pages={22870--22882},
  year={2022}
}

@inproceedings{raginsky2017non,
  title={Non-convex learning via stochastic gradient langevin dynamics: a nonasymptotic analysis},
  author={Raginsky, Maxim and Rakhlin, Alexander and Telgarsky, Matus},
  booktitle={Conference on Learning Theory},
  pages={1674--1703},
  year={2017},
  organization={PMLR}
}

@misc{shewchuk1994introduction,
  title={An introduction to the conjugate gradient method without the agonizing pain},
  author={Shewchuk, Jonathan Richard and others},
  year={1994},
  publisher={Carnegie-Mellon University. Department of Computer Science Pittsburgh}
}

@article{wei2023nonconvex,
  title={Nonconvex Rician noise removal via convergent plug-and-play framework},
  author={Wei, Deliang and Weng, Shiyang and Li, Fang},
  journal={Applied Mathematical Modelling},
  volume={123},
  pages={197--212},
  year={2023},
  publisher={Elsevier}
}

@inproceedings{jayasumana2024rethinking,
  title={Rethinking fid: Towards a better evaluation metric for image generation},
  author={Jayasumana, Sadeep and Ramalingam, Srikumar and Veit, Andreas and Glasner, Daniel and Chakrabarti, Ayan and Kumar, Sanjiv},
  booktitle={Proceedings of the IEEE/CVF conference on computer vision and pattern recognition},
  pages={9307--9315},
  year={2024}
}

@inproceedings{karras2019style,
  title={A style-based generator architecture for generative adversarial networks},
  author={Karras, Tero and Laine, Samuli and Aila, Timo},
  booktitle={Proceedings of the IEEE/CVF conference on computer vision and pattern recognition},
  pages={4401--4410},
  year={2019}
}

@article{pan2016l_0,
  title={$ l\_0 $-regularized intensity and gradient prior for deblurring text images and beyond},
  author={Pan, Jinshan and Hu, Zhe and Su, Zhixun and Yang, Ming-Hsuan},
  journal={IEEE transactions on pattern analysis and machine intelligence},
  volume={39},
  number={2},
  pages={342--355},
  year={2016},
  publisher={IEEE}
}

@inproceedings{kim2023differentiable,
  title={Differentiable Forward Projector for X-ray Computed Tomography},
  author={Kim, Hyojin and Champley, Kyle},
  booktitle={ICML 2023 Workshop on Differentiable Almost Everything: Differentiable Relaxations, Algorithms, Operators, and Simulators}
}

@inproceedings{levin2009understanding,
  title={Understanding and evaluating blind deconvolution algorithms},
  author={Levin, Anat and Weiss, Yair and Durand, Fredo and Freeman, William T},
  booktitle={2009 IEEE conference on computer vision and pattern recognition},
  pages={1964--1971},
  year={2009},
  organization={IEEE}
}

@InProceedings{Agustsson_2017_CVPR_Workshops,
	author = {Agustsson, Eirikur and Timofte, Radu},
	title = {NTIRE 2017 Challenge on Single Image Super-Resolution: Dataset and Study},
	booktitle = {The IEEE Conference on Computer Vision and Pattern Recognition (CVPR) Workshops},
	month = {July},
	year = {2017}
}

@inproceedings{martin2001database,
  title={A database of human segmented natural images and its application to evaluating segmentation algorithms and measuring ecological statistics},
  author={Martin, David and Fowlkes, Charless and Tal, Doron and Malik, Jitendra},
  booktitle={Proceedings eighth IEEE international conference on computer vision. ICCV 2001},
  volume={2},
  pages={416--423},
  year={2001},
  organization={Ieee}
}

@article{gudbjartsson1995rician,
  title={The Rician distribution of noisy MRI data},
  author={Gudbjartsson, H{\'a}kon and Patz, Samuel},
  journal={Magnetic resonance in medicine},
  volume={34},
  number={6},
  pages={910--914},
  year={1995},
  publisher={Wiley Online Library}
}

@article{zbontar2018fastmri,
  title={fastMRI: An open dataset and benchmarks for accelerated MRI},
  author={Zbontar, Jure and Knoll, Florian and Sriram, Anuroop and Murrell, Tullie and Huang, Zhengnan and Muckley, Matthew J and Defazio, Aaron and Stern, Ruben and Johnson, Patricia and Bruno, Mary and others},
  journal={arXiv preprint arXiv:1811.08839},
  year={2018}
}

@article{karras2022elucidating,
  title={Elucidating the design space of diffusion-based generative models},
  author={Karras, Tero and Aittala, Miika and Aila, Timo and Laine, Samuli},
  journal={Advances in neural information processing systems},
  volume={35},
  pages={26565--26577},
  year={2022}
}

@book{nesterov2004,
	author = {Y. Nesterov},
	publisher = {Kluwer Academic Publishers},
	title = {Introductory Lectures on Convex Optimization: A Basic Course},
	year = {2004}}

@article{vincent2011a,
	author = {Vincent, Pascal},
	journal = {Neural Comput.},
	number = {7},
	pages = {1661--1674},
	title = {A Connection between Score Matching and Denoising Autoencoders},
	volume = {23},
	year = {2011}}

@inbook{beck.teboulle2009convex,
	author = {A. Beck and M. Teboulle},
	chapter = {Gradient-Based Algorithms with Applications to Signal Recovery Problems},
	pages = {42--88},
	publisher = {Cambridge},
	title = {Convex Optimization in Signal Processing and Communications},
	year = {2009}}

@article{afonso2010fast,
	author = {M. V. Afonso and J. M.Bioucas-Dias and M. A. T. Figueiredo},
	journal = {IEEE Trans. Image Process.},
	month = {Sep.},
	number = {9},
	pages = {2345--2356},
	title = {Fast Image Recovery Using Variable Splitting and Constrained Optimization},
	volume = {19},
	year = {2010}}

@inproceedings{song2019generative,
	author = {Song, Yang and Ermon, Stefano},
	booktitle = {Advances in Neural Information Processing Systems},
	title = {Generative Modeling by Estimating Gradients of the Data Distribution},
	volume = {32},
	year = {2019}}

@article{beck.teboulle2009fast,
	author = {A. Beck and M. Teboulle},
	journal = {IEEE Trans. Image Process.},
	month = {November},
	number = {11},
	pages = {2419--2434},
	title = {Fast Gradient-Based Algorithm for Constrained Total Variation Image Denoising and Deblurring Problems},
	volume = {18},
	year = {2009}}

@article{boyd2011,
	author = {S. Boyd and N. Parikh and E. Chu and B. Peleato and J. Eckstein},
	journal = {Found. Trends Mach. Learn.},
	number = {1},
	pages = {1--122},
	title = {Distributed optimization and statistical learning via the alternating direction method of multipliers},
	volume = {3},
	year = {2011}}

@article{donoho2006compressed,
	author = {D. L. Donoho},
	journal = {IEEE Trans. Inf. Theory},
	month = {April},
	number = {4},
	pages = {1289--1306},
	title = {Compressed Sensing},
	volume = {52},
	year = {2006}}

@article{candes2006robust,
	author = {E. J. Cand{\`e}s and J. Romberg and T. Tao},
	journal = {IEEE Trans. Inf. Theory},
	month = {February},
	number = {2},
	pages = {489--509},
	title = {Robust Uncertainty Principles: Exact Signal Reconstruction From Highly Incomplete Frequency Information},
	volume = {52},
	year = {2006}}

@article{neal2001,
	author = {Neal, Radford M.},
	journal = {Statistics and Computing},
	number = {2},
	pages = {125--139},
	title = {Annealed importance sampling},
	volume = {11},
	year = {2001}}

@article{parisi1981correlation,
  title={Correlation functions and computer simulations},
  author={Parisi, Giorgio},
  journal={Nuclear Physics B},
  volume={180},
  number={3},
  pages={378--384},
  year={1981},
  publisher={Elsevier}
}

@inproceedings{zhou2018distributed,
  title={Distributed asynchronous optimization with unbounded delays: How slow can you go?},
  author={Zhou, Zhengyuan and Mertikopoulos, Panayotis and Bambos, Nicholas and Glynn, Peter and Ye, Yinyu and Li, Li-Jia and Fei-Fei, Li},
  booktitle={International Conference on Machine Learning},
  pages={5970--5979},
  year={2018},
  organization={PMLR}
}

@inproceedings{hannah2019a2bcd,
  title={A2BCD: Asynchronous acceleration with optimal complexity},
  author={Hannah, Robert and Feng, Fei and Yin, Wotao},
  booktitle={International Conference on Learning Representations},
  year={2019}
}

@article{richtarik2016parallel,
  title={Parallel coordinate descent methods for big data optimization},
  author={Richt{\'a}rik, Peter and Tak{\'a}{\v{c}}, Martin},
  journal={Mathematical Programming},
  volume={156},
  number={1},
  pages={433--484},
  year={2016},
  publisher={Springer}
}

@article{sun2017asynchronous,
  title={Asynchronous coordinate descent under more realistic assumptions},
  author={Sun, Tao and Hannah, Robert and Yin, Wotao},
  journal={Advances in Neural Information Processing Systems},
  volume={30},
  year={2017}
}

@article{lian2015asynchronous,
  title={Asynchronous parallel stochastic gradient for nonconvex optimization},
  author={Lian, Xiangru and Huang, Yijun and Li, Yuncheng and Liu, Ji},
  journal={Advances in neural information processing systems},
  volume={28},
  year={2015}
}

@article{lian2017can,
  title={Can decentralized algorithms outperform centralized algorithms? a case study for decentralized parallel stochastic gradient descent},
  author={Lian, Xiangru and Zhang, Ce and Zhang, Huan and Hsieh, Cho-Jui and Zhang, Wei and Liu, Ji},
  journal={Advances in neural information processing systems},
  volume={30},
  year={2017}
}

@article{fercoq2015accelerated,
  title={Accelerated, parallel, and proximal coordinate descent},
  author={Fercoq, Olivier and Richt{\'a}rik, Peter},
  journal={SIAM Journal on Optimization},
  volume={25},
  number={4},
  pages={1997--2023},
  year={2015},
  publisher={SIAM}
}

@article{peng2016arock,
  title={Arock: an algorithmic framework for asynchronous parallel coordinate updates},
  author={Peng, Zhimin and Xu, Yangyang and Yan, Ming and Yin, Wotao},
  journal={SIAM Journal on Scientific Computing},
  volume={38},
  number={5},
  pages={A2851--A2879},
  year={2016},
  publisher={SIAM}
}

@article{recht2011hogwild,
  title={Hogwild!: A lock-free approach to parallelizing stochastic gradient descent},
  author={Recht, Benjamin and Re, Christopher and Wright, Stephen and Niu, Feng},
  journal={Advances in neural information processing systems},
  volume={24},
  year={2011}
}

@inproceedings{sun2021asyncred,
title={{Async-RED}: A Provably Convergent Asynchronous Block Parallel Stochastic Method using Deep Denoising Priors},
author={Yu Sun and Jiaming Liu and Yiran Sun and Brendt Wohlberg and Ulugbek Kamilov},
booktitle={International Conference on Learning Representations},
year={2021}
}

@book{brooks2011handbook,
	author = {Brooks, Steve and Gelman, Andrew and Jones, Galin and Meng, Xiao-Li},
	publisher = {CRC press},
	title = {Handbook of markov chain monte carlo},
	year = {2011}}

@article{jalal2021robust,
  title={Robust compressed sensing MRI with deep generative priors},
  author={Jalal, Ajil and Arvinte, Marius and Daras, Giannis and Price, Eric and Dimakis, Alexandros G and Tamir, Jon},
  journal={Advances in neural information processing systems},
  volume={34},
  pages={14938--14954},
  year={2021}
}

@article{melidonis2023efficient,
  title={Efficient Bayesian computation for low-photon imaging problems},
  author={Melidonis, Savvas and Dobson, Paul and Altmann, Yoann and Pereyra, Marcelo and Zygalakis, Konstantinos},
  journal={SIAM Journal on Imaging Sciences},
  volume={16},
  number={3},
  pages={1195--1234},
  year={2023},
  publisher={SIAM}
}

@inproceedings{spagnoletti2025latino,
  title={Latino-pro: Latent consistency inverse solver with prompt optimization},
  author={Spagnoletti, Alessio and Prost, Jean and Almansa, Andr{\'e}s and Papadakis, Nicolas and Pereyra, Marcelo},
  booktitle={Proceedings of the IEEE/CVF International Conference on Computer Vision},
  pages={19597--19607},
  year={2025}
}

@article{coeurdoux2024normalizing,
  title={Normalizing flow sampling with Langevin dynamics in the latent space},
  author={Coeurdoux, Florentin and Dobigeon, Nicolas and Chainais, Pierre},
  journal={Machine Learning},
  volume={113},
  number={11},
  pages={8301--8326},
  year={2024},
  publisher={Springer}
}

@article{holden2022bayesian,
  title={Bayesian imaging with data-driven priors encoded by neural networks},
  author={Holden, Matthew and Pereyra, Marcelo and Zygalakis, Konstantinos C},
  journal={SIAM Journal on Imaging Sciences},
  volume={15},
  number={2},
  pages={892--924},
  year={2022},
  publisher={SIAM}
}

@article{durmus2019analysis,
  title={Analysis of Langevin Monte Carlo via convex optimization},
  author={Durmus, Alain and Majewski, Szymon and Miasojedow, B{\l}a{\.z}ej},
  journal={Journal of Machine Learning Research},
  volume={20},
  number={73},
  pages={1--46},
  year={2019}
}

@inproceedings{lau2022non,
  title={Non-log-concave and nonsmooth sampling via Langevin Monte Carlo algorithms},
  author={Lau, Tim Tsz-Kit and Liu, Han and Pock, Thomas},
  booktitle={INdAM Workshop: Advanced Techniques in Optimization for Machine learning and Imaging},
  pages={83--149},
  year={2022},
  organization={Springer}
}

@article{durmus2018efficient,
  title={Efficient bayesian computation by proximal markov chain monte carlo: when langevin meets moreau},
  author={Durmus, Alain and Moulines, Eric and Pereyra, Marcelo},
  journal={SIAM Journal on Imaging Sciences},
  volume={11},
  number={1},
  pages={473--506},
  year={2018},
  publisher={SIAM}
}

@inproceedings{ding2021random,
  title={Random coordinate langevin monte carlo},
  author={Ding, Zhiyan and Li, Qin and Lu, Jianfeng and Wright, Stephen J},
  booktitle={Conference on learning theory},
  pages={1683--1710},
  year={2021},
  organization={PMLR}
}

@article{salim2020primal,
  title={Primal dual interpretation of the proximal stochastic gradient Langevin algorithm},
  author={Salim, Adil and Richtarik, Peter},
  journal={Advances in Neural Information Processing Systems},
  volume={33},
  pages={3786--3796},
  year={2020}
}

@article{salim2019stochastic,
  title={Stochastic proximal Langevin algorithm: Potential splitting and nonasymptotic rates},
  author={Salim, Adil and Kovalev, Dmitry and Richt{\'a}rik, Peter},
  journal={Advances in Neural Information Processing Systems},
  volume={32},
  year={2019}
}

@article{wibisono2019proximal,
  title={Proximal langevin algorithm: Rapid convergence under isoperimetry},
  author={Wibisono, Andre},
  journal={arXiv preprint arXiv:1911.01469},
  year={2019}
}

@article{renaud2026stability,
  title={From stability of Langevin diffusion to convergence of proximal MCMC for non-log-concave sampling},
  author={Renaud, Marien and De Bortoli, Valentin and Leclaire, Arthur and Papadakis, Nicolas},
  journal={Advances in Neural Information Processing Systems},
  volume={38},
  pages={115709--115773},
  year={2026}
}

@article{gross1975logarithmic,
  title={Logarithmic sobolev inequalities},
  author={Gross, Leonard},
  journal={American Journal of Mathematics},
  volume={97},
  number={4},
  pages={1061--1083},
  year={1975},
  publisher={JSTOR}
}

@inproceedings{
renaud2026provably,
title={Provably Accelerated Imaging with Restarted Inertia and Score-based Image Priors},
author={Marien Renaud and Julien Hermant and Deliang Wei and Yu Sun},
booktitle={The Fourteenth International Conference on Learning Representations},
year={2026},
url={https://openreview.net/forum?id=8pQsiFyTQi}
}

@inproceedings{so2025pcm,
  title={PCM: Picard Consistency Model for Fast Parallel Sampling of Diffusion Models},
  author={So, Junhyuk and Shin, Jiwoong and Jang, Chaeyeon and Park, Eunhyeok},
  booktitle={Proceedings of the Computer Vision and Pattern Recognition Conference},
  pages={23313--23322},
  year={2025}
}

@misc{mahajan2025fast,
  title         = {Fast and Efficient Parallel Sampling Using Higher Order Langevin Dynamics},
  author        = {Mahajan, Jaideep and Zhang, Kaihong and Liang, Feng and Liu, Jingbo},
  year          = {2025},
  eprint        = {2510.18242},
  archivePrefix = {arXiv},
  primaryClass  = {math.ST},
  doi           = {10.48550/arXiv.2510.18242}
}

@misc{
zhou2025parallel,
title={Parallel simulation for sampling under isoperimetry and score-based diffusion models},
author={Huanjian Zhou and Masashi Sugiyama},
year={2025},
url={https://openreview.net/forum?id=6Gb7VfTKY7}
}

@article{xu2024provably,
  title={Provably robust score-based diffusion posterior sampling for plug-and-play image reconstruction},
  author={Xu, Xingyu and Chi, Yuejie},
  journal={Advances in Neural Information Processing Systems},
  volume={37},
  pages={36148--36184},
  year={2024}
}

@inproceedings{bouman2023generative,
  title={Generative plug and play: Posterior sampling for inverse problems},
  author={Bouman, Charles A and Buzzard, Gregery T},
  booktitle={2023 59th annual Allerton conference on communication, control, and computing (Allerton)},
  pages={1--7},
  year={2023},
  organization={IEEE}
}

@article{coeurdoux2024plug,
  title={Plug-and-play split gibbs sampler: embedding deep generative priors in bayesian inference},
  author={Coeurdoux, Florentin and Dobigeon, Nicolas and Chainais, Pierre},
  journal={IEEE Transactions on Image Processing},
  volume={33},
  pages={3496--3507},
  year={2024},
  publisher={IEEE}
}

@article{vono2019split,
  title={Split-and-augmented Gibbs sampler—Application to large-scale inference problems},
  author={Vono, Maxime and Dobigeon, Nicolas and Chainais, Pierre},
  journal={IEEE Transactions on Signal Processing},
  volume={67},
  number={6},
  pages={1648--1661},
  year={2019},
  publisher={IEEE}
}

@article{laumont2022bayesian,
  title={Bayesian imaging using plug \& play priors: when langevin meets tweedie},
  author={Laumont, R{\'e}mi and Bortoli, Valentin De and Almansa, Andr{\'e}s and Delon, Julie and Durmus, Alain and Pereyra, Marcelo},
  journal={SIAM Journal on Imaging Sciences},
  volume={15},
  number={2},
  pages={701--737},
  year={2022},
  publisher={SIAM}
}

@inproceedings{zhang2025improving,
  title={Improving diffusion inverse problem solving with decoupled noise annealing},
  author={Zhang, Bingliang and Chu, Wenda and Berner, Julius and Meng, Chenlin and Anandkumar, Anima and Song, Yang},
  booktitle={Proceedings of the Computer Vision and Pattern Recognition Conference},
  pages={20895--20905},
  year={2025}
}

@article{chung2022score,
  title={Score-based diffusion models for accelerated MRI},
  author={Chung, Hyungjin and Ye, Jong Chul},
  journal={Medical image analysis},
  volume={80},
  pages={102479},
  year={2022},
  publisher={Elsevier}
}

@inproceedings{hu2026prism,
  title={PRISM: Probabilistic and Robust Inverse Solver with Measurement-Conditioned Diffusion Prior for Blind Inverse Problems},
  author={Hu, Yuanyun and Bell, Evan and Wang, Guijin and Sun, Yu},
  booktitle={ICASSP 2026-2026 IEEE International Conference on Acoustics, Speech and Signal Processing (ICASSP)},
  pages={11432--11436},
  year={2026},
  organization={IEEE}
}

@inproceedings{chung2023parallel,
  title={Parallel diffusion models of operator and image for blind inverse problems},
  author={Chung, Hyungjin and Kim, Jeongsol and Kim, Sehui and Ye, Jong Chul},
  booktitle={Proceedings of the IEEE/CVF conference on computer vision and pattern recognition},
  pages={6059--6069},
  year={2023}
}

@inproceedings{feng2023score,
  title={Score-based diffusion models as principled priors for inverse imaging},
  author={Feng, Berthy T and Smith, Jamie and Rubinstein, Michael and Chang, Huiwen and Bouman, Katherine L and Freeman, William T},
  booktitle={Proceedings of the IEEE/CVF international conference on computer vision},
  pages={10520--10531},
  year={2023}
}

@article{wu2024principled,
  title={Principled probabilistic imaging using diffusion models as plug-and-play priors},
  author={Wu, Zihui and Sun, Yu and Chen, Yifan and Zhang, Bingliang and Yue, Yisong and Bouman, Katherine L},
  journal={Advances in Neural Information Processing Systems},
  volume={37},
  pages={118389--118427},
  year={2024}
}

@article{sun2024provable,
  title={Provable probabilistic imaging using score-based generative priors},
  author={Sun, Yu and Wu, Zihui and Chen, Yifan and Feng, Berthy T and Bouman, Katherine L},
  journal={IEEE Transactions on Computational Imaging},
  volume={10},
  pages={1290--1305},
  year={2024},
  publisher={IEEE}
}

@article{chung2022diffusion,
  title={Diffusion posterior sampling for general noisy inverse problems},
  author={Chung, Hyungjin and Kim, Jeongsol and Mccann, Michael T and Klasky, Marc L and Ye, Jong Chul},
  journal={arXiv preprint arXiv:2209.14687},
  year={2022}
}

@article{moen2021low,
  title={Low-dose {CT} image and projection dataset},
  author={Moen, Taylor R and Chen, Baiyu and Holmes III, David R and Duan, Xinhui and Yu, Zhicong and Yu, Lifeng and Leng, Shuai and Fletcher, Joel G and McCollough, Cynthia H},
  journal={Medical physics},
  volume={48},
  number={2},
  pages={902--911},
  year={2021},
  publisher={Wiley Online Library}
}

@article{song2020score,
  title={Score-based generative modeling through stochastic differential equations},
  author={Song, Yang and Sohl-Dickstein, Jascha and Kingma, Diederik P and Kumar, Abhishek and Ermon, Stefano and Poole, Ben},
  journal={arXiv preprint arXiv:2011.13456},
  year={2020}
}

@inproceedings{dhariwal2021diffusion,
	author = {Prafulla Dhariwal and Alexander Quinn Nichol},
	booktitle = {Advances in Neural Information Processing Systems},
	title = {Diffusion Models Beat {GAN}s on Image Synthesis},
	year = {2021}}

@inproceedings{song2020improved,
	author = {Yang Song and Stefano Ermon},
	booktitle = {Advances in Neural Information Processing Systems},
	title = {Improved Techniques for Training Score-Based Generative Models},
	year = {2020}}

@inproceedings{liu2023dolce,
  title={{DOLCE}: A model-based probabilistic diffusion framework for limited-angle CT reconstruction},
  author={Liu, Jiaming and Anirudh, Rushil and Thiagarajan, Jayaraman J and He, Stewart and Mohan, K Aditya and Kamilov, Ulugbek S and Kim, Hyojin},
  booktitle={Proceedings of the IEEE/CVF international conference on computer vision},
  pages={10498--10508},
  year={2023}
}

@article{dorjsembe2024conditional,
  title={Conditional diffusion models for semantic 3D brain MRI synthesis},
  author={Dorjsembe, Zolnamar and Pao, Hsing-Kuo and Odonchimed, Sodtavilan and Xiao, Furen},
  journal={IEEE Journal of Biomedical and Health Informatics},
  volume={28},
  number={7},
  pages={4084--4093},
  year={2024},
  publisher={IEEE}
}

@article{ho2020denoising,
  title={Denoising diffusion probabilistic models},
  author={Ho, Jonathan and Jain, Ajay and Abbeel, Pieter},
  journal={Advances in neural information processing systems},
  volume={33},
  pages={6840--6851},
  year={2020}
}

@inproceedings{gu2014weighted,
  title={Weighted nuclear norm minimization with application to image denoising},
  author={Gu, Shuhang and Zhang, Lei and Zuo, Wangmeng and Feng, Xiangchu},
  booktitle={Proceedings of the IEEE conference on computer vision and pattern recognition},
  pages={2862--2869},
  year={2014}
}

@article{sauer1992bayesian,
  title={Bayesian estimation of transmission tomograms using segmentation based optimization},
  author={Sauer, Ken and Bouman, Charles},
  journal={IEEE Transactions on Nuclear Science},
  volume={39},
  number={4},
  pages={1144--1152},
  year={1992}
}

@article{rudin1992nonlinear,
  title={Nonlinear total variation based noise removal algorithms},
  author={Rudin, Leonid I and Osher, Stanley and Fatemi, Emad},
  journal={Physica D: nonlinear phenomena},
  volume={60},
  number={1-4},
  pages={259--268},
  year={1992},
  publisher={Elsevier}
}

@inproceedings{balasubramanian2022towards,
  title={Towards a theory of non-log-concave sampling: first-order stationarity guarantees for Langevin Monte Carlo},
  author={Balasubramanian, Krishna and Chewi, Sinho and Erdogdu, Murat A and Salim, Adil and Zhang, Shunshi},
  booktitle={Conference on Learning Theory},
  pages={2896--2923},
  year={2022},
  organization={PMLR}
}

@article{chen2024accelerating,
  title={Accelerating diffusion models with parallel sampling: Inference at sub-linear time complexity},
  author={Chen, Haoxuan and Ren, Yinuo and Ying, Lexing and Rotskoff, Grant M},
  journal={Advances in Neural Information Processing Systems},
  volume={37},
  pages={133661--133709},
  year={2024}
}

@inproceedings{anari2024fast,
  title={Fast parallel sampling under isoperimetry},
  author={Anari, Nima and Chewi, Sinho and Vuong, Thuy-Duong},
  booktitle={The Thirty Seventh Annual Conference on Learning Theory},
  pages={161--185},
  year={2024},
  organization={PMLR}
}

@article{shih2023parallel,
  title={Parallel sampling of diffusion models},
  author={Shih, Andy and Belkhale, Suneel and Ermon, Stefano and Sadigh, Dorsa and Anari, Nima},
  journal={Advances in Neural Information Processing Systems},
  volume={36},
  pages={4263--4276},
  year={2023}
}

@article{efron2011tweedie,
  title={Tweedie’s formula and selection bias},
  author={Efron, Bradley},
  journal={Journal of the American Statistical Association},
  volume={106},
  number={496},
  pages={1602--1614},
  year={2011},
  publisher={Taylor \& Francis}
}

@inproceedings{song2024solving,
  title={Solving inverse problems with latent diffusion models via hard data consistency},
  author={Song, Bowen and Kwon, Soo Min and Zhang, Zecheng and Hu, Xinyu and Qu, Qing and Shen, Liyue},
  booktitle={International Conference on Learning Representations},
  volume={2024},
  pages={7624--7654},
  year={2024}
}

@article{he2025diffusion,
  title={Diffusion models in low-level vision: A survey},
  author={He, Chunming and Shen, Yuqi and Fang, Chengyu and Xiao, Fengyang and Tang, Longxiang and Zhang, Yulun and Zuo, Wangmeng and Guo, Zhenhua and Li, Xiu},
  journal={IEEE Transactions on Pattern Analysis and Machine Intelligence},
  year={2025},
  publisher={IEEE}
}

@ARTICLE{yang2025ct,
  author={Yang, Liutao and Huang, Jiahao and Yang, Guang and Zhang, Daoqiang},
  journal={IEEE Transactions on Medical Imaging}, 
  title={CT-SDM: A Sampling Diffusion Model for Sparse-View CT Reconstruction Across Various Sampling Rates}, 
  year={2025},
  volume={44},
  number={6},
  pages={2581-2593}
}

@article{kazerouni2023diffusion,
  title={Diffusion models in medical imaging: A comprehensive survey},
  author={Kazerouni, Amirhossein and Aghdam, Ehsan Khodapanah and Heidari, Moein and Azad, Reza and Fayyaz, Mohsen and Hacihaliloglu, Ilker and Merhof, Dorit},
  journal={Medical image analysis},
  volume={88},
  pages={102846},
  year={2023},
  publisher={Elsevier}
}

@inproceedings{kawar2022denoising,
    author = {Bahjat Kawar and Michael Elad and Stefano Ermon and Jiaming Song},
    booktitle = {Advances in Neural Information Processing Systems},
    title = {Denoising Diffusion Restoration Models},
    year = {2022}
}

@inproceedings{rout2023solving,
    title={Solving Linear Inverse Problems Provably via Posterior Sampling with Latent Diffusion Models},
    author={Litu Rout and Negin Raoof and Giannis Daras and Constantine Caramanis and Alex Dimakis and Sanjay Shakkottai},
    booktitle={Thirty-seventh Conference on Neural Information Processing Systems},
    year={2023},
    url={https://openreview.net/forum?id=XKBFdYwfRo}
}

@inproceedings{ho2021classifierfree,
title={Classifier-Free Diffusion Guidance},
author={Jonathan Ho and Tim Salimans},
booktitle={NeurIPS 2021 Workshop on Deep Generative Models and Downstream Applications},
year={2021},
}

@incollection{kam2019uncertainty,
	author = {Kampourakis, Kostas and McCain, Kevin},
	booktitle = {Uncertainty: How It Makes Science Advance},
	month = {12},
	publisher = {Oxford University Press},
	title = {How Uncertainty Makes Science Advance},
	year = {2019}
}

@article{begoli2019the,
	author = {Begoli, Edmon and Bhattacharya, Tanmoy and Kusnezov, Dimitri},
	journal = {Nature Machine Intelligence},
	number = {1},
	pages = {20--23},
	title = {The need for uncertainty quantification in machine-assisted medical decision making},
	volume = {1},
	year = {2019}
}

@article{faye2024regularization,
  title={Regularization by denoising: Bayesian model and Langevin-within-split Gibbs sampling},
  author={Faye, Elhadji C and Fall, Mame Diarra and Dobigeon, Nicolas},
  journal={IEEE Transactions on Image Processing},
  volume={34},
  pages={221--234},
  year={2024},
  publisher={IEEE}
}

@article{pereyra2023split,
  title={The split Gibbs sampler revisited: improvements to its algorithmic structure and augmented target distribution},
  author={Pereyra, Marcelo and Vargas-Mieles, Luis A and Zygalakis, Konstantinos C},
  journal={SIAM Journal on Imaging Sciences},
  volume={16},
  number={4},
  pages={2040--2071},
  year={2023},
  publisher={SIAM}
}

@article{chewi2025analysis,
  title={Analysis of langevin monte carlo from poincare to log-sobolev},
  author={Chewi, Sinho and Erdogdu, Murat A and Li, Mufan and Shen, Ruoqi and Zhang, Matthew S},
  journal={Foundations of Computational Mathematics},
  volume={25},
  number={4},
  pages={1345--1395},
  year={2025},
  publisher={Springer}
}

@article{gribonval2020characterization,
  title={A characterization of proximity operators},
  author={Gribonval, R{\'e}mi and Nikolova, Mila},
  journal={Journal of Mathematical Imaging and Vision},
  volume={62},
  number={6},
  pages={773--789},
  year={2020},
  publisher={Springer}
}

@article{vempala2019rapid,
  title={Rapid convergence of the unadjusted langevin algorithm: Isoperimetry suffices},
  author={Vempala, Santosh and Wibisono, Andre},
  journal={Advances in neural information processing systems},
  volume={32},
  year={2019}
}

@article{ji2022amos,
  title={{AMOS}: A large-scale abdominal multi-organ benchmark for versatile medical image segmentation},
  author={Ji, Yuanfeng and Bai, Haotian and Ge, Chongjian and Yang, Jie and Zhu, Ye and Zhang, Ruimao and Li, Zhen and Zhanng, Lingyan and Ma, Wanling and Wan, Xiang and others},
  journal={Advances in neural information processing systems},
  volume={35},
  pages={36722--36732},
  year={2022}
}

\end{document}